\documentclass[10pt,twocolumn,letterpaper]{article}

\usepackage{cvpr}              
\usepackage{multirow}
\usepackage{graphicx}
\usepackage[table]{xcolor}

\newcommand{\model}[1]{\texttt{#1}}

\usepackage{xcolor}

\usepackage[normalem]{ulem}
\usepackage{wrapfig}
\usepackage{adjustbox}
\usepackage[table]{xcolor}

\newcommand{\tablestyle}[2]{\setlength{\tabcolsep}{#1}\renewcommand{\arraystretch}{#2}\centering\footnotesize}

\usepackage{makecell}
\usepackage{algorithm}
\usepackage{algpseudocode}
\definecolor{cvprblue}{rgb}{0.21,0.49,0.74}
\usepackage[pagebackref,breaklinks,colorlinks,allcolors=cvprblue]{hyperref}

\def\paperID{xxxxxx} 
\def\confName{CVPR}
\def\confYear{2026}

\title{DC-Merge: Improving Model Merging with Directional Consistency}
\author{Han-Chen Zhang\thanks{Equal contribution,  \ $^\dag$Corresponding author} \quad
Zi-Hao Zhou$^{*}$ \quad
Mao-Lin Luo \quad
Shimin Di \quad\\
Min-Ling Zhang \quad
Tong Wei$^{\dag}$ \and
{
\textsuperscript{1}School of Computer Science and Engineering, Southeast University, Nanjing 210096, China\hfill} \\
{\textsuperscript{2}Key Laboratory of Computer Network and Information Integration (Southeast University)}\\
{\texttt{\small \{hanchenzh, zhouzih, weit\}@seu.edu.cn}}\\
}

\begin{document}
\maketitle
\begin{abstract}
Model merging aims to integrate multiple task-adapted models into a unified model that preserves the knowledge of each task.
In this paper, we identify that the key to this knowledge retention lies in maintaining the directional consistency of singular spaces between merged multi-task vector and individual task vectors.
However, this consistency is frequently compromised by two issues: i) an imbalanced energy distribution within task vectors, where a small fraction of singular values dominate the total energy, leading to the neglect of semantically important but weaker components upon merging, and ii) the geometric inconsistency of task vectors in parameter space, which causes direct merging to distort their underlying directional geometry.
To address these challenges, we propose DC-Merge, a method for directional-consistent model merging. It first balances the energy distribution of each task vector by smoothing its singular values, ensuring all knowledge components are adequately represented.
These energy-balanced vectors are then projected onto a shared orthogonal subspace to align their directional geometries with minimal reconstruction error.
Finally, the aligned vectors are aggregated in the shared orthogonal subspace and projected back to the original parameter space. 
Extensive experiments on vision and vision-language benchmarks show that DC-Merge consistently achieves state-of-the-art performance in both full fine-tuning and LoRA settings. The implementation code is available at \href{https://github.com/Tobeginwith/DC-Merge}{https://github.com/Tobeginwith/DC-Merge}.
\end{abstract}

\section{Introduction}

Pre-trained models are the foundation of modern machine learning systems~\cite{carion2020endtoend, radford2021learning, DINO, zhai2023sigmoid}. In practice, they are typically fine-tuned for specialization in specific tasks~\cite{wortsman2022robust, ilharco2022patching, fan2023federated,fan2024locally}. A growing body of research has focused on \emph{model merging}~\cite{li2023deepmodelfusionsurvey}, which integrates multiple task-adapted models into a unified model while preserving each task’s capability. Many methods have been proposed to improve the effectiveness of model merging by reducing sign conflicts~\cite{yadav2023tiesmerging}, by aligning gradients~\cite{DaheimMPGK24}, or through magnitude-based selection~\cite{MarczakTTC24, davari2023model, jin2023dataless}. 
Despite its potential to enable efficient multi-task adaptation without retraining, existing approaches often suffer performance degradation after merging~\cite{li2023deepmodelfusionsurvey}, especially when tasks originate from heterogeneous domains~\cite{DaheimMPGK24}. Recent studies aim to reduce the interference among different tasks~\cite{tsv, cheng2025whoever,luo2026keeplora}, while the underlying mechanism of how task-specific capabilities are preserved after merging remains underexplored. A central question thus arises: \emph{what property must be preserved to retain each task’s ability after merging?}


Following Task Arithmetic (TA)~\cite{ilharco2023task}, we define a task vector as the parameter difference between a fine-tuned model and its pre-trained counterpart.    
Each task vector can be decomposed via singular value decomposition (SVD) into a set of orthogonal \emph{knowledge vectors}, each representing a distinct adaptation direction weighted by its singular value. We term each of these directions as a \emph{knowledge component} and the corresponding singular values reflect the energy distributed across these components.
We observe that task performance after model merging primarily depends on the \emph{directional consistency} between the merged and original task vectors.
Specifically, as long as the directions of the knowledge components are preserved, the merged model retains most task capabilities, even if their energy distribution changes.  
In contrast, slight directional deviations significantly degrade performance, indicating that maintaining directional consistency of knowledge components is crucial to maintaining task performance.

To quantify this consistency, we propose \emph{directional similarity} (\(\mathrm{DirSim}\)), which measures the consistency of directional geometry between two task vectors while discounting the influence of energy distribution.
Unlike cosine similarity, which emphasizes consistency of high-energy components, \(\mathrm{DirSim}\) also accounts for the directional consistency of weaker yet semantically informative components.  
Empirically, \(\mathrm{DirSim}\) shows a strongly positive correlation with post-merge task-wise performance, validating it as a reliable indicator of knowledge preservation.

Despite its importance, directional consistency is often violated by two fundamental issues.  
First, the energy distribution of task vectors is imbalanced, where a few singular values capture most of the energy (as shown in Figure \ref{fig_longtailed}), causing the model to overemphasize on high-energy directions and thereby hindering generalization and directional geometry preservation.
Second, directly merging task vectors in the original parameter space leads to basis misalignment: different tasks span heterogeneous low-rank subspaces whose orientations are not geometrically aligned.
Consequently, the merged task vector fails to preserve the directional geometry of each task vector that characterizes the task’s knowledge.
To address these challenges, we propose a new method called \emph{DC-Merge}, which explicitly enforces directional consistency between the merged multi-task vector and each original task vector.  
DC-Merge consists of two complementary modules:  
i) \emph{energy smoothing} redistributes the singular values of each task vector to balance the energy distribution of its knowledge components, thereby preventing the merging process from overlooking weaker but semantically rich directions within each task vector.
ii) \emph{cover space merging} then projects all smoothed task vectors into a shared orthogonal subspace before aggregation, ensuring that merging occurs under a consistent cover basis without cross-task directional interference.
Together, these modules preserve the task directional geometry during merging, enabling stable multi-task compatibility and strong generalization. Extensive experiments on both full fine-tuning (FFT) and LoRA~\cite{hu2022lora} setups show that DC-Merge achieves state-of-the-art results on both vision and vision-language benchmarks while maintaining high directional consistency with original task vectors.

In summary, our key contributions are as follows:
\begin{itemize}
    \item We correlate the model merging performance with a novel concept \emph{directional consistency} between the merged multi-task vector and individual task vectors.

    \item We introduce $\mathrm{DirSim}$, a new metric that isolates directional consistency from energy distribution effects. $\mathrm{DirSim}$ shows a strong positive correlation with the performance of merged model.

    \item We propose \emph{DC-Merge}, a method that enhances directional consistency by first balancing energy distribution of task vectors and then merging them within a shared orthogonal subspace.

    \item Extensive experiments on vision and vision-language benchmarks demonstrate that DC-Merge achieves state-of-the-art performance in both FFT and LoRA settings.
\end{itemize}

\newtheorem{proposition}{Proposition}
\newtheorem{remark}{Remark}

\section{Directional Consistency Matters}

This section reveals the intrinsic imbalance of energy distribution across knowledge components and presents empirical evidence confirming the importance of directional consistency in model merging. We also introduce a new metric to quantify this consistency. Unless otherwise specified, the experiments in this section are based on a \model{ViT-B-32} visual encoder~\cite{dosovitskiy2021imageworth16x16words} under LoRA configuration.
\subsection{Preliminary}
\noindent
\textbf{Model Merging.}  
Given a pre-trained parameter set $\boldsymbol{W}_0$ and a collection of fine-tuned models $\{\boldsymbol{W}_i\}_{i=1}^T$ obtained from distinct tasks, model merging seeks a merged parameter set $\widetilde{\boldsymbol{W}}$ that approximates the behavior of each $\boldsymbol{W}_i$ on its corresponding task.  

\noindent
\textbf{Task Vectors.}  
For a FFT model of the $i$-th task, the task vector is defined as
$\Delta \boldsymbol{W}_i = \boldsymbol{W}_i - \boldsymbol{W}_0$,
which captures the direction and magnitude of adaptation in the weight space~\cite{ilharco2023task}.  
In the LoRA paradigm, the task-specific update is parameterized explicitly as
$\Delta \boldsymbol{W}_i^{\mathrm{LoRA}} = \boldsymbol{B}_i \boldsymbol{A}_i$,
where $\boldsymbol{A}_i \in \mathbb{R}^{r\times d}$ and $\boldsymbol{B}_i \in \mathbb{R}^{d\times r}$.  
Thus, LoRA directly produces a compact and structured low-rank task vector.

\noindent
\textbf{Unified View and Low-Rank Merging.}  
Although FFT and LoRA differ in parameterization, they are inherently connected under a unified low-rank formulation. Empirically, FFT updates $\Delta \boldsymbol{W}_i$ tend to reside in a low-dimensional subspace~\cite{hu2022lora} and can be well approximated by a truncated SVD as
$
\Delta \boldsymbol{W}_i \approx \boldsymbol{U}_i \boldsymbol{\Sigma}_i \boldsymbol{V}_i^\top.
$
From this perspective, LoRA explicitly constrains the parameter updates to a low-rank subspace, whereas FFT implicitly exhibits a similar low-rank structure that can be revealed through singular value analysis and compression. This insight bridges the two approaches under a unified low-rank adaptation paradigm.

Building on this unified view, model merging can be viewed as an operation on these low-rank matrices.  
The merging process thus involves:  
i) extracting low-rank representations for each task vector,  
ii) merging these low-rank matrices to effectively integrate multi-task knowledge, and 
iii) constructing the final model by combining the pre-trained weights and merged multi-task vector through
$\widetilde{\boldsymbol{W}} = \boldsymbol{W}_0 + \Delta \widetilde{\boldsymbol{W}}$.

\subsection{Balanced Energy Enhances Generalization}
\label{sec:Motivation}

\begin{figure}[!ht]
    \centering
    \begin{subfigure}[t]{0.23\textwidth}
        \centering
        \includegraphics[width=\textwidth]{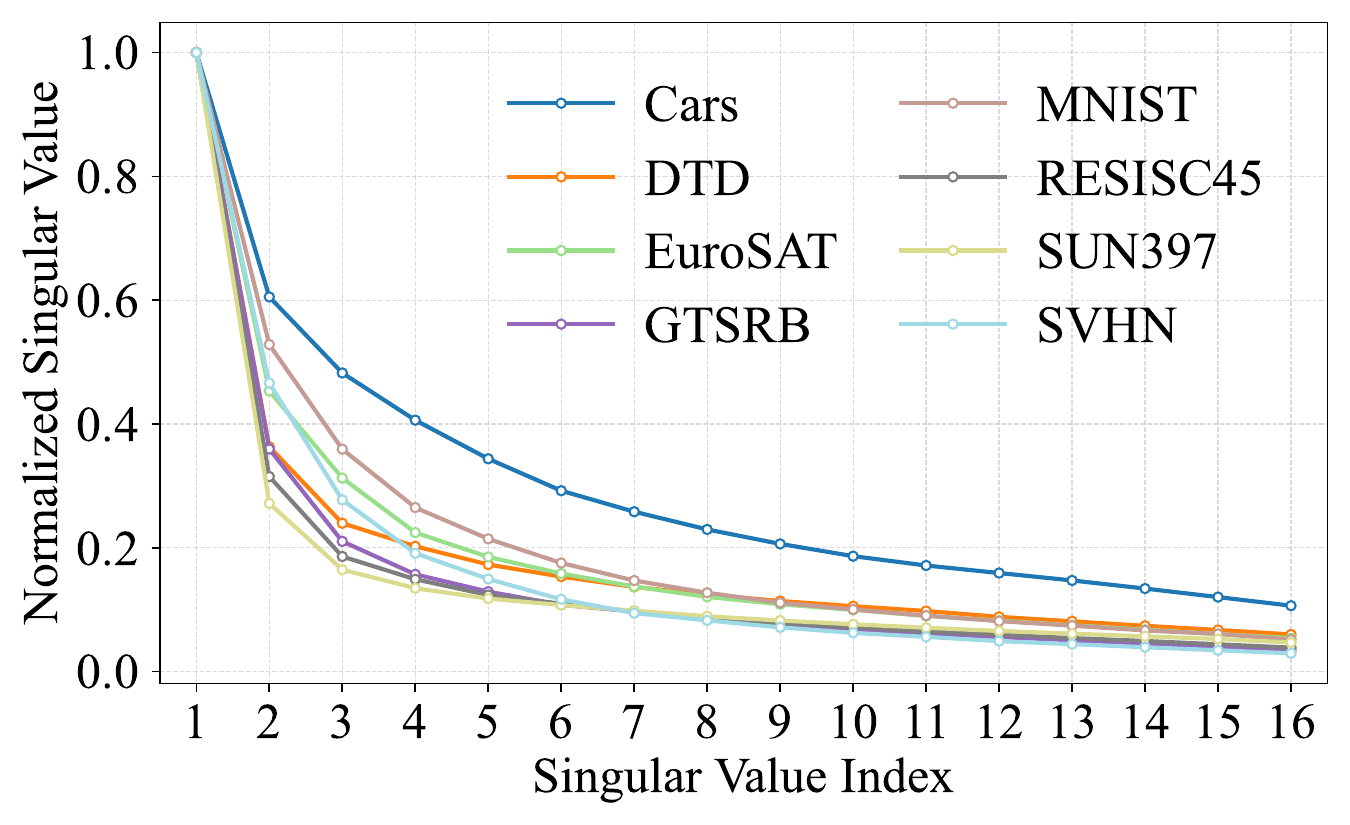}
        \caption{LoRA}
        \label{fig_longtailed_lora}
    \end{subfigure}
    \hfill
    \begin{subfigure}[t]{0.23\textwidth}
        \centering
        \includegraphics[width=\textwidth]{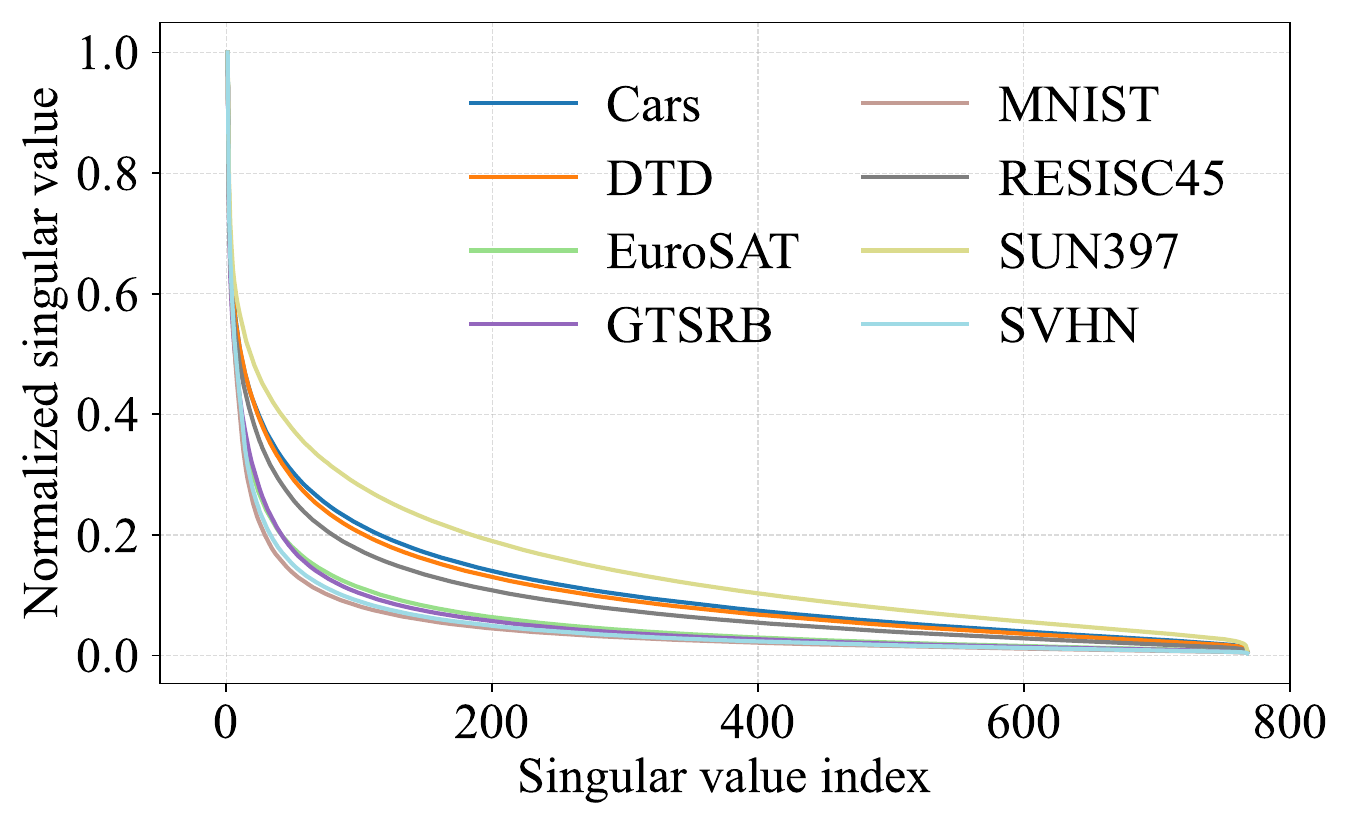}
        \caption{FFT}
        \label{fig_longtailed_fft}
    \end{subfigure}
    \caption{The singular value distribution of task vectors averaged across all layers. We normalize each singular value by the largest one within each dataset to eliminate the magnitude discrepancy among different datasets.}
    \label{fig_longtailed}
    \vspace{-0.5cm}
\end{figure}

\begin{figure*}[t]
    \centering
    \begin{subfigure}[!t]{0.52\textwidth}
        \centering
        \includegraphics[width=\textwidth]{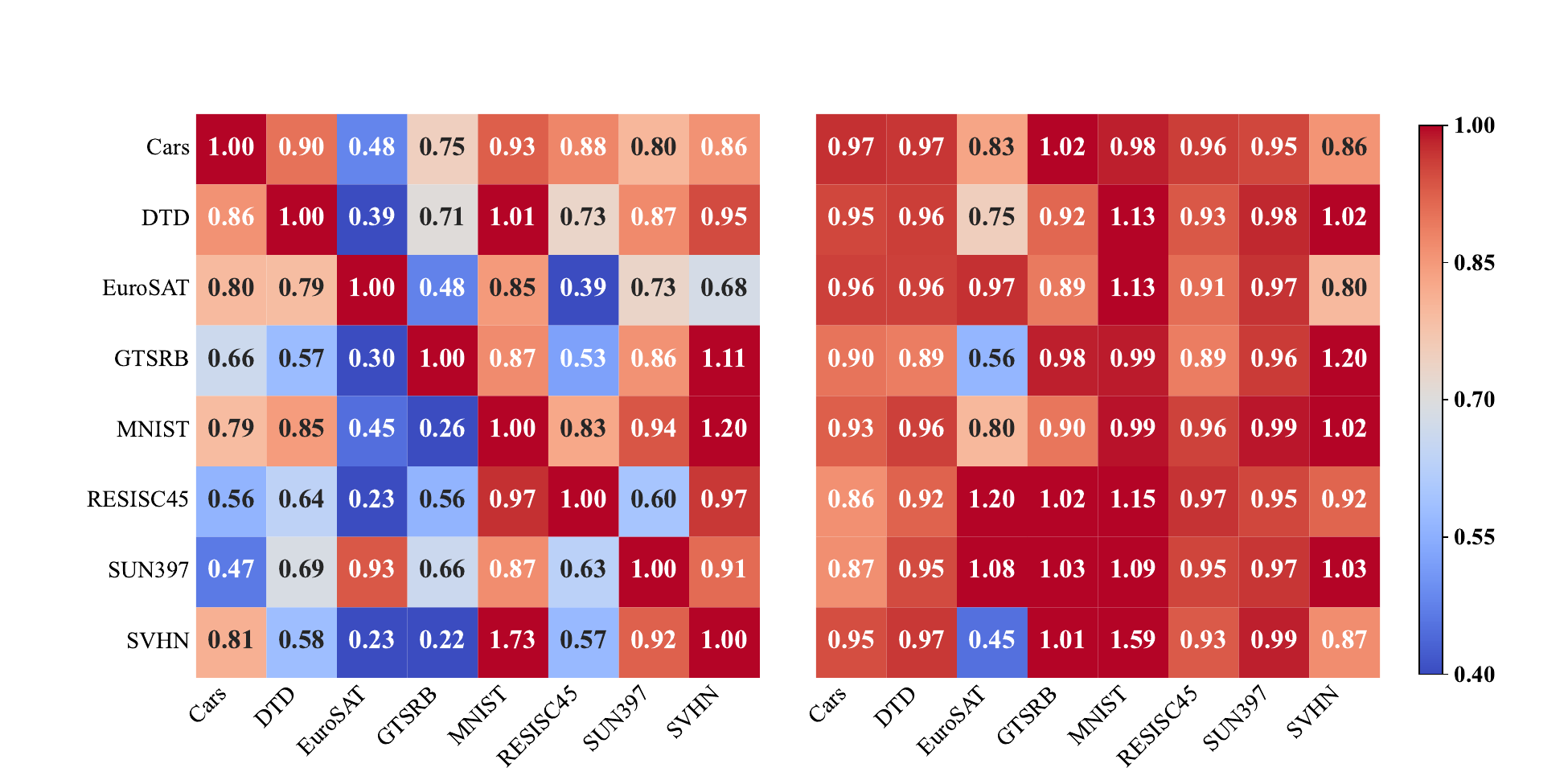}
        \caption{
        }
        \label{fig_transfer}
    \end{subfigure}
    \hfill
    \begin{subfigure}[!t]{0.44\textwidth}
        \centering
        \includegraphics[width=\textwidth]{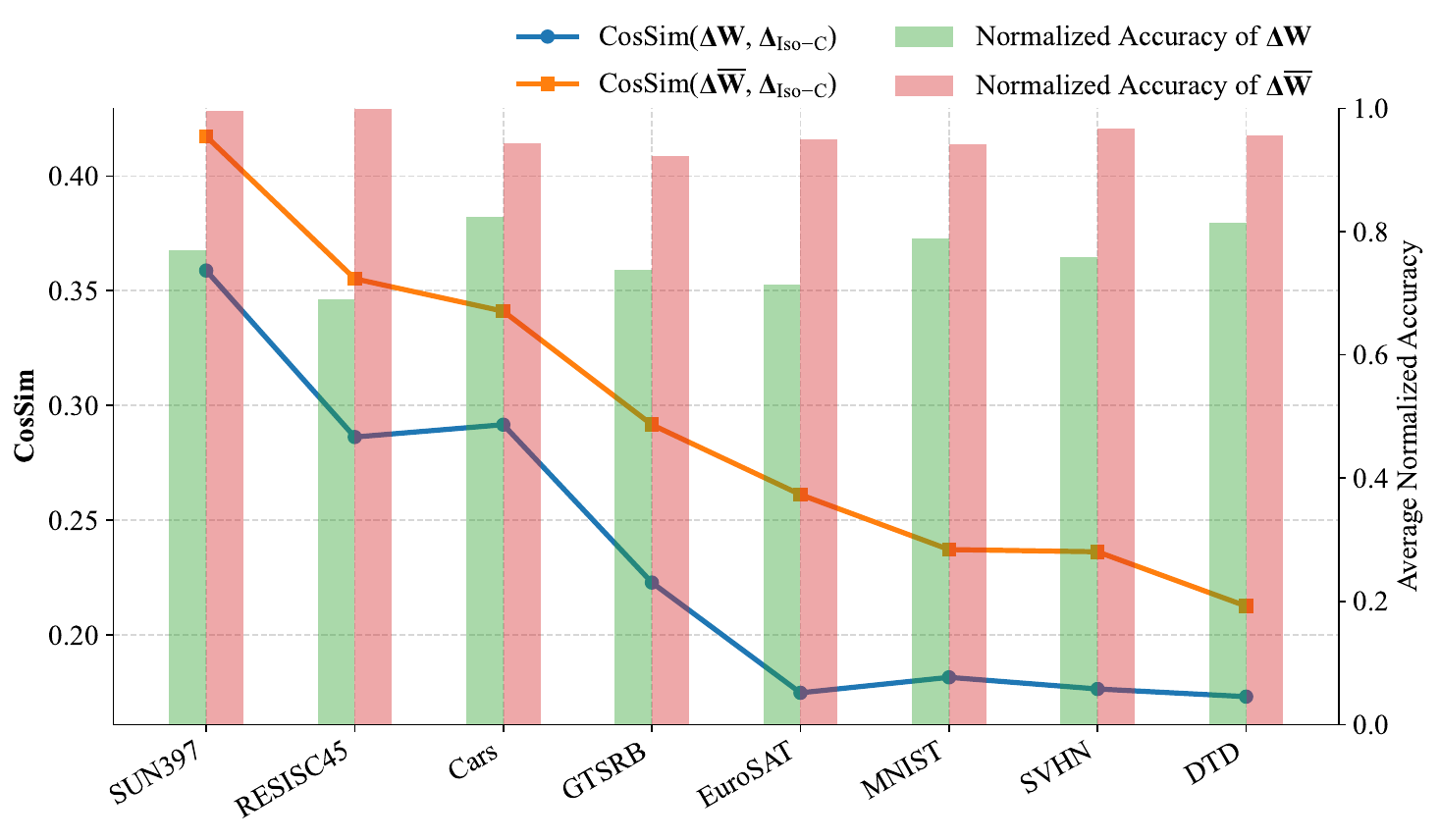}
        \caption{
        }
        \label{bars_fig}
    \end{subfigure}
    \vspace{-0.2cm}
    \caption{
    \textbf{(a)} Comparison of cross-task transferability before (left) and after (right) balancing the energy distribution. The diagonal elements represent the relative performance of each task with respect to its fine-tuned model, while the off-diagonal elements indicate the relative transfer performance on other tasks, normalized against their zero-shot baseline. 
    \textbf{(b)} Each task vector's cosine similarity with multi-task vector (lines) and average normalized transfer accuracy of each task vector (bars). We compare original task vectors and their energy-balanced counterparts. The energy-balanced task vectors achieve higher $\mathrm{CosSim}$ with the multi-task vector (we use $\boldsymbol{\Delta}_{\text{Iso-C}}$~\cite{marczak2025notaskleftbehind} for simplicity) and better cross-task generalization, indicating that balancing the energy distribution across knowledge components enhances multi-task expressiveness.
    }
    \label{fig_combined}
    \vspace{-0.3cm}
\end{figure*}

\noindent
\textbf{A New Perspective on Task Vectors.}  
We first provide a new interpretation of task vectors from the viewpoint of their intrinsic low-rank structure.  
Given a task vector $\Delta \boldsymbol{W}$ with rank $r\ll d$, we can decompose it using SVD:
\begin{equation}
\Delta \boldsymbol{W} = \sum_{i=1}^{r} \sigma^i \boldsymbol{u}^i \boldsymbol{v}^{i \top},
\end{equation}
where $\{\boldsymbol{u}^i\}_{i=1}^{r}$ and $\{\boldsymbol{v}^i\}_{i=1}^{r}$ are the left and right singular vectors, respectively, and $\{\sigma_i\}_{i=1}^{r}$ are the corresponding singular values. 
From this decomposition, each term $\sigma^i \boldsymbol{u}^i \boldsymbol{v}^{i\top}$ can be regarded as a specific \emph{knowledge vector} for task adaptation, and the adaptation direction $\boldsymbol{u}^i \boldsymbol{v}^{i\top}$ is defined as a \emph{knowledge component}. The singular value $\sigma^i$ quantifies the degree to which the corresponding knowledge component is utilized, thus the distribution of $\boldsymbol{\sigma}=(\sigma^1,\dots,\sigma^r)$ can be interpreted as the \emph{energy distribution} across knowledge components.
%

\noindent
\textbf{Observation: Singular Values Follow a Long-Tailed Distribution.}  
An important empirical observation arises when analyzing the low-rank task vectors obtained from practical training. As shown in Figure~\ref{fig_longtailed}, the decomposed knowledge vectors $\{\sigma_i \boldsymbol{u}_i \boldsymbol{v}_i^{\top}\}_{i=1}^{r}$ usually follow a long-tailed energy distribution, where a small fraction of knowledge components dominate the total energy, indicating that the knowledge captured by the task vector is inherently imbalanced. This intrinsic imbalance leads to a potential drawback in model behavior: as a small number of knowledge components dominate the task adaptation, the model tends to overfit to specific patterns while neglecting weaker but semantically important components. 

To further illustrate this phenomenon, Figure~\ref{fig_transfer} presents the transfer performance in eight tasks, contrasting the results before and after balancing the internal knowledge of each task vector.
The \emph{diagonal elements} of each heatmap represent the degree to which the original task’s capability is preserved, while the \emph{off-diagonal elements} measure the zero-shot transferability to other tasks, reflecting how well the knowledge generalizes beyond its training domain. As shown in Figure~\ref{fig_transfer} (left), directly using the original low-rank task vector obtained from fine-tuning causes severe performance degradation on unrelated tasks, with many off-diagonal entries significantly suppressed, implying that over-concentrated energy distribution harms cross-task generalization. In contrast, Figure~\ref{fig_transfer} (right) displays the results of \textit{energy-balanced task vectors} constructed by a simple averaging strategy:
\begin{equation}
\label{eq:avg}
\Delta \overline{\boldsymbol{W}} = \left(\frac{\sum_{i=1}^{r} \sigma^{i}}{r}\right) \left(\sum_{i=1}^{r} \boldsymbol{u}^{i} \boldsymbol{v}^{i\top}\right).
\end{equation}
The diagonal elements are quite close to $1.0$, indicating that the vast majority of task capability is preserved, and the off-diagonal elements increase notably, suggesting improved zero-shot transfer and multi-task compatibility.

\noindent
\textbf{Revisiting Task Vector Similarity from a Knowledge Decomposition Perspective.}
To further analyze the underlying reason behind better generalization capabilities after balancing the energy distribution, we revisit the cosine similarity between task vectors~\cite{ilharco2023task} from the perspective of knowledge decomposition.  
We argue that this metric can be interpreted as the \emph{expressive capacity} of one task vector to represent another, i.e., how well the knowledge of task $t$ can be linearly reconstructed by task $s$.
\vspace{-0.2cm} 
\begin{proposition}
\label{pro:cossim}
Given the knowledge vector decompositions of two task vectors $\Delta \boldsymbol{W}_{s}$ and $\Delta \boldsymbol{W}_{t}$, their cosine similarity can be equivalently expressed as
\begin{equation}
\begin{aligned}
\mathrm{CosSim}(\Delta \boldsymbol{W}_{s}, \Delta \boldsymbol{W}_{t}) 
&= \frac{\langle \Delta \boldsymbol{W}_{s}, \Delta \boldsymbol{W}_{t}\rangle}{\|\Delta \boldsymbol{W}_{s}\|_{F}\,\|\Delta \boldsymbol{W}_{t}\|_{F}} \\
&= \frac{ {\boldsymbol{\sigma}_{s}} \; {\boldsymbol{R}(s,t)} \; \big({\boldsymbol{\sigma}_{t}}\big)^{\top} }{\|{\boldsymbol{\sigma}_{s}}\|_{2}\,\|{\boldsymbol{\sigma}_{t}}\|_{2}},
\end{aligned}
\end{equation}
where ${\boldsymbol{R}(s,t)}\in\mathbb{R}^{n\times m}$ is defined entry-wise as:
\begin{equation}
{\boldsymbol{R}}_{i,j}(s,t)= \big(\boldsymbol{u}^i_{s}\big)^{\top}\boldsymbol{u}^j_{t}\;\big(\boldsymbol{v}^j_{t}\big)^{\top}\boldsymbol{v}^i_{s}.
\end{equation}
\end{proposition}

\noindent
\textbf{Remark.} The matrix $\boldsymbol{R}(s,t)$ measures the directional consistency between the two knowledge bases. Each entry ${\boldsymbol{R}}_{i,j}(s,t)
= (\boldsymbol{u}^i_{s})^{\top}\boldsymbol{u}^j_{t} \, (\boldsymbol{v}^j_{t})^{\top}\boldsymbol{v}^i_{s}$ quantifies how the $j$-th knowledge component of task $t$ can be projected onto the $i$-th knowledge component of task $s$.  
Therefore, $\boldsymbol{R}(s,t)$ can be interpreted as a projection operator that expresses the knowledge geometry of task $t$ in the basis of task $s$. From this perspective, the overall cosine similarity $\mathrm{CosSim}(\Delta \boldsymbol{W}_{s}, \Delta \boldsymbol{W}_{t})
= \frac{\boldsymbol{\sigma}_{s} \boldsymbol{R}(s,t) (\boldsymbol{\sigma}_{t})^{\top}}
{\|\boldsymbol{\sigma}_{s}\|_{2}\|\boldsymbol{\sigma}_{t}\|_{2}}$
can be viewed as the weighted aggregation of these projections, reflecting how effectively task $s$ can represent or reconstruct the knowledge of task $t$.  
When the rank $m$ of $\Delta \boldsymbol{W}_{t}$ is higher than the rank $n$ of $\Delta \boldsymbol{W}_{s}$, this interpretation is particularly intuitive: $\Delta \boldsymbol{W}_{s}$ spans a lower-dimensional subspace that attempts to encode the richer knowledge geometry of task $t$.  
Thus, a higher $\mathrm{CosSim}$ indicates stronger expressiveness of task $s$ with respect to task $t$. We leave further analysis of this perspective in Appendix~\ref{sec:cos_projection_expressiveness}.


\noindent
\textbf{Energy-Balanced Knowledge Components Enhance Multi-Task Performance.}
Building upon the cosine similarity analysis above, we can now explain why energy-balanced task vectors achieve stronger multi-task capability.  
When singular values $\boldsymbol{\sigma}$ are highly skewed, the corresponding task vectors collapse onto a few dominant knowledge components, limiting the span of the subspace and thereby reducing its expressive coverage over other tasks.  
In contrast, balancing the energy distribution across knowledge components prevents representational collapse and enlarges the subspace, thereby enhancing its ability to encode multiple tasks.

To verify this, we compute the cosine similarity between the single-task vectors and multi-task vector, and compare it with that of the energy-balanced task vectors obtained by Eq.~\eqref{eq:avg}.
As shown in Figure~\ref{bars_fig}, the balanced vectors consistently exhibit higher cosine similarity with the multi-task model vector. 
Empirically, this aligns with higher average normalized transfer accuracy. 
These results suggest that energy smoothing improves the task vector’s ability to represent multi-task knowledge.


\begin{figure*}[!t]
    \centering
    \begin{subfigure}[!t]{0.40\textwidth}
        \centering
        \includegraphics[width=\textwidth, height=3.7cm, trim={0 0.6cm 0 0.6cm}, clip]{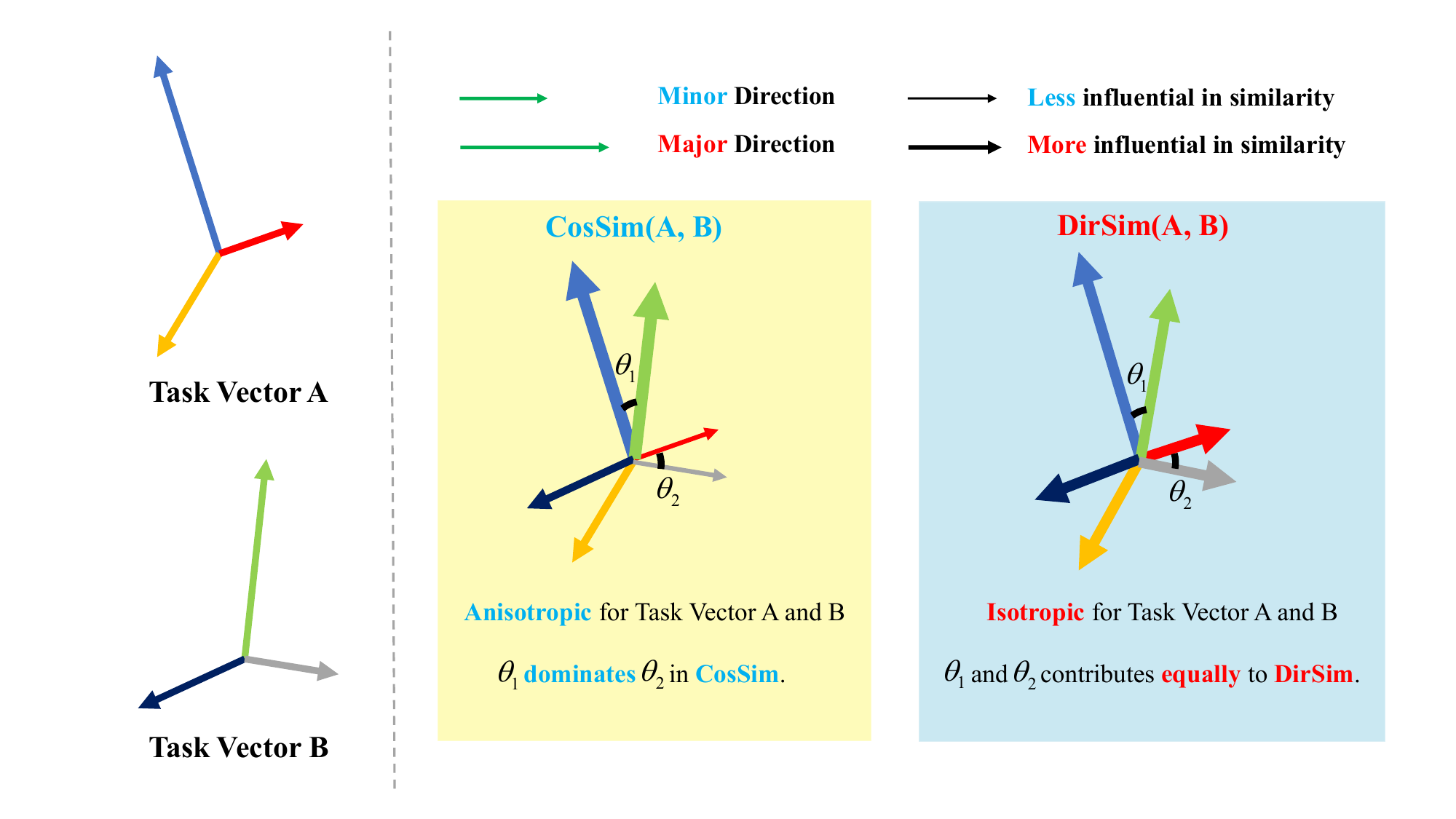}
        \caption{}
        \label{fig_cosSim_DirSim}
    \end{subfigure}
    \hfill
    \begin{subfigure}[!t]{0.29\textwidth}
        \centering
        \includegraphics[width=\textwidth, height=3.75cm]{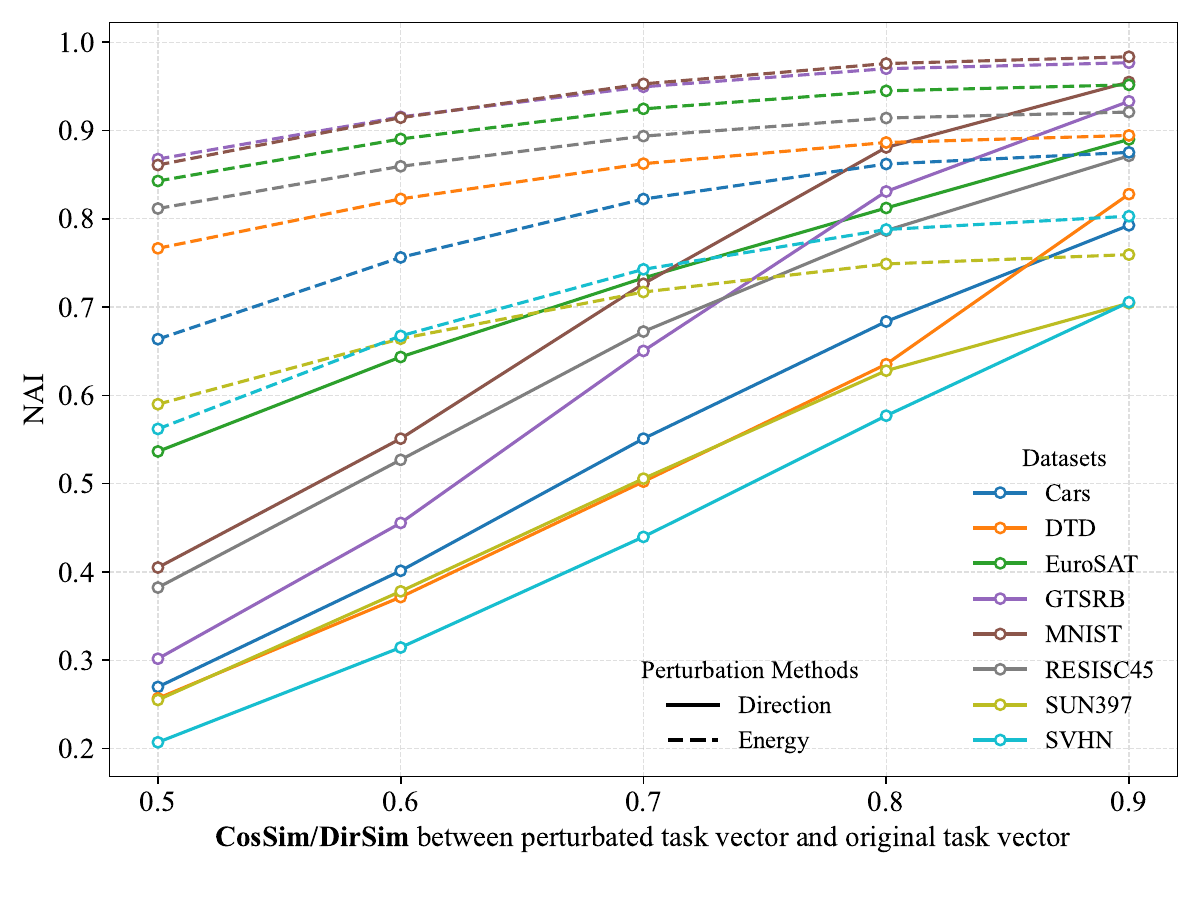}
        \caption{}
        \label{fig_mag_dir}
    \end{subfigure}
    \hfill
    \begin{subfigure}[!t]{0.30\textwidth}
        \centering
        \includegraphics[width=\textwidth, height=3.8cm]{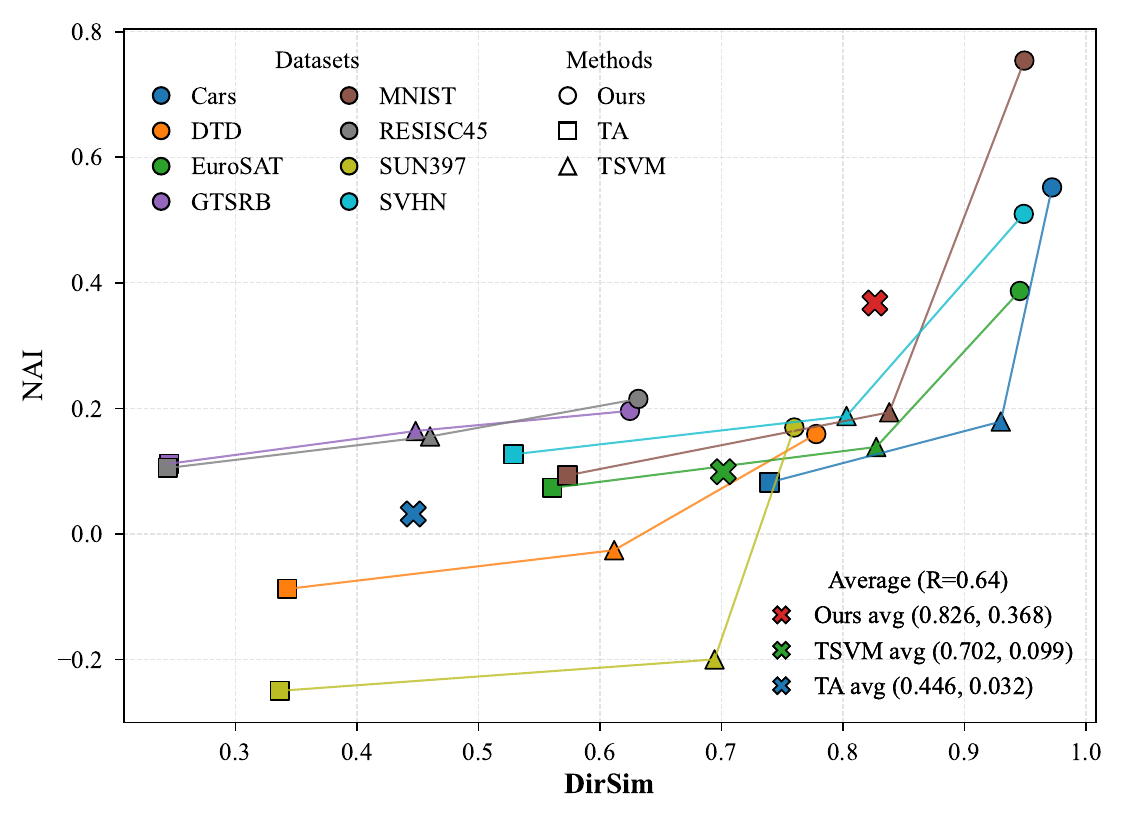}
        \caption{}
        \label{fig_metric_acc}
    \end{subfigure}
    \vspace{-0.25cm}
    \caption{
    \textbf{(a)} Comparison of $\mathrm{CosSim}$ and $\mathrm{DirSim}$. $\mathrm{DirSim}$ considers the similarity between every pair of directions equally, whereas $\mathrm{CosSim}$ mainly focuses on the similarity among the dominant directions while ignoring the minor ones.
    \textbf{(b)} {Empirical validation of the importance of preserving directional geometry.}
    Solid line: task performance vs. \(\mathrm{DirSim}\) under random directional perturbations;
    dashed line: task performance vs. \(\mathrm{CosSim}\) under energy distribution perturbations.
    \textbf{(c)} {Correlation of task-wise performance with projected \(\mathrm{DirSim}\).}
    The overall correlation is positive (Pearson \(R=0.64\)) and per-method averages follow the same trend. Similar patterns persist under larger task scales as illustrated in Figure~\ref{fig_metric_boxplot}.
    We utilize Normalized Accuracy Improvement (NAI)~\cite{marczak2025notaskleftbehind} to measure task-wise performance in \textbf{(b)} and \textbf{(c)}.
    }
    \label{fig_combined_mag_dir_metric}
    \vspace{-0.3cm}
\end{figure*}

\subsection{Measuring Directional Consistency}
\label{sc:dirsim_proposal}
\noindent
\textbf{Directional Knowledge Similarity.}  
We now delve into the underlying mechanism of task knowledge preservation. 
The cosine similarity can be factorized into two components: i) matrix ${\boldsymbol{R}(s,t)}$ quantifies the directional consistency between the two knowledge components, and ii) singular values ${\boldsymbol{\sigma}}_{s}, \boldsymbol{\sigma}_{t}$ encode the importance of these knowledge components. While both terms contribute to the overall similarity, we argue that the preservation of task ability primarily depends on the \emph{directional consistency} rather than on the energy distribution.
Empirically, this observation is supported by the fact that energy-balanced task vectors can largely maintain the performance of the original task vector, which implies that as long as the relative directions of knowledge vectors are preserved, the model can retain most of its learned behavior. 

To isolate this directional factor and quantify it explicitly, we propose a new similarity metric that removes the influence of energy distribution by uniformizing it with $\bar{\boldsymbol{\sigma}}_{s} = \frac{\boldsymbol{1}_{n}}{\sqrt{n}}, 
\bar{\boldsymbol{\sigma}}_{t} = \frac{\boldsymbol{1}_{m}}{\sqrt{m}}$.
Substituting back to the cosine similarity leads to a purely directional consistency measure:
\begin{equation}
\begin{aligned}
\label{eq:dir_sim}
\mathrm{DirSim}(\Delta \boldsymbol{W}_{s}, \Delta \boldsymbol{W}_{t}) 
&\triangleq \bar{\boldsymbol{\sigma}}_{s}\; {\boldsymbol{R}(s,t)}\; \big(\bar{\boldsymbol{\sigma}}_{t}\big)^{\top}
\\ & = \frac{1}{\sqrt{nm}}\,\boldsymbol{1}_{n}^{\top}\,{\boldsymbol{R}(s,t)}\,\boldsymbol{1}_{m}.
\end{aligned}
\end{equation}
$\mathrm{DirSim}$ equally considers the similarity between every pair of directions, whereas $\mathrm{CosSim}$ is subject to the similarity among the dominant directions (Figure~\ref{fig_cosSim_DirSim}).  
A higher value $\mathrm{DirSim}$ implies that the two task vectors share more directionally consistent knowledge components and that one can better represent the other within its knowledge basis.

To verify the above claim, we conduct controlled perturbation experiments and analyze how the retained task performance correlates with both $\mathrm{DirSim}$ and $\mathrm{CosSim}$.  
As illustrated in Figure~\ref{fig_mag_dir}, performance decreases as $\mathrm{DirSim}$ declines, indicating that directional inconsistency leads to substantial performance loss.  
In contrast, when only energy is redistributed while the directions remain aligned, performance remains largely stable despite notable changes in $\mathrm{CosSim}$.  
These results provide strong empirical evidence that preserving knowledge directions is the key to maintaining task ability, while variation of energy distribution has a relatively minor effect. We leave implementation details in Appendix~\ref{sec:fig_imple}.

\noindent
\textbf{Directional Similarity for Post-merge Task-wise Performance.} 
Building upon the previous analysis, we argue that preserving the directions of knowledge components is the key factor for retaining each task's performance during model merging. To quantify this, one might consider computing the directional similarity (\(\mathrm{DirSim}\)) between each task vector \(\Delta \boldsymbol{W}_i\) and the merged task vector \(\Delta \widetilde{\boldsymbol{W}}\). 
However, \(\Delta \widetilde{\boldsymbol{W}}\) aggregates multiple task vectors, introducing directional redundancy that artificially deflates  \(\mathrm{DirSim}\).

To address this, we consider the task-specific activation of the merged vector by projecting it onto the low-rank feature subspace.
Empirical evidence shows that features from a given task $i$ collapse into a low-rank subspace \(\boldsymbol{U}_i\)~\cite{delétang2024languagemodelingcompression, franceschelli2025trainingfoundationmodelsdata}. We thus project the merged multi-task vector onto the low-rank feature subspace to acquire the task-activated part:
\begin{equation}
\label{eq:dirsim_formulation}
\Delta \widetilde{\boldsymbol{W}}_{i} \;=\; \boldsymbol{U}_i \boldsymbol{U}_i^\top \, \Delta \widetilde{\boldsymbol{W}}.
\end{equation}
We then compute \(\mathrm{DirSim}(\Delta \boldsymbol{W}_i, \Delta \widetilde{\boldsymbol{W}}_{i})\) as our metric for task-wise capability retention.


In practice, the exact low-rank feature subspace \(\boldsymbol{U}_i\) is inaccessible and we approximate it using the subspace spanned by the left singular vectors of \(\Delta \boldsymbol{W}_i\). As shown in Figure~\ref{fig_metric_acc}, the projected \(\mathrm{DirSim}\) computed in this manner exhibits a clear monotonic relationship with the normalized accuracy improvement of merged models across different datasets and methods. 
Points with higher projected \(\mathrm{DirSim}\) correspond to better task knowledge retention. 
This confirms that the proposed projected \(\mathrm{DirSim}\) can measure how well task-specific knowledge is preserved during model merging.



\section{The Proposed DC-Merge Approach}

\noindent
The goal of our method is to preserve the complete directional geometry of task vectors during model merging. To this end, we propose two modules, i.e., \textit{energy smoothing} and \textit{cover space merging}. 

\noindent
\textbf{Energy Smoothing for Balanced Knowledge Representations.}  
As discussed in Section~\ref{sec:Motivation},  
the imbalanced energy distribution of knowledge components biases the merging process, potentially ignoring the direction of weaker but semantically rich knowledge components.
To mitigate these issues, we balance the energy distribution of each task vector via energy smoothing before merging.  
For each task vector $\boldsymbol{\Delta}_i$, we consider its knowledge decomposition $    \boldsymbol{\Delta}_i = \sum_{j=1}^{r} \sigma_i^j \boldsymbol{u}_i^j\boldsymbol{v}_i^{j \top}$,
where $\boldsymbol{\sigma}_i = (\sigma_i^1, \dots, \sigma_i^r)$ contains the singular values sorted in descending order.  
Instead of directly using the original energy distribution $\boldsymbol{\sigma}_i$, we replace it with a smoothed version $\overline{\boldsymbol{\sigma}}_i$ to redistribute the energy more evenly across the top-$r$ components, alleviating dominance on a few knowledge components.

For example, we consider the simple yet effective form of smoothing by replacing all top-$r$ singular values with their mean: $\overline{\boldsymbol{\sigma}}_{i} = \left(\frac{1}{r}\sum_{j=1}^{r} \sigma_{i}^{j}\right) \mathbf{1}_{r}$, 
which equalizes the contribution of all retained knowledge components. For completeness, additional smoothing strategies are discussed in Appendix~\ref{sec:smoothing}.
We then perform merging on these energy-balanced task vectors rather than the original ones.

\begin{algorithm}[t]
\caption{DC-Merge}
\label{alg:ioa}
\begin{algorithmic}[1]
\State \textbf{Input:} Task vectors $\{\boldsymbol{\Delta}_i\}_{i=1}^T$ with $\boldsymbol{\Delta}_i \in \mathbb{R}^{m \times n}$
\State \textbf{Output:} Merged multi-task vector $\widetilde{\boldsymbol{\Delta}}$

\State {\color{cvprblue}{$\triangleright$ \textbf{Step 1:} Construct cover space for all task vectors}}
\For{$i = 1 \to T$}
  \State Compute $r$-rank SVD: $\boldsymbol{\Delta}_i \approx \boldsymbol{U}_i^{(r)} \boldsymbol{\Sigma}_i^{(r)} \boldsymbol{V}_i^{(r)\top}$
  \State {Smoothing} $\boldsymbol{\Sigma}_i^{(r)}$ by 
  $\overline{\boldsymbol{\Sigma}}_i^{(r)} = \mathrm{diag}([\overline{\sigma}_i^1;\overline{\sigma}_i^2;\dots;\overline{\sigma}_i^r])$
  \State {Reconstruct} $\overline{\boldsymbol{\Delta}}_i = 
  \boldsymbol{U}_i^{(r)} \overline{\boldsymbol{\Sigma}}_i^{(r)} \boldsymbol{V}_i^{(r)\top}$
\EndFor
\State Obtain concatenated basis $\boldsymbol{U}, \boldsymbol{V}$ by Eq.~\eqref{eq:concat}
\State {Whitening} $\boldsymbol{U}$ and $\boldsymbol{V}$ respectively to obtain $\widetilde{\boldsymbol{U}}$ and $\widetilde{\boldsymbol{V}}$
\State {\color{cvprblue}{ $\triangleright$ \textbf{Step 2:} Project $\overline{\boldsymbol{\Delta}}_i$ onto cover space and merge}}
\State Project each $\overline{\boldsymbol{\Delta}}_i$ and obtain $\boldsymbol{M}_i$ by Eq.~\eqref{eq:proj_first}
\State $\widetilde{\boldsymbol{M}} \gets \text{Merging}(\{\boldsymbol{M}_i\}_{i=1}^{T})$ via TA or TIES

\State {\color{cvprblue}{$\triangleright$ \textbf{Step 3:} Project $\widetilde{\boldsymbol{M}}$ back to parameter space}}
\State {Construct} mask $\boldsymbol{\mathcal{M}} \gets \mathrm{block\mbox{-}diag}(\mathbf{1}_{r\times r}, \cdots, \mathbf{1}_{r\times r})$
\State Obtain merged task vector $\widetilde{\boldsymbol{\Delta}}$ by Eq.~\eqref{eq:proj_para}
\end{algorithmic}
\end{algorithm}

\noindent
\textbf{Projection and Merging in the Cover Space.}
Directly merging task vectors in the original parameter space may distort directional geometry due to misaligned subspaces $\{(\boldsymbol{U}_i, \boldsymbol{V}_i)\}$. 
To preserve directional consistency across all tasks, we seek a pair of shared orthonormal basis $(\widetilde{\boldsymbol{U}},\widetilde{\boldsymbol{V}})$ that define a \emph{cover space} capturing the directional geometry of all task vectors, and perform merging within the cover space. The objective can be formulated as:
\begin{equation}
\label{eq:dir_objective_final}
\begin{aligned}
\min_{\widetilde{\boldsymbol{U}},\widetilde{\boldsymbol{V}}} 
\sum_{i=1}^{T} \sum_{j=1}^{r} 
\min_{\boldsymbol{\sigma}_{i}^{j} \in \mathbb{R}^{k}}
\left\|
\boldsymbol{u}_{i}^{j}\boldsymbol{v}_{i}^{j\top}
- \widetilde{\boldsymbol{U}}\,\mathrm{diag}(\boldsymbol{\sigma}_{i}^{j})\,\widetilde{\boldsymbol{V}}^{\top}
\right\|_{F}^{2}, \\
\text{s.t.}\;\;
\widetilde{\boldsymbol{U}}^{\top}\widetilde{\boldsymbol{U}}
=\widetilde{\boldsymbol{V}}^{\top}\widetilde{\boldsymbol{V}}
=\boldsymbol{I}.
\end{aligned}
\end{equation}
where $k = rT$. The inner minimization over $\boldsymbol{\sigma}_{i}^{j}$ determines the optimal coefficients along the cover basis. Using Proposition~\ref{prop:equivalences}(a), we obtain the surrogate objective:
\begin{equation}
\label{eq:sur-obj}
\begin{aligned}
\max_{\widetilde{\boldsymbol{U}},\widetilde{\boldsymbol{V}}} 
\sum_{i=1}^{T}\sum_{j=1}^{r}
\big\|
\boldsymbol{\sigma}(\widetilde{\boldsymbol{U}},\widetilde{\boldsymbol{V}},\boldsymbol{u}_{i}^{j} \boldsymbol{v}_{i}^{j \top})
\big\|_{2}^{2}, \\
\text{s.t.}\;
\widetilde{\boldsymbol{U}}^{\top}\widetilde{\boldsymbol{U}}=
\widetilde{\boldsymbol{V}}^{\top}\widetilde{\boldsymbol{V}}=\boldsymbol{I},
\end{aligned}
\end{equation}
where
$
\boldsymbol{\sigma}(\boldsymbol{U},\boldsymbol{V},\boldsymbol{\Delta})
=\mathrm{diag}\!\left(\boldsymbol{U}^{\top}\boldsymbol{\Delta}\boldsymbol{V}\right)\in\mathbb{R}^k
$
denotes projection onto the shared dyadic directions.


\begin{table*}[!t]
  \resizebox{1\textwidth}{!}{
    \begin{tabular}{cccc|ccc|ccc}
      \toprule
      \multicolumn{1}{c}{\multirow{2}{*}[-0.5ex]{\textbf{Method}}}         & \multicolumn{3}{c}{\model{ViT-B-32}}                     & \multicolumn{3}{c}{\model{ViT-B-16}}                     & \multicolumn{3}{c}{\model{ViT-L-14}}                    \\ 
                                               \cmidrule{2-10}
                                               & 8 tasks           & 12 tasks          & 16 tasks           & 8 tasks           & 12 tasks          & 16 tasks          & 8 tasks           & 12 tasks          & 16 tasks          \\ 
                                               \midrule
      \multicolumn{1}{c}{Individual}            & $87.82$ & $88.90$ & $87.50$ & $89.71$ & $90.76$ & $89.11$ & $92.36$ & $93.57$ & $92.11$ \\
      \multicolumn{1}{c}{Task Arithmetic}     & $52.80_{(61.73)}$ & $60.76_{(69.12)}$ & $60.04_{(68.94)}$ & $57.70_{(65.30)}$ & $64.26_{(71.12)}$ & $62.40_{(70.16)}$ & $68.29_{(74.38)}$ & $73.69_{(78.84)}$ & $69.98_{(75.58)}$ \\
      \multicolumn{1}{c}{KnOTS-TIES}    & $55.93_{(65.03)}$ & $63.03_{(71.43)}$ & $61.78_{(70.78)}$ & $60.80_{(68.59)}$ & $66.35_{(73.38)}$ & $64.31_{(72.29)}$ & $73.61_{(79.92)}$ & $75.64_{(80.90)}$ & $72.19_{(77.98)}$ \\
      \multicolumn{1}{c}{WUDI-Merging}        & $55.25_{(64.38)}$ & $62.20_{(70.64)}$ & $61.24_{(70.27)}$ & $58.95_{(66.63)}$ & $65.29_{(72.29)}$ & $64.59_{(72.51)}$ & $69.78_{(75.91)}$ & $74.25_{(79.45)}$ & $71.79_{(77.59)}$ \\
      \multicolumn{1}{c}{TSV-M} & $58.91_{(68.25)}$ & $65.30_{(73.86)}$ & $63.51_{(72.72)}$ & $62.97_{(70.87)}$ & $68.92_{(76.06)}$ & $67.21_{(75.35)}$ & $76.52_{(83.00)}$ & $79.67_{(85.13)}$ & $74.37_{(80.31)}$ \\
      \multicolumn{1}{c}{Iso-CTS} & $\underline{63.01_{(72.71)}}$ & $\underline{66.28_{(75.01)}}$ & $\underline{64.61_{(74.02)}}$ & $\underline{69.06_{(77.38)}}$ & $\underline{71.52_{(78.87)}}$ & $\underline{69.88_{(78.22)}}$ & $\underline{81.64_{(88.31)}}$ & $\underline{81.35_{(86.87)}}$ & $\underline{77.50_{(83.65)}}$ \\
      \bottomrule
      \rowcolor{gray!15}
      \multicolumn{1}{c}{\textbf{ DC-Merge}} & $\mathbf{64.17_{(73.90)}}$ & $\mathbf{68.40_{(77.22)}}$ & $\mathbf{66.27_{(75.80)}}$ & $\mathbf{70.53_{(78.86)}}$ & $\mathbf{73.12_{(80.56)}}$ & $\mathbf{70.57_{(78.91)}}$ & $\mathbf{82.61_{(89.42)}}$ & $\mathbf{83.62_{(89.31)}}$ & $\mathbf{79.53_{(85.71)}}$ \\
      \bottomrule
    \end{tabular}
  }
  \caption{Average absolute accuracy results on vision model merging benchmarks in LoRA setting; subscript (in parentheses) is the average normalized accuracy. The best results are in \textbf{bold} and the second-best are \underline{underlined}.}
  \label{tab:LoRA_task_acc}
\end{table*}

\begin{table*}[t]
  \resizebox{1\textwidth}{!}{
    \begin{tabular}{cccc|ccc|ccc}
      \toprule
      \multicolumn{1}{c}{\multirow{2}{*}[-0.5ex]{\textbf{Method}}}         & \multicolumn{3}{c}{\model{ViT-B-32}}                     & \multicolumn{3}{c}{\model{ViT-B-16}}                     & \multicolumn{3}{c}{\model{ViT-L-14}}                    \\ 
                                               \cmidrule{2-10}
                                               & 8 tasks           & 14 tasks          & 20 tasks           & 8 tasks           & 14 tasks          & 20 tasks          & 8 tasks           & 14 tasks          & 20 tasks          \\ 
                                               \midrule
      \multicolumn{1}{c}{Individual}            & $92.83$ & $90.88$ & $91.37$ & $94.64$ & $92.76$ & $93.17$ & $95.81$ & $94.29$ & $94.73$ \\
      \multicolumn{1}{c}{Weight Averaging}    & $66.34_{(72.13)}$ & $64.34_{(71.12)}$ & $61.04_{(67.53)}$ & $72.22_{(76.60)}$ & $69.46_{(74.82)}$ & $65.31_{(70.36)}$ & $79.56_{(83.15)}$ & $76.73_{(81.10)}$ & $71.60_{(75.60)}$ \\
      \multicolumn{1}{c}{Task Arithmetic}     & $70.79_{(76.55)}$ & $65.32_{(72.09)}$ & $60.52_{(66.79)}$ & $75.41_{(79.58)}$ & $70.52_{(75.89)}$ & $65.78_{(70.76)}$ & $84.93_{(88.65)}$ & $79.41_{(83.95)}$ & $74.01_{(78.07)}$ \\
      \multicolumn{1}{c}{TIES-Merging} & $75.09_{(81.08)}$ & $68.02_{(74.83)}$ & $63.38_{(69.90)}$ & $79.74_{(84.34)}$ & $73.22_{(78.73)}$ & $68.18_{(73.26)}$ & $86.88_{(90.69)}$ & $79.46_{(84.05)}$ & $75.71_{(79.80)}$ \\
      \multicolumn{1}{c}{Consensus TA}        & $75.03_{(80.84)}$ & $70.39_{(77.36)}$ & $65.43_{(71.98)}$ & $79.39_{(83.86)}$ & $74.39_{(79.92)}$ & $69.76_{(74.93)}$ & $86.34_{(90.08)}$ & $82.22_{(86.94)}$ & $79.00_{(83.22)}$ \\
      \multicolumn{1}{c}{TSV-M} & $85.86_{(92.31)}$ & $80.06_{(87.88)}$ & $77.07_{(84.29)}$ & $89.01_{(93.94)}$ & $84.58_{(91.01)}$ & $80.57_{(86.45)}$ & $92.98_{(96.98)}$ & $89.17_{(94.43)}$ & $87.72_{(92.50)}$ \\
      \multicolumn{1}{c}{Iso-CTS}        & $\underline{86.20_{(91.78)}}$ & $\underline{81.71_{(89.70)}}$ & $\underline{78.05_{(85.48)}}$ & $\mathbf{90.91_{(95.95)}}$ & $\underline{86.40_{(92.81)}}$ & $\underline{82.38_{(88.36)}}$ & $\mathbf{94.69_{(98.81)}}$ & $\underline{90.98_{(96.28)}}$ & $\underline{90.05_{(94.88)}}$ \\
      \bottomrule
      \rowcolor{gray!15}
      \multicolumn{1}{c}{\textbf{ DC-Merge}} & $\mathbf{87.05_{(93.55)}}$ & $\mathbf{82.52_{(90.62)}}$ & $\mathbf{80.58_{(88.18)}}$ & $\underline{90.78_{(95.83)}}$ & $\mathbf{87.06_{(93.70)}}$ & $\mathbf{84.57_{(90.76)}}$ & $\underline{94.31_{(98.38)}}$ & $\mathbf{91.01_{(96.43)}}$ & $\mathbf{90.51_{(95.43)}}$ \\
      \bottomrule
    \end{tabular}
  }
  \caption{Average absolute accuracy results on vision model merging benchmarks in FFT setting; subscript (in parentheses) is the average normalized accuracy. The best results are in \textbf{bold} and the second-best are \underline{underlined}.}
  \label{tab:FFT_task_acc}
  \vspace{-0.3cm}
\end{table*}

As directly optimizing Eq.~\eqref{eq:sur-obj} incurs non-trivial computational overhead, we adopt whitening~\cite{schonemann1966} here as it serves as a near-optimal solution to Eq.~\eqref{eq:sur-obj} while being computationally efficient. In Appendix~\ref{app:D}, we provide an iterative approach for constructing cover basis and theoretically show its relation to the whitening transformation.
Specifically, we construct the cover basis $(\widetilde{\boldsymbol{U}}, \widetilde{\boldsymbol{V}})$ by whitening the column-wise concatenated per-task knowledge basis:
\begin{equation}
\label{eq:concat}
\boldsymbol{U} = [\boldsymbol{U}_1^{(r)}, \ldots, \boldsymbol{U}_T^{(r)}], \qquad
\boldsymbol{V} = [\boldsymbol{V}_1^{(r)}, \ldots, \boldsymbol{V}_T^{(r)}].
\end{equation}
Thus, $\widetilde{\boldsymbol{U}}^{\top}\widetilde{\boldsymbol{U}} = \widetilde{\boldsymbol{V}}^{\top}\widetilde{\boldsymbol{V}} = \boldsymbol{I}$ defines an orthogonal basis that contains the union of all tasks' directional geometry.  
Each smoothed task vector is then projected onto cover space by:
\begin{equation}
\label{eq:proj_first}
    \boldsymbol{M}_i = \widetilde{\boldsymbol{U}}^{\top}\,\overline{\boldsymbol{\Delta}}_i\,\widetilde{\boldsymbol{V}}.
\end{equation}
This projection ensures that all task vectors are expressed under shared cover basis, which facilitates directionally consistent task vectors aggregation via existing element-wise merging methods, such as TA~\cite{ilharco2023task} and TIES-Merging~\cite{yadav2023tiesmerging}, to obtain $\widetilde{\boldsymbol{M}}$.
Finally, the merged multi-task vector is reconstructed by projecting $\widetilde{\boldsymbol{M}}$ back to the original parameter space:
\begin{equation}
\label{eq:proj_para}
\widetilde{\boldsymbol{\Delta}} = \widetilde{\boldsymbol{U}}\,(\widetilde{\boldsymbol{M}}\odot\boldsymbol{\mathcal{M}})\,\widetilde{\boldsymbol{V}}^{\top},
\end{equation}
where $\boldsymbol{\mathcal{M}}$ serves as a structural mask. We leave further discussion of the structural mask in Appendix~\ref{app:D} and summarize the key steps of our approach in Algorithm \ref{alg:ioa}.

\section{Experiments}
In this section, we evaluate the performance of DC-Merge against existing baselines through extensive experiments using vision models and vision-language models (VLMs) in both FFT and LoRA settings, demonstrating the versatility of DC-Merge. 
We further perform comprehensive ablation studies to analyze the effectiveness of each key component in DC-Merge.
\subsection{Results for Vision Tasks}
In this subsection, we investigate the merging of vision models. For fully fine-tuned vision models, following prior works~\cite{tsv, marczak2025notaskleftbehind}, we use 8-task, 14-task, and 20-task benchmarks for evaluation, respectively, and employ three CLIP~\cite{radford2021learning} variants: \model{ViT-B-32}, \model{ViT-B-16}, and \model{ViT-L-14} as visual encoders~\cite{dosovitskiy2021imageworth16x16words}. For LoRA fine-tuned vision models, we extend previous evaluations by assessing both our method and existing baselines on a larger number of tasks, specifically 8, 12 and 16. Consistent with prior work~\cite{tsv, marczak2025notaskleftbehind, stoica2024knots, cheng2025whoever}, we report the average absolute and normalized accuracy of merged models.

In the full parameter fine-tuning setting, we compare our method against Weight Averaging~\cite{wortsman2022model}, Task Arithmetic~\cite{ilharco2023task}, TIES-Merging~\cite{yadav2023tiesmerging}, Consensus TA~\cite{wang2024localizing}, TSV-M~\cite{tsv} and Iso-CTS~\cite{marczak2025notaskleftbehind}. For LoRA fine-tuned models, Task Arithmetic, KnOTS-TIES~\cite{stoica2024knots}, WUDI-Merging~\cite{cheng2025whoever}, TSV-M, and Iso-CTS serve as baselines. We provide details on benchmarks and experimental setups in the Appendix~\ref{app:E}.

Table~\ref{tab:LoRA_task_acc} shows the results of merging vision models fine-tuned by LoRA. Across three different backbones, the performance of our method consistently surpasses the current state-of-the-art methods. Moreover, the performance gains remain substantial with the growth of tasks.
We also conduct experiments on the checkpoints provided by KnOTS~\cite{stoica2024knots}, where our method still achieves superior performance compared to existing state-of-the-art methods. The corresponding results are reported in Appendix~\ref{sec:other_ckpts}. Table~\ref{tab:FFT_task_acc} presents the results under the full fine-tuning setting. The results demonstrate that our approach not only exhibits strong capability when merging LoRA fine-tuned models but also achieves state-of-the-art performance under the FFT scenario. Notably, the superiority of our method becomes more significant as the number of tasks increases. 

\subsection{Results for Vision-Language Tasks}
In the multi-modal model merging setting, we compare our method with Task Arithmetic, TIES-Merging, DARE~\cite{yu2024language}, PCB-Merging~\cite{du2024parameter}, and RobustMerge~\cite{zeng2025parameter} on eight multi-modal datasets using \model{LLaVA-v1.5-7B}~\cite{liu2023visual} as backbone. Following the experimental setup of RobustMerge, we further evaluate the merged model on four additional datasets to assess its generalization ability to unseen tasks. We adopt the checkpoints released by RobustMerge and provide detailed experimental configurations in Appendix~\ref{app:E}.

As shown in Table \ref{tab:8_pertask_MM}, our method notably outperforms existing state-of-the-art methods on both seen and unseen tasks, demonstrating that its applicability is not limited to vision models but can also scale to large multi-modal models. We also evaluate the generalization capability to unseen tasks of our method on vision models and the detailed results are presented in Appendix~\ref{sec:vision_unseen_tasks}.

\begin{table}[h]
\centering
\resizebox{\columnwidth}{!}{%
\tablestyle{14pt}{1}
\begin{tabular}{c|cc}
\toprule
\textbf{Method} & \model{Seen Tasks} & \model{Unseen Tasks} \\
\midrule
Zeroshot & $43.37$ & $25.22$ \\
Individual & $69.23$ & $-$ \\
Multi Task  & $63.62$ & $36.06$ \\
\midrule
Task Arithmetic & $53.93$ & $33.31$ \\
DARE-Merging  & $53.84$ & $33.15$ \\
TIES-Merging         & $53.09$ & $33.14$ \\
PCB-Merging             & $53.70$ & $33.53$ \\
RobustMerge            & $\underline{57.33}$ & $\underline{37.99}$ \\
\midrule
\rowcolor{gray!15}
\textbf{DC-Merge}       & $\mathbf{59.63}$ & $\mathbf{39.84}$ \\
\bottomrule
\end{tabular}
}
\caption{Performance on MM-MergeBench~\cite{zeng2025parameter}, containing eight seen tasks (LoRA fine-tuned) and four unseen tasks. The best results are in \textbf{bold} and the second-best are \underline{underlined}. We report average absolute accuracy. See Appendix~\ref{sec:other_baselines_MMBench} for detailed results.}
\label{tab:8_pertask_MM}
\vspace{-0.4cm}
\end{table}

\subsection{Ablations and Analysis}
Unless otherwise specified, all experiments in this subsection are conducted in LoRA setting. 

\noindent
\textbf{The Effectiveness of Energy Smoothing.}
We investigate the impact of our energy smoothing strategy on the performance of our method, with the results summarized in Table~\ref{lora_b32_smoothing_strategy}. The results align well with our observations: applying energy smoothing to each task vector effectively ensures that all the knowledge components can be adequately expressed, leading to a significant improvement in overall performance. Notably, preserving a moderate degree of skewness in the energy distribution (i.e., linear smoothing) can yield better results than averaging. We provide additional comparisons of smoothing strategy on \model{ViT-B-16} and \model{ViT-L-14} in Appendix~\ref{sec:smoothing}.

\vspace{-0.2cm}
\begin{table}[H]
\centering
\resizebox{\columnwidth}{!}{%
\begin{tabular}{lccc}
\toprule
Method & 8 tasks & 12 tasks & 16 tasks \\ 
\midrule
No smoothing        & 69.27 & 74.60 & 74.47 \\
Averaging           & 73.09 (\textcolor{green!50!black}{+3.82}) & 76.42 (\textcolor{green!50!black}{+1.82}) & 75.51 (\textcolor{green!50!black}{+1.04}) \\
Linear smoothing  & 73.90 (\textcolor{green!50!black}{+4.63}) & 77.22 (\textcolor{green!50!black}{+2.62}) & 75.80 (\textcolor{green!50!black}{+1.33}) \\

\bottomrule
\end{tabular}%
}
\caption{Performance comparison of different smoothing strategies. We report average normalized accuracy using \model{ViT-B-32}.}
\label{lora_b32_smoothing_strategy}
\vspace{-0.4cm}
\end{table}

\noindent
\textbf{Impact of Performing a Post-hoc Pruning.}
In Algorithm \ref{alg:ioa}, we perform a post-hoc pruning by applying a mask $\boldsymbol{\mathcal{M}}$ to mitigate directional inconsistency of different tasks before projecting the merged parameter matrix $\widetilde{\boldsymbol{M}}$ back to the original parameter space. Table~\ref{b32_with_and_wo_mask} presents the effect of such structural pruning on overall performance in both LoRA and FFT settings. The performance degradation becomes more pronounced with the increase of tasks. Moreover, since the number of fine-tuned parameters in the FFT setting is substantially larger than that in LoRA, incorporating masks leads to significant performance gains of up to 10.55\% in average normalized accuracy, highlighting the crucial role that structural pruning plays in preventing cross-task directional inconsistency. We investigate the impact of mask size on the performance in Appendix~\ref{app:D}.

\begin{table}[htbp]
\centering
\resizebox{\columnwidth}{!}{%
\tablestyle{12pt}{1}
\begin{tabular}{lccc}
\toprule
Method & Tasks & w/o mask & w/ mask\\
\midrule
\multirow{3}{*}{FFT}
  & 8 tasks  & 87.98 & 93.55 (\textcolor{green!50!black}{+5.57})\\
  & 14 tasks & 82.39 & 90.50 (\textcolor{green!50!black}{+8.11})\\
  & 20 tasks & 77.63 & 88.18 (\textcolor{green!50!black}{+10.55})\\
\midrule
\multirow{3}{*}{LoRA}
  & 8 tasks  & 73.61 & 73.90 (\textcolor{green!50!black}{+0.29})\\
  & 12 tasks & 75.94 & 77.22 (\textcolor{green!50!black}{+1.28})\\
  & 16 tasks & 74.42 & 75.80 (\textcolor{green!50!black}{+1.38}) \\
\bottomrule
\end{tabular}%
}
\caption{Comparison of performance with and w/o applying masks. We report average normalized accuracy using \model{ViT-B-32}.}
\label{b32_with_and_wo_mask}
\vspace{-0.3cm}
\end{table}

\noindent
\textbf{Impact of Merging in the Shared Cover space.}
To maintain the directional geometry of each task vector, we project the smoothed task vectors onto a shared subspace prior to model merging. As shown in Table~\ref{lora_b32_ta_ties}, compared to merging in the original parameter space, CSM significantly boosts the performance of both TA~\cite{ilharco2023task} and TIES~\cite{yadav2023tiesmerging}. Moreover, after applying energy smoothing to the task vectors, the performance is further enhanced, indicating that the two main components of our proposed method are complementary. An illustrative example provided in Appendix~\ref{app:D} further demonstrates the importance of shared cover basis in preserving the directional structure of task vectors.

\begin{table}[htbp]
\centering
\resizebox{\columnwidth}{!}{%
\begin{tabular}{lccc}
\toprule
 Method & 8 tasks & 12 tasks & 16 tasks \\ 
\midrule
Vanilla TA        & 61.73 & 69.12 & 68.94 \\
TA + ES  & 69.12 (\textcolor{green!50!black}{+7.39}) & 74.35 (\textcolor{green!50!black}{+5.23}) & 73.01 (\textcolor{green!50!black}{+4.07}) \\
TA + CSM  & 68.13 (\textcolor{green!50!black}{+6.40}) & 73.92 (\textcolor{green!50!black}{+4.80}) & 72.64 (\textcolor{green!50!black}{+3.70}) \\
TA + CSM + ES           & 73.82 (\textcolor{green!50!black}{+12.09}) & 77.16 (\textcolor{green!50!black}{+8.04}) & 75.73 (\textcolor{green!50!black}{+6.79}) \\
\midrule
Vanilla TIES        & 62.09 & 69.30 & 70.06 \\
TIES + ES  & 69.94 (\textcolor{green!50!black}{+7.85}) & 74.97 (\textcolor{green!50!black}{+5.67}) & 74.74 (\textcolor{green!50!black}{+4.68}) \\
TIES + CSM  & 69.27 (\textcolor{green!50!black}{+7.18}) & 74.60 (\textcolor{green!50!black}{+5.30}) & 74.47 (\textcolor{green!50!black}{+4.41}) \\
TIES + CSM + ES           & 73.90 (\textcolor{green!50!black}{+11.81}) & 77.22 (\textcolor{green!50!black}{+7.92}) & 75.80 (\textcolor{green!50!black}{+5.74}) \\
\bottomrule
\end{tabular}%
}
\caption{Performance of individually applying energy smoothing (ES) and cover space merging (CSM) as well as combining them to TA or TIES compared with vanilla settings. We report the average normalized accuracy using \model{ViT-B-32}.}
\label{lora_b32_ta_ties}
\vspace{-0.5cm}
\end{table}

\section{Related Work}

\noindent
\textbf{Model merging} has emerged as a promising approach to integrate expert models fine-tuned on different downstream tasks into a single multi-task model. Task Arithmetic (TA)~\cite{ilharco2023task} first introduces the concept of a \emph{task vector}, defined as the difference between an expert and its pre-trained model, and combines them through scaled averaging to construct a merged model. Subsequent studies propose meticulously crafted parameter-wise strategies to mitigate interference during merging. TIES~\cite{yadav2023tiesmerging} reduces sign conflicts by adopting the majority sign across all models. Consensus Merging~\cite{wang2024localizing} applies binary masks to exclude parameters important to fewer than two tasks. Recent studies WUDI-Merging~\cite{cheng2025whoever} and FDA~\cite{shi2025modelmergingfunctionaldual} optimizes the merged task vector to keep the output of merged model align with each fine-tuned model given the same input of corresponding task.

These merging methods are data-free, producing merged task vectors that can be directly integrated into the pre-trained model. A number of recent approaches, however, focus on creating model with multi-task capabilities by modifying the inference stage. Twin-Merging~\cite{lu2024twin0merging0} composes task-specific components at test time, requiring two forward passes. EMR-Merging~\cite{huang2024emr0merging0} employs additional per-task masks and rescalers for inference. In this paper, we restrict our study to merging methods which are data-free and leave the inference stage unaffected.

\noindent
\textbf{SVD-based Model Merging.}
Recent data-free model merging methods have incorporated SVD to improve performance~\cite{stoica2024knots, wei2025modeling}. 
State-of-the-art approaches include TSV-M~\cite{tsv}, which enforces orthogonality between task-specific subspaces to reduce task interference, and Iso-CTS~\cite{marczak2025notaskleftbehind}, which standardizes singular values after combining a common subspace constructed by TA~\cite{ilharco2023task} and task-specific subspaces. More recently, ESM~\cite{li2026modelmergingessentialsubspace} projects parameter updates into an activation-aware essential subspace and applies polarized scaling to amplify critical weights.
In contrast to these methods, our approach prioritizes the directional consistency of each original task vector with the merged vector. We achieve this by balancing the energy distribution of knowledge components and performing the merge process within a shared orthogonal subspace induced by a pair of cover basis.


\section{Conclusion and Limitation}


\textbf{Conclusion.} In this work, we are the first to identify that preserving the directional consistency of task vectors after merging is crucial for retaining the capabilities of individual tasks. Building upon this insight, we propose DC-Merge, which maintains the directional consistency
between the merged multi-task vector and each original task vector by \textit{energy smoothing} and \textit{cover space merging}. Our method achieves state-of-the-art performance in both FFT and LoRA settings. 

\noindent
\textbf{Limitation.}
There still exists a noticeable performance gap between merging LoRA fine-tuned models and full parameter fine-tuned models. This phenomenon may arise from the number of knowledge components in each task vector. A larger number of knowledge components provides redundancy that is robust to direction shift, whereas a smaller set makes the task vector more fragile to directional inconsistency. The knowledge components of each LoRA task vector are scarce, even fewer than the LoRA rank due to its long-tailed energy distribution. A potential remedy lies in promoting a balanced energy distribution of the knowledge components during fine-tuning. 





{
    \small
    \bibliographystyle{ieeenat_fullname}
    \bibliography{main}

@String(IJCV = {Int. J. Comput. Vis.})

@String(CVPR= {IEEE Conf. Comput. Vis. Pattern Recog.})

@String(ICCV= {Int. Conf. Comput. Vis.})

@String(ECCV= {Eur. Conf. Comput. Vis.})

@String(NIPS= {Adv. Neural Inform. Process. Syst.})

@String(ICLR = {Int. Conf. Learn. Represent.})

@String(IJCV  = {IJCV})

@String(CVPR  = {CVPR})

@String(ICCV  = {ICCV})

@String(ECCV  = {ECCV})

@String(NIPS  = {NeurIPS})

@String(ICLR  = {ICLR})

@article{shi2025modelmergingfunctionaldual,
      title={Model Merging with Functional Dual Anchors}, 
      author={Kexuan Shi and Yandong Wen and Weiyang Liu},
      year={2025},
      journal={arXiv preprint arXiv:2510.21223},
}

@inproceedings{wortsman2022robust,
  author       = {Mitchell Wortsman and
                  Gabriel Ilharco and
                  Jong Wook Kim and
                  Mike Li and
                  Simon Kornblith and
                  Rebecca Roelofs and
                  Raphael Gontijo Lopes and
                  Hannaneh Hajishirzi and
                  Ali Farhadi and
                  Hongseok Namkoong and
                  Ludwig Schmidt},
  title        = {Robust fine-tuning of zero-shot models},
  booktitle    = CVPR,
  year         = {2022},
}

@inproceedings{ilharco2023task,
  author       = {Gabriel Ilharco and
                  Marco T{\'{u}}lio Ribeiro and
                  Mitchell Wortsman and
                  Ludwig Schmidt and
                  Hannaneh Hajishirzi and
                  Ali Farhadi},
  title        = {Editing models with task arithmetic},
  booktitle    = ICLR,
  year         = {2023}
}

@inProceedings{pmlr-v97-lezcano-casado19a,
  title = 	 {Cheap Orthogonal Constraints in Neural Networks: A Simple Parametrization of the Orthogonal and Unitary Group},
  author =       {Lezcano-Casado, Mario and Mart\'{\i}nez-Rubio, David},
  booktitle = {ICML},
  year = 	 {2019},
}

@inproceedings{
      yadav2023tiesmerging,
      title={{TIES}-Merging: Resolving Interference When Merging Models},
      author={Prateek Yadav and Derek Tam and Leshem Choshen and Colin Raffel and Mohit Bansal},
      booktitle=NIPS,
      year={2023},
}

@inproceedings{ilharco2022patching,
  title   = {Patching open-vocabulary models by interpolating weights},
  author  = {Ilharco, Gabriel and Wortsman, Mitchell and Gadre, Samir Yitzhak and Song, Shuran and Hajishirzi, Hannaneh and Kornblith, Simon and Farhadi, Ali and Schmidt, Ludwig},
  booktitle=NIPS,
  year    = {2022}
}

@article{krizhevsky2009learning,
  title={Learning multiple layers of features from tiny images},
  author={Krizhevsky, Alex and Hinton, Geoffrey and others},
  year={2009},
  journal={University of Toronto},
}

@inproceedings{cheng2025whoever,
  title={Whoever started the interference should end it: Guiding data-free model merging via task vectors},
  author={Cheng, Runxi and Xiong, Feng and Wei, Yongxian and Zhu, Wanyun and Yuan, Chun},
  booktitle={ICML},
  year={2025}
}

@article{li2026modelmergingessentialsubspace,
      title={Model Merging in the Essential Subspace}, 
      author={Longhua Li and Lei Qi and Qi Tian and Xin Geng},
      year={2026},
      journal={arXiv preprint arXiv:2602.20208},
}

@inproceedings{radford2021learning,
  title={Learning transferable visual models from natural language supervision},
  author={Radford, Alec and Kim, Jong Wook and Hallacy, Chris and Ramesh, Aditya and Goh, Gabriel and Agarwal, Sandhini and Sastry, Girish and Askell, Amanda and Mishkin, Pamela and Clark, Jack and others},
  booktitle={ICML},
  year={2021},
}

@inproceedings{cars,
  title={{3D Object representations for fine-grained categorization}},
  author={Krause, Jonathan and Stark, Michael and Deng, Jia and Fei-Fei, Li},
  booktitle={ICCV Workshops},
  year={2013}
}

@inproceedings{svhn,
  title={Reading digits in natural images with unsupervised feature learning},
  author={Netzer, Yuval and Wang, Tao and Coates, Adam and Bissacco, Alessandro and Wu, Bo and Ng, Andrew Y},
  booktitle={NeurIPS Workshops},
  year={2011}
}

@inproceedings{gtsrb,
  title={The German traffic sign recognition benchmark: a multi-class classification competition},
  author={Stallkamp, Johannes and Schlipsing, Marc and Salmen, Jan and Igel, Christian},
  booktitle={IJCNN},
  year={2011},
}

@article{eurosat,
  title={Eurosat: A novel dataset and deep learning benchmark for land use and land cover classification},
  author={Helber, Patrick and Bischke, Benjamin and Dengel, Andreas and Borth, Damian},
  journal={Journal of Selected Topics in Applied Earth Observations and Remote Sensing},
  year={2019},
}

@article{cheng2017remote,
  title={Remote sensing image scene classification: Benchmark and state of the art},
  author={Cheng, Gong and Han, Junwei and Lu, Xiaoqiang},
  journal={Proceedings of the IEEE},
  year={2017},
}

@article{sun397,
  title={Sun database: Exploring a large collection of scene categories},
  author={Xiao, Jianxiong and Ehinger, Krista A and Hays, James and Torralba, Antonio and Oliva, Aude},
  journal={IJCV},
  year={2016},
}

@inproceedings{dtd,
  title={Describing textures in the wild},
  author={Cimpoi, Mircea and Maji, Subhransu and Kokkinos, Iasonas and Mohamed, Sammy and Vedaldi, Andrea},
  booktitle=CVPR,
  year={2014}
}

@inproceedings{wortsman2022model,
  title={Model soups: averaging weights of multiple fine-tuned models improves accuracy without increasing inference time},
  author={Wortsman, Mitchell and Ilharco, Gabriel and Gadre, Samir Yitzhak and Roelofs, Rebecca and Gontijo-Lopes, Raphael and Morcos, Ari S and Namkoong, Hongseok and Farhadi, Ali and Carmon, Yair and Kornblith, Simon and others},
  booktitle={ICML},
  year={2022}
}

@inproceedings{DaheimMPGK24,
  author       = {Nico Daheim and
                  Thomas M{\"{o}}llenhoff and
                  Edoardo M. Ponti and
                  Iryna Gurevych and
                  Mohammad Emtiyaz Khan},
  title        = {Model Merging by Uncertainty-Based Gradient Matching},
  booktitle    = {ICLR},
  year         = {2024},
}

@article{DINO,
   title={Emerging Properties in Self-Supervised Vision Transformers},
   journal=ICCV,
   author={Caron, Mathilde and Touvron, Hugo and Misra, Ishan and Jegou, Herve and Mairal, Julien and Bojanowski, Piotr and Joulin, Armand},
   year={2021},
}

@inproceedings{hu2022lora,
    author = {Hu, Edward J and Shen, Yelong and Wallis, Phillip and Allen-Zhu, Zeyuan and Li, Yuanzhi and Wang, Shean and Wang, Lu and Chen, Weizhu and others},
    title = {Lora: Low-rank adaptation of large language models.},
    booktitle = {ICLR},
    year = {2022}
}

@inproceedings{dosovitskiy2021imageworth16x16words,
    author = {Alexey Dosovitskiy and Lucas Beyer and Alexander Kolesnikov and Dirk Weissenborn and Xiaohua Zhai and Thomas Unterthiner and Mostafa Dehghani and Matthias Minderer and Georg Heigold and Sylvain Gelly and Jakob Uszkoreit and Neil Houlsby},
    title = {An Image is Worth 16x16 Words: Transformers for Image Recognition at Scale.},
    booktitle = {ICLR},
    year = {2021}
}

@inproceedings{lu2022learn,
  title={Learn to explain: Multimodal reasoning via thought chains for science question answering},
  author={Lu, Pan and Mishra, Swaroop and Xia, Tanglin and Qiu, Liang and Chang, Kai-Wei and Zhu, Song-Chun and Tafjord, Oyvind and Clark, Peter and Kalyan, Ashwin},
  booktitle={NeurIPS},
  year={2022}
}

@inproceedings{goyal2017making,
  title={Making the v in vqa matter: Elevating the role of image understanding in visual question answering},
  author={Goyal, Yash and Khot, Tejas and Summers-Stay, Douglas and Batra, Dhruv and Parikh, Devi},
  booktitle={CVPR},
  year={2017}
}

@inproceedings{kazemzadeh2014referitgame,
  title={Referitgame: Referring to objects in photographs of natural scenes},
  author={Kazemzadeh, Sahar and Ordonez, Vicente and Matten, Mark and Berg, Tamara},
  booktitle={EMNLP},
  year={2014}
}

@inproceedings{mao2016generation,
  title={Generation and comprehension of unambiguous object descriptions},
  author={Mao, Junhua and Huang, Jonathan and Toshev, Alexander and Camburu, Oana and Yuille, Alan L and Murphy, Kevin},
  booktitle={CVPR},
  year={2016}
}

@inproceedings{deng2009imagenet,
  title={Imagenet: A large-scale hierarchical image database},
  author={Deng, Jia and Dong, Wei and Socher, Richard and Li, Li-Jia and Li, Kai and Fei-Fei, Li},
  booktitle={CVPR},
  year={2009},
}

@inproceedings{mishra2019ocr,
  title={Ocr-vqa: Visual question answering by reading text in images},
  author={Mishra, Anand and Shekhar, Shashank and Singh, Ajeet Kumar and Chakraborty, Anirban},
  booktitle={ICDAR},
  year={2019},
}

@inproceedings{plummer2015flickr30k,
  title={Flickr30k entities: Collecting region-to-phrase correspondences for richer image-to-sentence models},
  author={Plummer, Bryan A and Wang, Liwei and Cervantes, Chris M and Caicedo, Juan C and Hockenmaier, Julia and Lazebnik, Svetlana},
  booktitle={ICCV},
  year={2015}
}

@inproceedings{gurari2018vizwiz,
  title={Vizwiz grand challenge: Answering visual questions from blind people},
  author={Gurari, Danna and Li, Qing and Stangl, Abigale J and Guo, Anhong and Lin, Chi and Grauman, Kristen and Luo, Jiebo and Bigham, Jeffrey P},
  booktitle={CVPR},
  year={2018}
}

@article{lu2021iconqa,
  title={Iconqa: A new benchmark for abstract diagram understanding and visual language reasoning},
  author={Lu, Pan and Qiu, Liang and Chen, Jiaqi and Xia, Tony and Zhao, Yizhou and Zhang, Wei and Yu, Zhou and Liang, Xiaodan and Zhu, Song-Chun},
  journal={arXiv preprint arXiv:2110.13214},
  year={2021}
}

@inproceedings{hendrycks2021many,
  title={The many faces of robustness: A critical analysis of out-of-distribution generalization},
  author={Hendrycks, Dan and Basart, Steven and Mu, Norman and Kadavath, Saurav and Wang, Frank and Dorundo, Evan and Desai, Rahul and Zhu, Tyler and Parajuli, Samyak and Guo, Mike and others},
  booktitle={ICCV},
  year={2021}
}

@inproceedings{schwenk2022okvqa,
  title={A-okvqa: A benchmark for visual question answering using world knowledge},
  author={Schwenk, Dustin and Khandelwal, Apoorv and Clark, Christopher and Marino, Kenneth and Mottaghi, Roozbeh},
  booktitle={ECCV},
  year={2022},
}

@inproceedings{wang2021screen2words,
  title={Screen2words: Automatic mobile ui summarization with multimodal learning},
  author={Wang, Bryan and Li, Gang and Zhou, Xin and Chen, Zhourong and Grossman, Tovi and Li, Yang},
  booktitle={UIST},
  year={2021}
}

@article{lu2022dynamic,
  title={Dynamic prompt learning via policy gradient for semi-structured mathematical reasoning},
  author={Lu, Pan and Qiu, Liang and Chang, Kai-Wei and Wu, Ying Nian and Zhu, Song-Chun and Rajpurohit, Tanmay and Clark, Peter and Kalyan, Ashwin},
  journal={arXiv preprint arXiv:2209.14610},
  year={2022}
}

@inproceedings{carion2020endtoend,
  title   = {End-to-End Object Detection with Transformers},
  author  = {Nicolas Carion and Francisco Massa and Gabriel Synnaeve and Nicolas Usunier and Alexander Kirillov and Sergey Zagoruyko},
  year    = {2020},
  booktitle = {ECCV}
}

@inproceedings{zhai2023sigmoid,
  title   = {Sigmoid Loss for Language Image Pre-Training},
  author  = {Xiaohua Zhai and Basil Mustafa and Alexander Kolesnikov and Lucas Beyer},
  year    = {2023},
  booktitle = {ICCV}
}

@inproceedings{tsv,
      title={Task Singular Vectors: Reducing Task Interference in Model Merging}, 
      author={Antonio Andrea Gargiulo and Donato Crisostomi and Maria Sofia Bucarelli and Simone Scardapane and Fabrizio Silvestri and Emanuele Rodolà},
      year={2025},
      booktitle=CVPR,
}

@inproceedings{stoica2024knots,
  title   = {{Model merging with SVD to tie the Knots}},
  author  = {George Stoica and Pratik Ramesh and Boglarka Ecsedi and Leshem Choshen and Judy Hoffman},
  year    = {2025},
  booktitle = ICLR
}

@inproceedings{davari2023model,
  title     = {Model Breadcrumbs: Scaling Multi-Task Model Merging with Sparse Masks},
  author    = {Mohammad-Javad Davari and Eugene Belilovsky},
  booktitle   = {ECCV},
  year      = 2024,
}

@inproceedings{MarczakTTC24,
  author       = {Daniel Marczak and
                  Bartlomiej Twardowski and
                  Tomasz Trzcinski and
                  Sebastian Cygert},
  title        = {{MagMax:} {Leveraging} {Model} {Merging} for {Seamless} {Continual} {Learning}},
  booktitle    = ECCV,
  year         = {2024},
}

@article{MNIST,
  author={Lecun, Y. and Bottou, L. and Bengio, Y. and Haffner, P.},
  journal={Proceedings of the IEEE}, 
  title={Gradient-based learning applied to document recognition}, 
  year={1998},
}

@inproceedings{cohen_emnist_2017,
	title = {{EMNIST}: {Extending} {MNIST} to handwritten letters},
	booktitle = {IJCNN},
	author = {Cohen, Gregory and Afshar, Saeed and Tapson, Jonathan and van Schaik, André},
	year = {2017},
}

@inproceedings{bossard_food-101_2014,
	title = {Food-101 – {Mining} {Discriminative} {Components} with {Random} {Forests}},
	booktitle = {ECCV},
	author = {Bossard, Lukas and Guillaumin, Matthieu and Van Gool, Luc},
	year = {2014},
}

@article{clanuwat_deep_2018,
	title = {Deep {Learning} for {Classical} {Japanese} {Literature}},
    journal = {arXiv preprint arXiv: 1607.06450},
	author = {Clanuwat, Tarin and Bober-Irizar, Mikel and Kitamoto, Asanobu and Lamb, Alex and Yamamoto, Kazuaki and Ha, David},
	year = {2018},
}

@inproceedings{coates_analysis_2011,
	title = {An {Analysis} of {Single}-{Layer} {Networks} in {Unsupervised} {Feature} {Learning}},
	booktitle = {Proceedings of the {Fourteenth} {International} {Conference} on {Artificial} {Intelligence} and {Statistics}},
	publisher = {JMLR Workshop and Conference Proceedings},
	author = {Coates, Adam and Ng, Andrew and Lee, Honglak},
	year = {2011},
}

@inproceedings{nilsback_automated_2008,
	title = {Automated {Flower} {Classification} over a {Large} {Number} of {Classes}},
	booktitle = {2008 {Sixth} {Indian} {Conference} on {Computer} {Vision}, {Graphics} \& {Image} {Processing}},
	author = {Nilsback, Maria-Elena and Zisserman, Andrew},
	year = {2008},
}

@inproceedings{parkhi_cats_2012,
	title = {Cats and dogs},
	booktitle = CVPR,
	author = {Parkhi, Omkar M and Vedaldi, Andrea and Zisserman, Andrew and Jawahar, C. V.},
	year = {2012},
}

@inproceedings{veeling_rotation_2018,
	title = {Rotation {Equivariant} {CNNs} for {Digital} {Pathology}},
	booktitle = {MICCAI},
	author = {Veeling, Bastiaan S. and Linmans, Jasper and Winkens, Jim and Cohen, Taco and Welling, Max},
	year = {2018},
}

@article{goodfellow_challenges_2013,
	title = {Challenges in {Representation} {Learning}: {A} {Report} on {Three} {Machine} {Learning} {Contests}},
        journal   = {Neural Networks},
	author = {Goodfellow, Ian J. and Erhan, Dumitru and Carrier, Pierre Luc and Courville, Aaron and Mirza, Mehdi and Hamner, Ben and Cukierski, Will and Tang, Yichuan and Thaler, David and Lee, Dong-Hyun and Zhou, Yingbo and Ramaiah, Chetan and Feng, Fangxiang and Li, Ruifan and Wang, Xiaojie and Athanasakis, Dimitris and Shawe-Taylor, John and Milakov, Maxim and Park, John and Ionescu, Radu and Popescu, Marius and Grozea, Cristian and Bergstra, James and Xie, Jingjing and Romaszko, Lukasz and Xu, Bing and Chuang, Zhang and Bengio, Yoshua},
	year = {2013},
}

@inproceedings{socher_recursive_nodate,
    title = "Recursive Deep Models for Semantic Compositionality Over a Sentiment Treebank",
    author = "Socher, Richard  and
      Perelygin, Alex  and
      Wu, Jean  and
      Chuang, Jason  and
      Manning, Christopher D.  and
      Ng, Andrew  and
      Potts, Christopher",
    booktitle = {EMNLP},
    year = "2013",
}

@article{xiao_fashion-mnist_2017,
  title   = {Fashion-MNIST: a Novel Image Dataset for Benchmarking Machine Learning Algorithms},
  author  = {Han Xiao and Kashif Rasul and Roland Vollgraf},
  year    = {2017},
  journal = {arXiv preprint arXiv: 1708.07747}
}

@inproceedings{wang2024localizing,
  author    = {Ke Wang and Nikolaos Dimitriadis and Guillermo Ortiz{-}Jim{\'{e}}nez and Fran{\c{c}}ois Fleuret and Pascal Frossard},
  title     = {Localizing Task Information for Improved Model Merging and Compression},
  booktitle = {ICML},
  year      = {2024},
}

@article{schonemann1966,
    title={A Generalized Solution of the Orthogonal Procrustes Problem},
    journal={Psychometrika},
    author={Schönemann, Peter H.}, 
    year={1966},
}

@article{ba2016layer,
  title   = {Layer Normalization},
  author  = {Jimmy Lei Ba and Jamie Ryan Kiros and Geoffrey E. Hinton},
  year    = {2016},
  journal = {arXiv preprint arXiv: 1607.06450}
}

@inproceedings{jin2023dataless,
  author    = {Xisen Jin and Xiang Ren and Daniel Preotiuc{-}Pietro and Pengxiang Cheng},
  title     = {Dataless Knowledge Fusion by Merging Weights of Language Models},
  booktitle = ICLR,
  year      = {2023},
}

@article{li2023deepmodelfusionsurvey,
      title={Deep Model Fusion: A Survey}, 
      author={Weishi Li and Yong Peng and Miao Zhang and Liang Ding and Han Hu and Li Shen},
      year={2023},
      journal={arXiv preprint arXiv: 2309.15698},
}

@article{delétang2024languagemodelingcompression,
      title={Language Modeling Is Compression}, 
      author={Grégoire Delétang and Anian Ruoss and Paul-Ambroise Duquenne and Elliot Catt and Tim Genewein and Christopher Mattern and Jordi Grau-Moya and Li Kevin Wenliang and Matthew Aitchison and Laurent Orseau and Marcus Hutter and Joel Veness},
      year={2024},
      journal={arXiv preprint arXiv:2309.10668},
}

@article{wei2025modeling,
  title={Modeling multi-task model merging as adaptive projective gradient descent},
  author={Wei, Yongxian and Tang, Anke and Shen, Li and Yuan, Chun and Cao, Xiaochun},
  journal={arXiv preprint arXiv:2501.01230},
  year={2025}
}

@article{franceschelli2025trainingfoundationmodelsdata,
      title={Training Foundation Models as Data Compression: On Information, Model Weights and Copyright Law}, 
      author={Giorgio Franceschelli and Claudia Cevenini and Mirco Musolesi},
      year={2025},
      journal={arXiv preprint arXiv:2407.13493},
}

@inproceedings{zeng2025parameter,
  title={RobustMerge: Parameter-Efficient Model Merging for MLLMs with Direction Robustness},
  author={Zeng, Fanhu and Guo, Haiyang and Zhu, Fei and Shen, Li and Tang, Hao},
  booktitle={NeurIPS},
  year={2025}
}

@inproceedings{liu2023visual,
  title={Visual instruction tuning},
  author={Liu, Haotian and Li, Chunyuan and Wu, Qingyang and Lee, Yong Jae},
  booktitle={NeurIPS},
  year={2023}
}

@inproceedings{du2024parameter,
  title={Parameter competition balancing for model merging},
  author={Du, Guodong and Lee, Junlin and Li, Jing and Jiang, Runhua and Guo, Yifei and Yu, Shuyang and Liu, Hanting and Goh, Sim K and Tang, Ho-Kin and He, Daojing and others},
  booktitle={NeurIPS},
  year={2024}
}

@inproceedings{yu2024language,
  title={Language models are super mario: Absorbing abilities from homologous models as a free lunch},
  author={Yu, Le and Yu, Bowen and Yu, Haiyang and Huang, Fei and Li, Yongbin},
  booktitle={ICML},
  year={2024}
}

@inproceedings{marczak2025notaskleftbehind,
  title     = {{N}o {T}ask {L}eft {B}ehind: {I}sotropic {M}odel {M}erging with {C}ommon and {T}ask-{S}pecific {S}ubspaces},
  author    = {Daniel Marczak and Simone Magistri and Sebastian Cygert and Bartłomiej Twardowski and Andrew D. Bagdanov and Joost van de Weijer},
  year      = {2025},
  booktitle = {ICML},
}

@article{vasudevan2017hierarchical,
  title   = {A Hierarchical Singular Value Decomposition Algorithm for Low Rank Matrices},
  author  = {Vinita Vasudevan and M. Ramakrishna},
  year    = {2017},
  journal = {arXiv preprint arXiv: 1710.02812}
}

@inproceedings{lu2024twin0merging0,
  title     = {Twin-Merging: Dynamic Integration of Modular Expertise in Model Merging},
  author    = {Zhenyi Lu and Chenghao Fan and Wei Wei and Xiaoye Qu and Dangyang Chen and Yu Cheng},
  booktitle   = {NeurIPS},
  year      = {2024},
}

@inproceedings{huang2024emr0merging0,
  title     = {EMR-Merging: Tuning-Free High-Performance Model Merging},
  author    = {Chenyu Huang and Peng Ye and Tao Chen and Tong He and Xiangyu Yue and Wanli Ouyang},
  booktitle   = {NeurIPS},
  year      = {2024},
}

@inproceedings{fan2023federated,
  title={Federated learning with bilateral curation for partially class-disjoint data},
  author={Fan, Ziqing and Zhang, Ruipeng and Yao, Jiangchao and Han, Bo and Zhang, Ya and Wang, Yanfeng},
  booktitle={Proceedings of the 37th International Conference on Neural Information Processing Systems},
  pages={32006--32019},
  year={2023}
}

@inproceedings{fan2024locally,
  title={Locally estimated global perturbations are better than local perturbations for federated sharpness-aware minimization},
  author={Fan, Ziqing and Hu, Shengchao and Yao, Jiangchao and Niu, Gang and Zhang, Ya and Sugiyama, Masashi and Wang, Yanfeng},
  booktitle={Proceedings of the 41st International Conference on Machine Learning},
  pages={12858--12881},
  year={2024}
}

@article{
luo2026keeplora,
title={KeepLo{RA}: Continual Learning with Residual Gradient Adaptation},
author={Mao-Lin Luo and Zi-Hao Zhou and Yi-Lin Zhang and Yuanyu Wan and Min-Ling Zhang and Tong Wei},
journal={arXiv preprint arXiv:2601.19659},
year={2026}
}
}

\appendix
\newtheorem{proof}{Proof}
\clearpage
\maketitlesupplementary

\section{Notations}
We present the definitions used in our paper and their corresponding mathematical symbols in Table~\ref{tab:notations} for ease of reference.
\begin{table*}[!h]
\centering
\caption{List of definitions and their corresponding mathematical symbols used in this paper.}
\label{tab:notations}
\begin{tabular}{>{\centering\arraybackslash}p{0.40\textwidth} >{\raggedright\arraybackslash}p{0.55\textwidth}}
\toprule
\textbf{Symbol} & \textbf{Definition} \\
\midrule
$\boldsymbol{W}_0, \{\boldsymbol{W}_i\}_{i=1}^{T}$ & The weights of pretrained model  \\
$\{\boldsymbol{W}_i\}_{i=1}^{T}$ & The weights of $T$ tasks fine-tuned model \\
$\Delta \boldsymbol{W}_i$ & The task vector of task $i$ \\
$\Delta \boldsymbol{W}_i = \sum_{j=1}^{r} \sigma_i^j \boldsymbol{u}_i^j \boldsymbol{v}_i^{j \top}$ & Knowledge decomposition of $\Delta \boldsymbol{W}_i$ \\
$\sigma_i^j \boldsymbol{u}_i^j \boldsymbol{v}_i^{j \top}$ & $j$-th knowledge vector of $\Delta\boldsymbol{W}_i$ \\
$\boldsymbol{u}_i^j \boldsymbol{v}_i^{j \top}$ & $j$-th knowledge component of $\Delta\boldsymbol{W}_i$ \\
$\boldsymbol{\sigma}_i=(\sigma_i^1, \dots, \sigma_i^r)$ & Energy distribution of task $i$ \\
$\{\boldsymbol{u}_i^j \boldsymbol{v}_i^{j \top}\}_{j=1}^r$ & Directional geometry of task $i$ \\
$\Delta \overline{\boldsymbol{W}}_i = \left(\frac{\sum_{j=1}^{r} \sigma_i^{j}}{r}\right) \left(\sum_{j=1}^{r} \boldsymbol{u}_i^{j} \boldsymbol{v}_i^{j\top}\right)$ & Energy-balanced task vector of task $i$ \\
$\widetilde{\boldsymbol{W}}, \Delta \widetilde{\boldsymbol{W}}$ & The weights of merged model and merged task vector, respectively \\
$\boldsymbol{U}_i^{(r)}, \boldsymbol{V}_i^{(r)}$ & The top-$r$ left and right singular vectors of $\Delta \boldsymbol{W}_i$ \\
$\boldsymbol{U}, \boldsymbol{V}$ & The concatenated per-task knowledge basis \\
$(\widetilde{\boldsymbol{U}}, \widetilde{\boldsymbol{V}})$ & The shared cover basis across all tasks \\
$\boldsymbol{M}_i = \widetilde{\boldsymbol{U}}^\top \Delta\boldsymbol{W}_i\widetilde{\boldsymbol{V}}$ & Task vector $i$ under shared cover space induced by $(\widetilde{\boldsymbol{U}}, \widetilde{\boldsymbol{V}})$ \\
$\widetilde{\boldsymbol{M}} = \text{Element-wise Merging}(\{\boldsymbol{M}_i\}_{i=1}^T)$ & The merged vector under shared cover space \\
$\boldsymbol{\mathcal{M}}=\mathrm{block\mbox{-}diag}(\boldsymbol{1}_{r\times r}, \cdots, \boldsymbol{1}_{r\times r})$ & Structural mask for reconstructing parameters \\

\bottomrule
\end{tabular}
\end{table*}
\section{Proof of Proposition \ref{pro:cossim}}

\begin{proof}
Let the (reduced) SVDs (knowledge-vector decompositions) of the two task updates be
\begin{equation}
\begin{aligned}
& \Delta \boldsymbol{W}_{s} = \sum_{i=1}^{n}\sigma_{s}^i\,\boldsymbol{u}_{s}^i\left(\boldsymbol{v}_{s}^i\right)^{\top}
= \boldsymbol{U}_s\,\mathrm{diag}(\boldsymbol{\sigma}^{s})\,\boldsymbol{V}_s^{\top},
\\ &
\Delta \boldsymbol{W}_{t} \;=\; \sum_{j=1}^{m}\sigma_{t}^j\,\boldsymbol{u}_{t}^j\left(\boldsymbol{v}_{t}^j\right)^{\top}
= \boldsymbol{U}_t\,\mathrm{diag}(\boldsymbol{\sigma}^{t})\,\boldsymbol{V}_t^{\top},
\end{aligned}
\end{equation}
where $\{\boldsymbol{u}_{s}^i\}_{i=1}^{n}$ and $\{\boldsymbol{v}_{s}^i\}_{i=1}^{n}$ (resp. $\{\boldsymbol{u}_{t}^j\}_{j=1}^{m}$ and $\{\boldsymbol{v}_{t}^j\}_{j=1}^{m}$) are orthonormal left/right singular vectors of $\Delta\boldsymbol{W}_{s}$ (resp. $\Delta\boldsymbol{W}_{t}$), and $\boldsymbol{\sigma}_{s}\in\mathbb{R}^{n}_{\ge 0}$, $\boldsymbol{\sigma}_{t}\in\mathbb{R}^{m}_{\ge 0}$ collect the singular values in descending order. We calculate the numerator and  denominator of $\mathrm{CosSim}$ as follows:

\noindent
\textbf{Numerator.}
Using the Frobenius inner product $\langle \boldsymbol{A},\boldsymbol{B}\rangle=\mathrm{tr}(\boldsymbol{A}^{\top}\boldsymbol{B})$ and bilinearity, we have:
\begin{equation}
\begin{aligned}
\big\langle \Delta \boldsymbol{W}_{s}, \Delta \boldsymbol{W}_{t}\big\rangle
&= \left\langle \sum_{i=1}^{n}\sigma_{s}^i\,\boldsymbol{u}_{s}^i\left(\boldsymbol{v}_{s}^i\right)^{\top},
\sum_{j=1}^{m}\sigma_{t}^j\,\boldsymbol{u}_{t}^j\left(\boldsymbol{v}_{t}^j\right)^{\top} \right\rangle
\\& = \sum_{i=1}^{n}\sum_{j=1}^{m}\sigma_{s}^i \sigma_{t}^j
\left\langle \boldsymbol{u}_{s}^i\left(\boldsymbol{v}_{s}^i\right)^{\top},
\boldsymbol{u}_{t}^j\left(\boldsymbol{v}_{t}^j\right)^{\top} \right\rangle .
\end{aligned}
\end{equation}
For rank-one matrices $\boldsymbol{a}\boldsymbol{b}^{\top}$ and $\boldsymbol{c}\boldsymbol{d}^{\top}$ we have
$\langle \boldsymbol{a}\boldsymbol{b}^{\top}, \boldsymbol{c}\boldsymbol{d}^{\top}\rangle
= (\boldsymbol{a}^{\top}\boldsymbol{c})\,(\boldsymbol{b}^{\top}\boldsymbol{d})$.
Thus, we have:
\begin{equation}
\begin{aligned}
\left\langle \boldsymbol{u}_{s}^i\left(\boldsymbol{v}_{s}^i\right)^{\top},
\boldsymbol{u}_{t}^j\left(\boldsymbol{v}_{t}^j\right)^{\top} \right\rangle
& =
\left(\boldsymbol{u}_{s}^i\right)^{\top}\boldsymbol{u}_{t}^j \;\cdot\; \left(\boldsymbol{v}_{s}^i\right)^{\top}\boldsymbol{v}_{t}^j
\\ & =
\left(\boldsymbol{u}_{s}^i\right)^{\top}\boldsymbol{u}_{t}^j \;\cdot\; \left(\boldsymbol{v}_{t}^j\right)^{\top}\boldsymbol{v}_{s}^i .
\end{aligned}
\end{equation}
Therefore, we have:
\begin{equation}
\begin{aligned}
\big\langle \Delta \boldsymbol{W}_{s}, \Delta \boldsymbol{W}_{t}\big\rangle
&= \sum_{i=1}^{n}\sum_{j=1}^{m}\sigma_{s}^i \sigma_{t}^j\;
\underbrace{\left(\boldsymbol{u}_{s}^i\right)^{\top}\boldsymbol{u}_{t}^j\;\left(\boldsymbol{v}_{t}^j\right)^{\top}\boldsymbol{v}_{s}^i}_{\displaystyle \boldsymbol{R}_{i,j}(s,t)} \\
&= \boldsymbol{\sigma}_{t}\,\boldsymbol{R}(s,t)\,\boldsymbol{\sigma}_{s}^{\top} \;=\; \boldsymbol{\sigma}_{s}\,\boldsymbol{R}(s,t)\,\boldsymbol{\sigma}_{t}^{\top},
\end{aligned}
\end{equation}
where we have defined $\boldsymbol{R}(s,t)\in\mathbb{R}^{n\times m}$ entrywise by
$\boldsymbol{R}_{i,j}(s,t)= (\boldsymbol{u}_{s}^i)^{\top}\boldsymbol{u}_{t}^j \, (\boldsymbol{v}_{t}^j)^{\top}\boldsymbol{v}_{s}^i$.

\noindent
\textbf{Denominator.}
For any matrix $\boldsymbol{A}$ with SVD $\boldsymbol{A}=\sum_{\ell}\sigma^\ell \boldsymbol{u}^\ell \boldsymbol{v}^{\ell\top}$, the Frobenius norm satisfies
$\|\boldsymbol{A}\|_F^2=\sum_{\ell}\left(\sigma^\ell\right)^2$.
Hence
\begin{equation}
\begin{aligned}
& \big\|\Delta \boldsymbol{W}_{s}\big\|_F
= \left(\sum_{i=1}^{n}(\sigma_{s}^i)^2\right)^{1/2}
= \big\|\boldsymbol{\sigma}_{s}\big\|_2,
\\ &
\big\|\Delta \boldsymbol{W}_{t}\big\|_F
= \left(\sum_{j=1}^{m}(\sigma_{t}^j)^2\right)^{1/2}
= \big\|\boldsymbol{\sigma}_{t}\big\|_2.
\end{aligned}
\end{equation}

\noindent
\textbf{Conclusion.}
Combining the numerator and denominator gives
\begin{equation}
\begin{aligned}
\mathrm{CosSim}\left(\Delta \boldsymbol{W}_{s}, \Delta \boldsymbol{W}_{t}\right)
& = \frac{\langle \Delta \boldsymbol{W}_{s}, \Delta \boldsymbol{W}_{t}\rangle}
{\|\Delta \boldsymbol{W}^{s}\|_{F}\,\|\Delta \boldsymbol{W}^{t}\|_{F}}
\\ & = \frac{ \boldsymbol{\sigma}_{s}\, \boldsymbol{R}(s,t)\, \left(\boldsymbol{\sigma}_{t}\right)^{\top} }
{\|\boldsymbol{\sigma}_{s}\|_{2}\,\|\boldsymbol{\sigma}_{t}\|_{2}},
\end{aligned}
\end{equation}
which proves the proposition.
\end{proof}



\section{CosSim as Expressiveness via Coefficient Projections}
\label{sec:cos_projection_expressiveness}

\noindent\textbf{Setting and notation.}
Let $\boldsymbol{\Delta}\in\mathbb{R}^{m\times n}$, and let
$\boldsymbol{U}\in\mathbb{R}^{m\times k}$, $\boldsymbol{V}\in\mathbb{R}^{n\times k}$ have
orthonormal columns, i.e., $\boldsymbol{U}^\top\boldsymbol{U}=\boldsymbol{V}^\top\boldsymbol{V}=\boldsymbol{I}_k$.
Define the \emph{bi-directional coefficient vector}
\[
\boldsymbol{\sigma}(\boldsymbol{U},\boldsymbol{V},\boldsymbol{\Delta})
\;\triangleq\;
\mathrm{diag}\!\left(\boldsymbol{U}^{\top}\boldsymbol{\Delta}\,\boldsymbol{V}\right)\in\mathbb{R}^k .
\]

\vspace{0.2cm}
\begin{proposition}
\label{prop:equivalences}
With the notation above, the following hold:

\noindent\textbf{(A) Least-squares fitting $\leftrightarrow$ maximizing $\ell_2$-coefficients.}
\begin{equation}
\label{eq:A_equiv}
\min_{\boldsymbol{U}, \boldsymbol{V}} \min_{\boldsymbol{\sigma} \in \mathbb{R}^{k}}
\left\|
\boldsymbol{\Delta} - \boldsymbol{U} \, \mathrm{diag}(\boldsymbol{\sigma}) \, \boldsymbol{V}^{\top}
\right\|_F \leftrightarrow \max_{\boldsymbol{U}, \boldsymbol{V}}
\left\|\boldsymbol{\sigma}(\boldsymbol{U},\boldsymbol{V}, \boldsymbol{\Delta})
\right\|_2
\end{equation}

\noindent\textbf{(B) Cosine similarity with the unweighted dyadic sum $\leftrightarrow$ maximizing $\ell_1$-coefficients.}
\begin{equation}
\label{eq:B_equiv}
\max_{\boldsymbol{U}, \boldsymbol{V}} \ \mathrm{CosSim} \left(\boldsymbol{\Delta},\boldsymbol{U} \boldsymbol{V}^{\top}\right)
\leftrightarrow \max_{\boldsymbol{U}, \boldsymbol{V}} 
\left\|\boldsymbol{\sigma}(\boldsymbol{U},\boldsymbol{V}, \boldsymbol{\Delta})
\right\|_1
\end{equation}
The equivalences are up to the usual sign-alignment of column pairs
$\{(\boldsymbol{u}_j,\boldsymbol{v}_j)\}_{j=1}^k$ (flipping both signs leaves
$\boldsymbol{U}\boldsymbol{V}^\top$ unchanged), which can be chosen to make all
entries of $\boldsymbol{\sigma}(\boldsymbol{U},\boldsymbol{V},\boldsymbol{\Delta})$ nonnegative.
\end{proposition}

\vspace{0.1cm}
\begin{proof}
\textbf{(A).} Fix $\boldsymbol{U},\boldsymbol{V}$ orthonormal and consider
\[
\min_{\boldsymbol{\sigma}\in\mathbb{R}^k}
\left\|\,
\boldsymbol{\Delta}-\boldsymbol{U}\,\mathrm{diag}(\boldsymbol{\sigma})\,\boldsymbol{V}^\top
\,\right\|_F^2.
\]
Let $\boldsymbol{C}=\boldsymbol{U}^\top\boldsymbol{\Delta}\,\boldsymbol{V}$ and
$\boldsymbol{d}=\mathrm{diag}(\boldsymbol{C})=\boldsymbol{\sigma}(\boldsymbol{U},\boldsymbol{V},\boldsymbol{\Delta})$.
Using orthonormality and the bilinearity of the Frobenius inner product,
\[
\big\langle \boldsymbol{\Delta},\,\boldsymbol{U}\mathrm{diag}(\boldsymbol{\sigma})\boldsymbol{V}^\top\big\rangle
=\big\langle \boldsymbol{C},\,\mathrm{diag}(\boldsymbol{\sigma})\big\rangle
=\sum_{j=1}^k d_j\,\sigma_j .
\]
Therefore
\begin{equation}
\label{eq:proj_2}
\begin{aligned}
& \left\|
\boldsymbol{\Delta}-\boldsymbol{U}\mathrm{diag}(\boldsymbol{\sigma})\boldsymbol{V}^\top
\right\|_F^2 
\\ & =\|\boldsymbol{\Delta}\|_F^2+\|\mathrm{diag}(\boldsymbol{\sigma})\|_F^2-2\sum_{j=1}^k d_j\,\sigma_j
\\ & =\|\boldsymbol{\Delta}\|_F^2+\sum_{j=1}^k(\sigma_j^2-2d_j\sigma_j).
\end{aligned}
\end{equation}
The unique minimizer is $\sigma_j^\star=d_j$ for all $j$, i.e.,
$\boldsymbol{\sigma}^\star=\boldsymbol{d}=\boldsymbol{\sigma}(\boldsymbol{U},\boldsymbol{V},\boldsymbol{\Delta})$,
and the minimal value equals
\[
\|\boldsymbol{\Delta}\|_F^2-\sum_{j=1}^k d_j^2
=\|\boldsymbol{\Delta}\|_F^2-\big\|\boldsymbol{\sigma}(\boldsymbol{U},\boldsymbol{V},\boldsymbol{\Delta})\big\|_2^2.
\]
Hence, after inner minimization over $\boldsymbol{\sigma}$, the outer optimization over
$\boldsymbol{U},\boldsymbol{V}$ minimizes the residual iff it \emph{maximizes}
$\|\boldsymbol{\sigma}(\boldsymbol{U},\boldsymbol{V},\boldsymbol{\Delta})\|_2$, proving
Eq.~\eqref{eq:A_equiv}.

\noindent\textbf{(B).} Using $\langle \boldsymbol{A},\boldsymbol{B}\rangle=\mathrm{tr}(\boldsymbol{A}^\top\boldsymbol{B})$, we have:
\begin{equation}
\begin{aligned}
\big\langle \boldsymbol{\Delta},\,\boldsymbol{U}\boldsymbol{V}^\top\big\rangle & 
=\mathrm{tr}(\boldsymbol{U}^\top\boldsymbol{\Delta}\,\boldsymbol{V})
=\sum_{j=1}^k (\boldsymbol{U}^\top\boldsymbol{\Delta}\,\boldsymbol{V})_{jj}
\\&
=\sum_{j=1}^k \sigma_j(\boldsymbol{U},\boldsymbol{V},\boldsymbol{\Delta})
=\big\|\boldsymbol{\sigma}(\boldsymbol{U},\boldsymbol{V},\boldsymbol{\Delta})\big\|_1
\end{aligned}
\end{equation}
after sign alignment (flip $(\boldsymbol{u}_j,\boldsymbol{v}_j)$ together if needed).
Moreover, $\|\boldsymbol{U}\boldsymbol{V}^\top\|_F^2
=\mathrm{tr}(\boldsymbol{V}\boldsymbol{U}^\top\boldsymbol{U}\boldsymbol{V}^\top)
=\mathrm{tr}(\boldsymbol{V}\boldsymbol{V}^\top)=k$.
Therefore
\begin{equation}
\begin{aligned}
\mathrm{CosSim}\!\left(\boldsymbol{\Delta},\,\boldsymbol{U}\boldsymbol{V}^\top\right)
& =\frac{\langle \boldsymbol{\Delta},\,\boldsymbol{U}\boldsymbol{V}^\top\rangle}{\|\boldsymbol{\Delta}\|_F\,\|\boldsymbol{U}\boldsymbol{V}^\top\|_F}
\\ & =\frac{\big\|\boldsymbol{\sigma}(\boldsymbol{U},\boldsymbol{V},\boldsymbol{\Delta})\big\|_1}{\|\boldsymbol{\Delta}\|_F\,\sqrt{k}} .
\end{aligned}
\end{equation}
Since the denominator is independent of $(\boldsymbol{U},\boldsymbol{V})$, maximizing the cosine
is equivalent to maximizing $\|\boldsymbol{\sigma}(\boldsymbol{U},\boldsymbol{V},\boldsymbol{\Delta})\|_1$,
establishing Eq.~\eqref{eq:B_equiv}.
\end{proof}

\noindent
\textbf{Interpretation.}
Each entry of $\boldsymbol{\sigma}(\boldsymbol{U},\boldsymbol{V},\boldsymbol{\Delta})
=\mathrm{diag}(\boldsymbol{U}^\top\boldsymbol{\Delta}\boldsymbol{V})$ is a \emph{bi-directional projection}
of $\boldsymbol{\Delta}$ onto the dyad $\boldsymbol{u}_j\boldsymbol{v}_j^\top$ (input- and output-side alignments jointly).
Part (A) shows that the best reconstruction by dyads in the basis $(\boldsymbol{U},\boldsymbol{V})$ is achieved when the
coefficients equal these projections, and the residual decreases as the $\ell_2$ energy of the projections increases.
Part (B) shows that the cosine with the \emph{unweighted} dyadic sum $\boldsymbol{U}\boldsymbol{V}^\top$ is (up to a constant)
the $\ell_1$ sum of these projections. Consequently, a higher cosine indicates stronger ability of the basis $(\boldsymbol{U},\boldsymbol{V})$
to \emph{express} (encode) the knowledge contained in $\boldsymbol{\Delta}$—exactly mirroring the main-text view where
$\boldsymbol{R}(s,t)$ aggregates bi-directional projections to quantify how effectively task $s$ represents task $t$.

\section{Theoretical Analysis of Merging in the Cover Space}
\label{app:D}

In this section, we first provide a theoretical justification for whitening-based construction of cover space. The empirical results provide an intuitive and solid validation of the proposed objective in Eq.~\eqref{eq:sur-obj}. We also present visualization of task representations in cover space and investigate the effect of structural mask. Finally, we provide a toy example to intuitively demonstrate the crucial role that shared cover space plays in directional preservation.

\vspace{0.5em}
\noindent
\textbf{Theoretical Analysis of Whitening-based Cover Space Construction.}
First, by expanding and applying the Hadamard norm inequality, we obtain:
\begin{equation} 
\label{eq:upperbound_obj} 
\begin{aligned} 
\sum_{i=1}^{T}\sum_{j=1}^{r} & \left\|\boldsymbol{\sigma}(\widetilde{\boldsymbol{U}},\widetilde{\boldsymbol{V}},\boldsymbol{u}_{i}^{j} \boldsymbol{v}_{i}^{j \top} )\right\|_{2}^{2} \\& = \sum_{i=1}^{T} \sum_{j=1}^{r} \left\|\mathrm{diag} \left( (\widetilde{\boldsymbol{U}}^{\top} \boldsymbol{u}_{i}^{j} ) (\boldsymbol{v}_{i}^{j \top} \widetilde{\boldsymbol{V}})\right)\right\|_{2}^{2} \\& = \sum_{i=1}^{T}\sum_{j=1}^{r} \left\| \left(\widetilde{\boldsymbol{U}}^{\top} \boldsymbol{u}_{i}^{j}\right) \odot \left(\widetilde{\boldsymbol{V}}^{\top} \boldsymbol{v}_{i}^{j}\right) \right\|_{2}^{2} \\& = \left\|\left(\widetilde{\boldsymbol{U}}^{\top} \boldsymbol{U} \right) \odot \left(\widetilde{\boldsymbol{V}}^{\top} \boldsymbol{V}\right)\right\|_{F}^{2} \\& \le \left\|\widetilde{\boldsymbol{U}}^{\top} \boldsymbol{U} \right\|_{F}^{2}\left\|\widetilde{\boldsymbol{V}}^{\top} \boldsymbol{V} \right\|_{F}^{2}, \end{aligned} 
\end{equation}
where $\boldsymbol{U} = ([\boldsymbol{U}_1^{(r)},\ldots,\boldsymbol{U}_T^{(r)}])$ and $\boldsymbol{V} = ([\boldsymbol{V}_1^{(r)},\ldots,\boldsymbol{V}_T^{(r)}])$.  
This shows that the surrogate objective is upper-bounded by two alignment terms $\left\|\widetilde{\boldsymbol{U}}^{\top} \boldsymbol{U} \right\|_{F}^{2}$ and $\left\|\widetilde{\boldsymbol{V}}^{\top} \boldsymbol{V} \right\|_{F}^{2}$, which characterize how well $\widetilde{\boldsymbol{U}}$ aligns with all $\boldsymbol{U}_i^{(r)}$ and how well $\widetilde{\boldsymbol{V}}$ aligns with all $\boldsymbol{V}_i^{(r)}$.

Next, using the diagonal-trace inequality and arithmetic mean inequality, we have:
\begin{equation}
\label{eq:lowerbound_obj}
\begin{aligned}
\left\|\widetilde{\boldsymbol{U}}^{\top} \boldsymbol{U}  \right\|_{F}^{2}& \ge \left\|\mathrm{diag} \left(\widetilde{\boldsymbol{U}}^{\top} \boldsymbol{U}  \right)\right\|_{2}^{2} \ge \frac{1}{k} \mathrm{tr} \left(\widetilde{\boldsymbol{U}}^{\top} \boldsymbol{U}\right)^{2}, \\
\left\|\widetilde{\boldsymbol{V}}^{\top} \boldsymbol{V}  \right\|_{F}^{2}& \ge \left\|\mathrm{diag} \left(\widetilde{\boldsymbol{V}}^{\top} \boldsymbol{V}  \right)\right\|_{2}^{2} \ge \frac{1}{k} \mathrm{tr} \left(\widetilde{\boldsymbol{V}}^{\top} \boldsymbol{V}\right)^{2},
\end{aligned}
\end{equation}
we see that maximizing the surrogate in Eq.~\eqref{eq:sur-obj} effectively corresponds to maximizing these trace terms.  
Hence, the optimal $(\widetilde{\boldsymbol{U}},\widetilde{\boldsymbol{V}})$ should maximize the average alignment with each task knowledge directions.

Finally, whitening~\cite{schonemann1966}  naturally achieves this objective:
\begin{equation}
\begin{aligned}
& \widetilde{\boldsymbol{U}} = \arg\max_{\widetilde{\boldsymbol{U}}^{\top} \widetilde{\boldsymbol{U}}=\boldsymbol{I}} \Big\langle \boldsymbol{U},\;\widetilde{\boldsymbol{U}}\Big\rangle, 
\\ &
\widetilde{\boldsymbol{V}} = \arg\max_{\widetilde{\boldsymbol{V}}^{\top} \widetilde{\boldsymbol{V}}=\boldsymbol{I}} \Big\langle \boldsymbol{V},\;\widetilde{\boldsymbol{V}}\Big\rangle.
\end{aligned}
\end{equation}
This whitening-based alignment maximizes $\mathrm{tr}\!\left(\widetilde{\boldsymbol{U}}^{\top} \boldsymbol{U}\right)$ and $\mathrm{tr}\!\left(\widetilde{\boldsymbol{V}}^{\top} \boldsymbol{V}\right)$,  
which correspond to the lower-bound traces in Eq.~\eqref{eq:lowerbound_obj}.  
Therefore, whitening can be interpreted as an efficient approximate solution to the subspace alignment problem in Eq.~\eqref{eq:dir_objective_final}, providing a computationally simple yet theoretically justified way to construct shared cover basis.

\noindent
\textbf{Validation of Objective for Cover Basis.}
We further verify the rationality of the proposed trace-based objective by directly optimizing it to obtain the cover basis that minimizes reconstruction error, and compare the result with the whitening-based approximation.  
Our goal is to confirm that whitening indeed provides a close approximation to the optimal cover subspace defined by Eq.~\eqref{eq:sur-obj}.  

To optimize a set of rank-$k$ orthonormal bases in a $d$-dimensional space, we employ a differentiable parameterization using skew-symmetric matrices~\cite{pmlr-v97-lezcano-casado19a}.  
Specifically, any orthonormal matrix $\widetilde{\boldsymbol{U}}, \widetilde{\boldsymbol{V}}\in\mathbb{R}^{d\times k}$ can be expressed as
\begin{equation}
\widetilde{\boldsymbol{U}} = \exp(\boldsymbol{A})\,\widetilde{\boldsymbol{U}}_{0}, \quad \widetilde{\boldsymbol{V}} = \exp(\boldsymbol{B})\,\widetilde{\boldsymbol{V}}_{0}
\end{equation}
where $\boldsymbol{A}, \boldsymbol{B}\in\mathbb{R}^{d\times d}$ is a skew-symmetric matrix ($\boldsymbol{A}^\top=-\boldsymbol{A}, \boldsymbol{B}^\top=-\boldsymbol{B}$) and $\widetilde{\boldsymbol{U}}_0, \widetilde{\boldsymbol{V}}_0$ is an initial orthonormal basis.  
This parameterization utilizes the isomorphic properties of Lie groups and rotation operations~\cite{pmlr-v97-lezcano-casado19a} and ensures that $\widetilde{\boldsymbol{U}}, \widetilde{\boldsymbol{V}}$ always remains on the \emph{Stiefel manifold} $\mathcal{S}(d, k)$, thus preserving orthogonality during gradient-based optimization.  

To further validate the rationality of our optimization target, we optimize the cover basis following Algorithm~\ref{alg:opt_cover_basis}  
Starting from a poor initialization $\left(\widetilde{\boldsymbol{U}}_0, \widetilde{\boldsymbol{V}}_0\right)$ obtained by decomposing the merged task vector from TA~\cite{ilharco2023task} through SVD, we perform iterative optimization toward the objective in Eq.~\eqref{eq:sur-obj}.  Figure~\ref{fig_optimize} illustrates the optimization trajectory of the alignment score (left) and its logarithmic scale counterpart (right), along with the corresponding accuracy evolution by using corresponding cover basis.  
Initially, the alignment score is low, indicating poor consistency between the cover basis and the directions of each task.  
As optimization proceeds, both the alignment score and downstream task accuracy increase steadily and exhibit similar trends, confirming that higher alignment directly translates to improved task retention and generalization.  
In particular, the convergence of the alignment score to a stable maximum implies that the optimized basis successfully captures the shared directional geometry among all tasks.  

Notably, the final alignment and accuracy achieved by the optimized cover basis closely match those obtained via our whitening-based construction, validating that whitening provides a strong approximation to the optimal solution of the cover subspace objective.  Therefore, this empirical results confirm the theoretical justification that whitening acts as an efficient surrogate optimizer for constructing a shared cover basis.

\begin{algorithm}[t]
\caption{Optimization of Cover Basis via Skew-Symmetric Parameterization}
\label{alg:cover_opt}
\begin{algorithmic}[1]
\State \textbf{Input:} Task directional geometry $\{\boldsymbol{u}_i^j\boldsymbol{v}_i^{j\top}\}_{j=1}^{r}$, Task vectors$\{\boldsymbol{\Delta_{i}}\}_{i=1}^{T}$, learning rate $\eta$, maximum iteration steps $L$
\State \textbf{Output:} Optimized cover basis $\widetilde{\boldsymbol{U}}, \widetilde{\boldsymbol{V}}$
\vspace{0.5em}
\State {\color{cvprblue}{$\triangleright$ \textbf{Step 1:} Initialize cover basis }}
\State Initialize $\boldsymbol{A}, \boldsymbol{B} \in \mathbb{R}^{d\times d}$ as zero matrices
\State $\widetilde{\boldsymbol{U}}_{0},  \widetilde{\boldsymbol{V}}_{0} = \mathrm{SVD}(\sum_{i=1}^{T} \boldsymbol{\Delta}_{i})$
\vspace{0.5em}
\State {\color{cvprblue}{$\triangleright$ \textbf{Step 2:} Iterative optimization of cover basis}}
\For{$\ell = 1 \to L$}
    \State $\widetilde{\boldsymbol{U}} \gets \exp(\boldsymbol{A}-\boldsymbol{A}^{\top})\,\widetilde{\boldsymbol{U}}_{0}$,\quad
    $\widetilde{\boldsymbol{V}} \gets \exp(\boldsymbol{B} - \boldsymbol{B}^{\top})\,\widetilde{\boldsymbol{V}}_{0}$
    \State Compute alignment score $\mathcal{L}$ from Eq.~\eqref{eq:sur-obj}
    \State Compute gradients $\nabla_{\boldsymbol{A}}\mathcal{L}$, $\nabla_{\boldsymbol{B}}\mathcal{L}$
    \State $\boldsymbol{A} \leftarrow \boldsymbol{A} + \eta \nabla_{\boldsymbol{A}}\mathcal{L}$,\quad
        $\boldsymbol{B} \leftarrow \boldsymbol{B} + \eta \nabla_{\boldsymbol{B}}\mathcal{L}$
\EndFor
\State \textbf{return} $\widetilde{\boldsymbol{U}}, \widetilde{\boldsymbol{V}}$
\end{algorithmic}
\label{alg:opt_cover_basis}
\end{algorithm}

\begin{figure*}[!t]
    \centering
    \begin{subfigure}[!t]{0.4\textwidth}
        \centering
        \includegraphics[width=\textwidth]{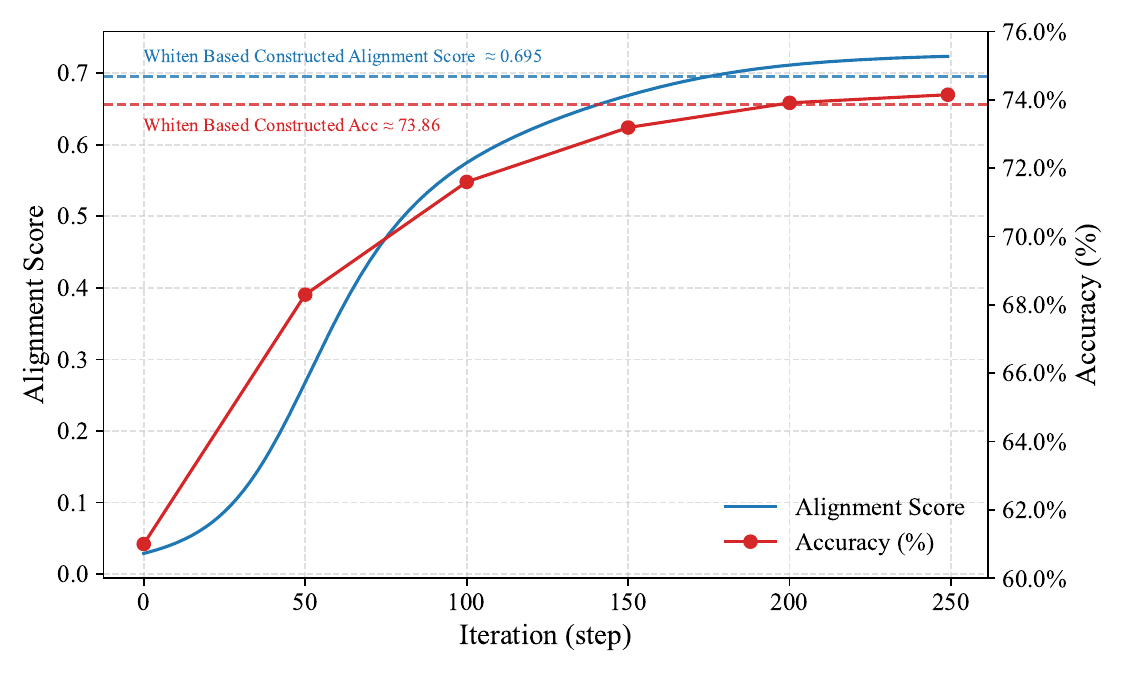}
        \caption{ Linear Scale
        }
    \end{subfigure}
    \hspace{0.03\textwidth}
    \begin{subfigure}[!t]{0.4\textwidth}
        \centering
        \includegraphics[width=\textwidth]{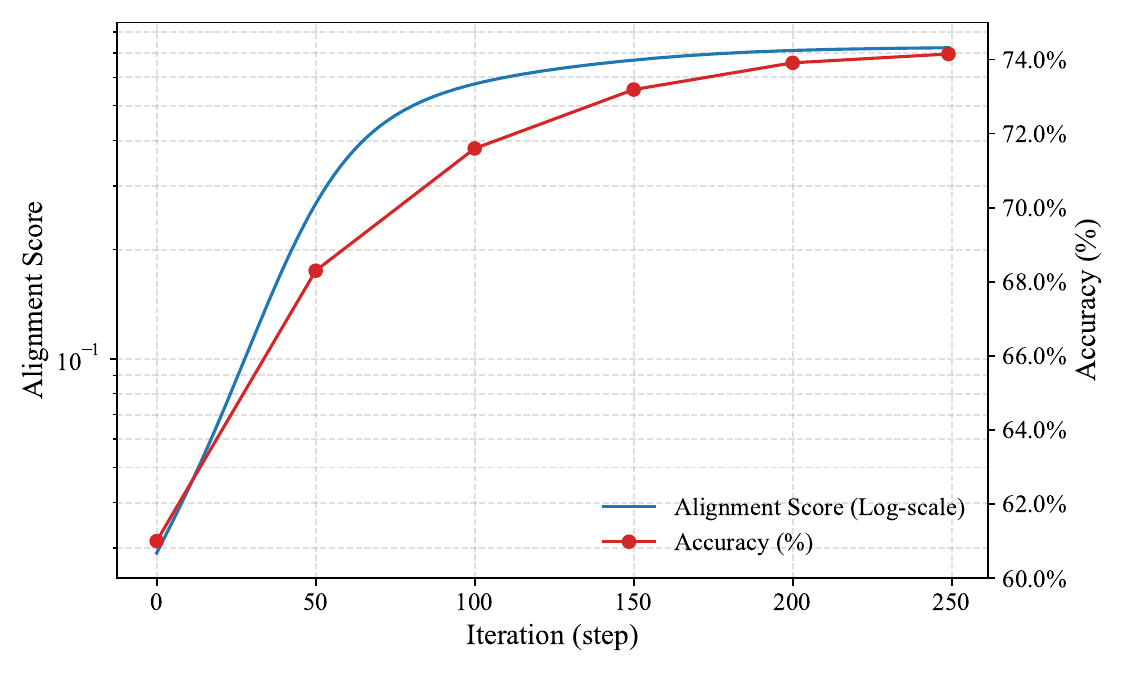}
        \caption{ Semi-log scale
        }
    \end{subfigure}
    \caption{
    \textbf{Optimization trajectory of the cover basis.}
    We start from a poor initialization $\left(\widetilde{\boldsymbol{U}}_0, \widetilde{\boldsymbol{V}}_0\right)$ obtained by SVD of task vectors from TA and iteratively optimize the surrogate objective in Eq.~\eqref{eq:sur-obj}.  
    (a) The optimization trace of alignment score and average normalized accuracy of DC-Merge under linear scale.  
    (b) The same process in semi-log scale for better visibility of early-stage dynamics.  Both alignment score and accuracy increase monotonically during training and converge to a stable optimum, demonstrating that higher alignment directly corresponds to stronger task retention.  
    The final accuracy and alignment closely match those of the whitening-based approximation, confirming that whitening yields a near-optimal cover basis.
    }
    \label{fig_optimize}
    \vspace{-0.3cm}
\end{figure*}

\noindent
\textbf{Visualization of Task Representations in Cover Space.}
To provide an intuitive illustration of how each task vector is represented in the cover space constructed by Algorithm~\ref{alg:cover_opt}, we visualize the representation in Figure~\ref{fig:mi_cover_space_part1} using a block-wise averaged aggregation. We perform the optimization for 250 steps on \model{ViT-B-32} 8-task benchmark in LoRA setting, and the size of each block is $r\times r$.

\noindent
\textbf{Effect of Structural Masking.}
After projection and merging within the cover space, a block-diagonal mask $\boldsymbol{\mathcal{M}}$ is applied:
\[
\boldsymbol{\Delta}_M = \widetilde{\boldsymbol{U}} \left(\hat{\boldsymbol{M}} \odot \boldsymbol{\mathcal{M}}\right) \widetilde{\boldsymbol{V}}^{\top}.
\]
We now investigate the underlying mechanism of the block-diagonal mask. Consider $T$ task vectors $\{\boldsymbol{\Delta}_i\}_{i=1}^T$ that share approximately directional geometry. Suppose that there exist shared orthonormal basis $\widetilde{\boldsymbol{U}}, \widetilde{\boldsymbol{V}}$ such that each task vector can be represented as:
\begin{equation}
\boldsymbol{\Delta}_i \approx \widetilde{\boldsymbol{U}} \, \mathrm{diag}(\boldsymbol{\sigma}_i) \, \widetilde{\boldsymbol{V}}^{\top}, 
\end{equation}
with $ \widetilde{\boldsymbol{U}}^{\top}\widetilde{\boldsymbol{U}} = \widetilde{\boldsymbol{V}}^{\top}\widetilde{\boldsymbol{V}} = \boldsymbol{I}$.
In this case, the merged model vector obtained by summation or averaging becomes:
\begin{equation}
\sum_{i=1}^{T} \boldsymbol{\Delta}_i 
\approx \widetilde{\boldsymbol{U}} \left(\sum_{i=1}^{T} \mathrm{diag}(\boldsymbol{\sigma}_i)\right) \widetilde{\boldsymbol{V}}^{\top},
\end{equation}
which clearly preserves the original knowledge directions $\widetilde{\boldsymbol{U}}$ and $\widetilde{\boldsymbol{V}}$ and $\boldsymbol{\sigma}_i$ is the corresponding optimal coefficients under $\widetilde{\boldsymbol{U}}$ and $\widetilde{\boldsymbol{V}}$ for $\boldsymbol{\Delta}_i$.  
That is, merging under a fixed cover basis ensures that the directional geometry of all task vectors remains unchanged while only the coefficients are combined. We construct a pair of cover basis under which task representations are largely distributed along diagonal blocks. When restricted to these diagonal components, merging does not change task directions. The mask explicitly suppresses off-diagonal elements to remove cross-task directional inconsistency that induces task interference.

The size of mask controls the trade-off between \textit{directional preservation} and \textit{individual task reconstruction fidelity}.  
In the extreme case where the mask size is $1$ (i.e., only diagonal elements retained), merging reduces to purely aggregating the aligned components without mixing any directions, which perfectly preserves subspace directions but sacrifices per-task reconstruction fidelity.  
As the mask size increases, more cross-task interactions are included, improving reconstruction of each individual task at the cost of slight directional deviation.  
Therefore, $\boldsymbol{\mathcal{M}}$ implicitly balances between directional consistency and task-specific fidelity.

\begin{figure}[H]
    \centering
    \includegraphics[width=\linewidth]{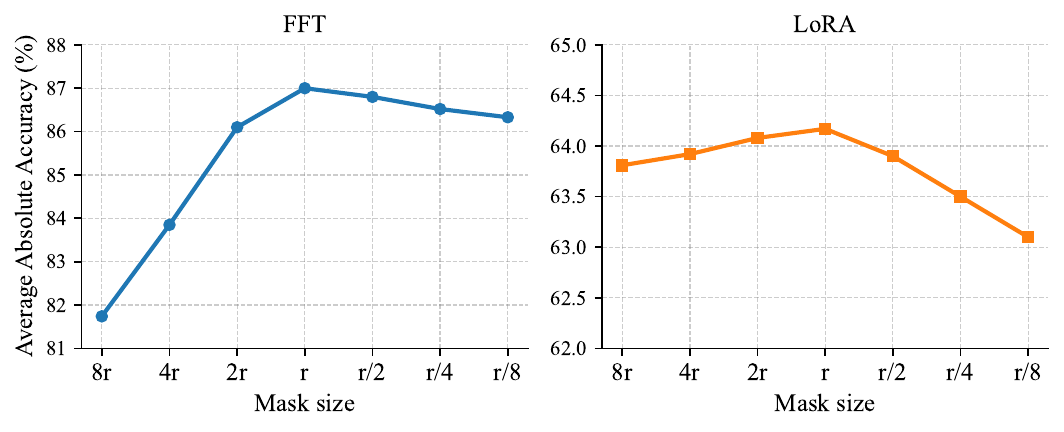}
    \caption{
    The average absolute accuracy of DC-Merge with mask size $\{8r,4r,2r,r,r/2,r/4,r/8\}$ in FFT (left) and LoRA (right) settings. Performance peaks near $r$ for FFT and shows a milder variation for LoRA (slightly improves till $r$ then decreases), indicating that moderate masks retain performance while overly small/large masks degrade it. The results are based on \model{ViT-B-32} 8-task benchmark.
    }
    \label{fig:mask_size_ablation}
    \vspace{-0.2cm}
\end{figure}


\noindent
\textbf{Illustrative Example.}
We provide a simple two-task example to intuitively demonstrate the advantage of merging within the shared cover space.  
Consider two tasks, each represented by a single knowledge component:
\begin{equation*}
\begin{aligned}
&u_1 = v_1 = [1,\,0]^{\top} \\&
u_2 = v_2 = [0.1104,\,0.9939]^{\top}.
\end{aligned}
\end{equation*}
When directly merging in the parameter space,
\begin{equation*}
\boldsymbol{\Delta}_M = u_1 v_1^{\top} + u_2 v_2^{\top} = \begin{bmatrix}
  1.0121& 0.1098\\
  0.1098 & 0.9878
\end{bmatrix}
\end{equation*}
the resulting SVD gives principal directions:
\begin{equation*}
\begin{aligned}
&u_{1}' = v'_1 = [0.7451,\,0.6669]^{\top}, 
\\& u'_2 = v'_2 = [0.6669,\,-0.7451]^{\top},
\end{aligned}
\end{equation*}
which significantly deviate from the original task directions, implying that direct merging distorts the directional geometric of task knowledge.

In contrast, when get shared cover basis by whitening the concatenation per-task knowledge basis  before merging, we have:
\begin{equation*}
\widetilde{\boldsymbol{U}} = \widetilde{\boldsymbol{V}} = \left[u_{1}', u_{2}'\right],
\end{equation*}
where $u_{1}' = \left[0.9985, -0.0533\right]^{\top}, u_{2}' = [0.0533, 0.9985]$. When only using a diagonal mask $\mathcal{M}$, the resulting directional equal to $\widetilde{\boldsymbol{U}},\widetilde{\boldsymbol{V}}$ remains nearly aligned with the original task directions.  This demonstrates that whitening preserves the cover geometry of knowledge directions and prevents directional shift during model merging. The visualization of this result can refer to Figure~\ref{fig:illustrative_example}.

\begin{figure}[ht]
    \centering
    \includegraphics[width=\linewidth, trim={100pt 15pt 150pt 15pt}, clip]{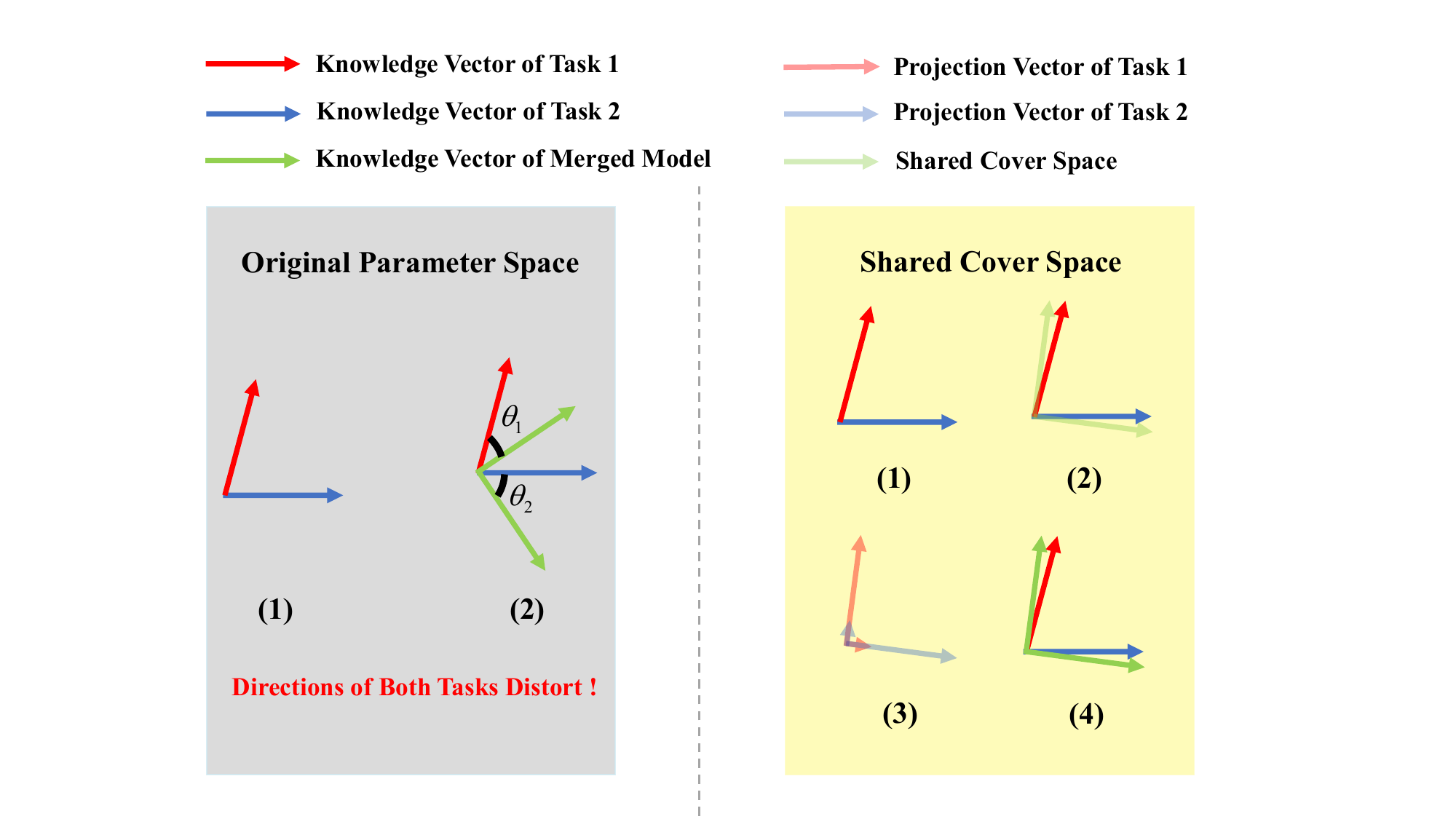}
    \caption{Illustration of mitigating directional shift when merging in the shared cover space. We take $u_1$ and $u_2$ as example. Left: (1) The direction of $u_1$ and $u_2$. (2) Directly merging in the original parameter space yields $u_1'$ and $u_2'$, causing severe directional deviation. Right: (1) The direction of $u_1$ and $u_2$. (2) Obtain $u_1'$ and $u_2'$ by whitening. (3) Project $u_1$ and $u_2$ onto the space spanned by $u_1'$ and $u_2'$. (4) Aggregate the projections along the directions of $u_1'$ and $u_2'$ respectively as merged knowledge vectors, which significantly alleviates directional shift.
    }
    \label{fig:illustrative_example}
    \vspace{-0.3cm}
\end{figure}

\begin{figure*}[!t]
    \centering
    \begin{subfigure}[!t]{0.32\textwidth}
        \centering
        \includegraphics[width=\textwidth, height=4.1cm]{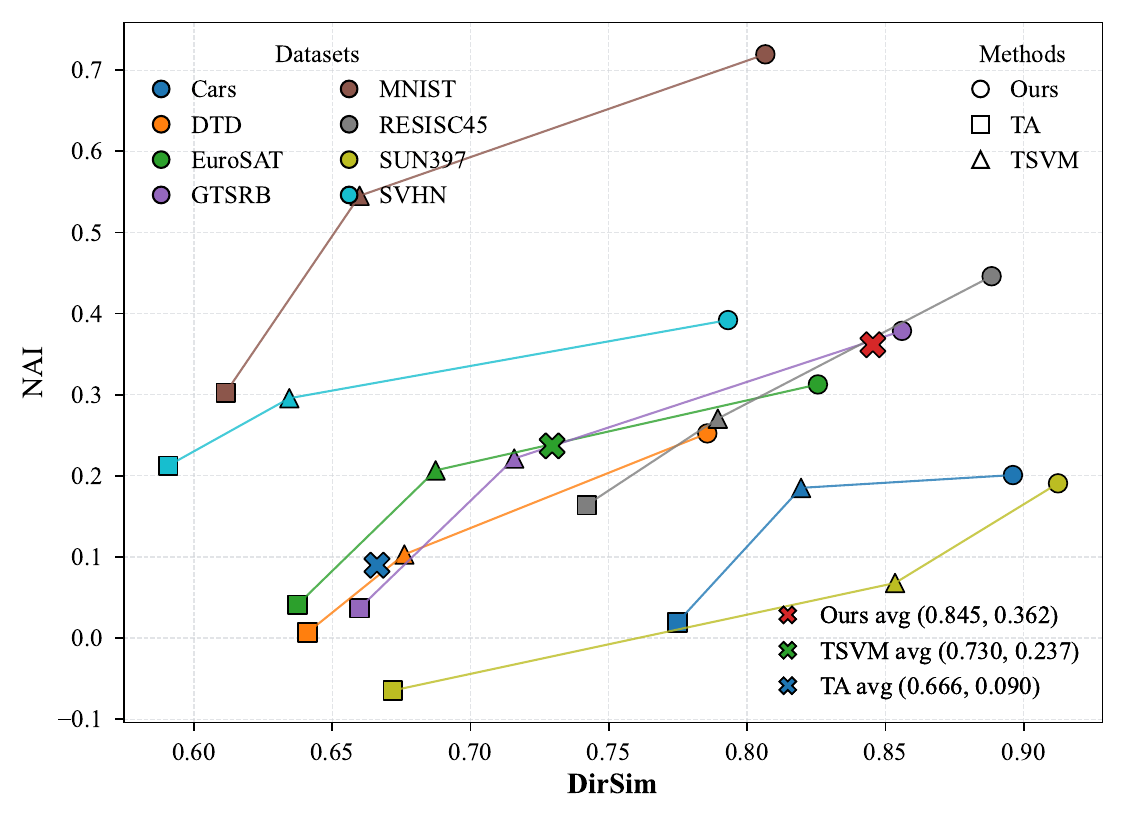}
        \caption{}
        \label{fig_mao}
    \end{subfigure}
    \hfill
    \begin{subfigure}[!t]{0.30\textwidth}
        \centering
        \includegraphics[width=\textwidth, height=4.3cm]{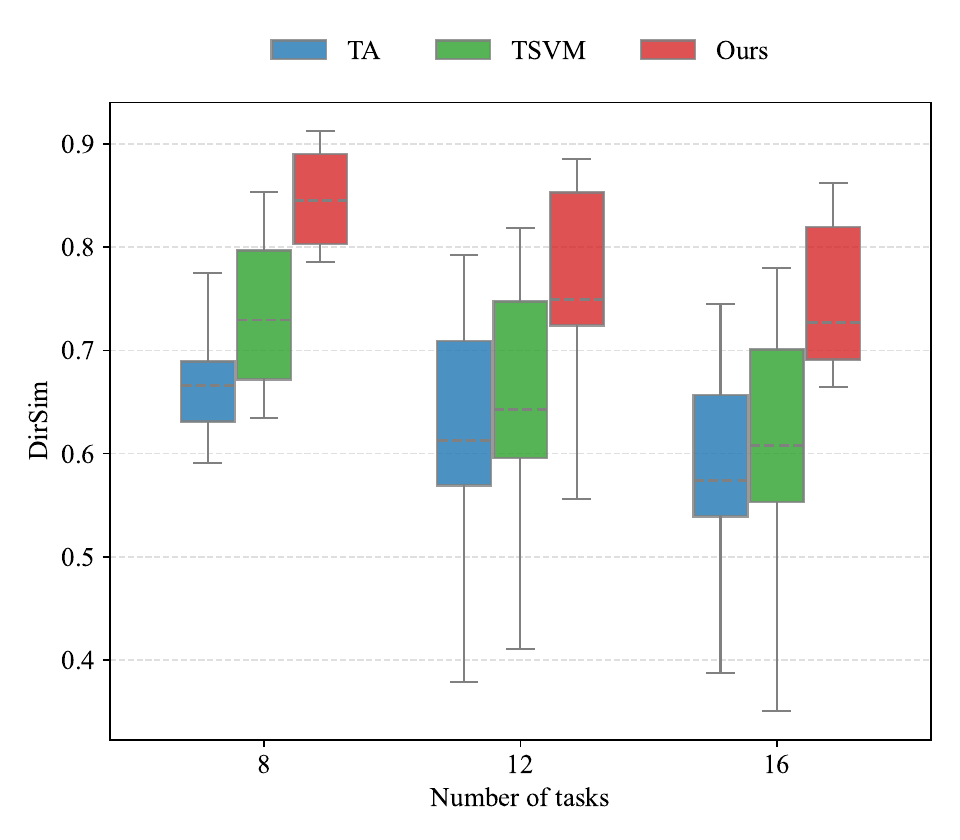}
        \caption{}
        \label{fig_box_dir}
    \end{subfigure}
    \hfill
    \begin{subfigure}[!t]{0.33\textwidth}
        \centering
        \includegraphics[width=\textwidth, height=4.3cm]{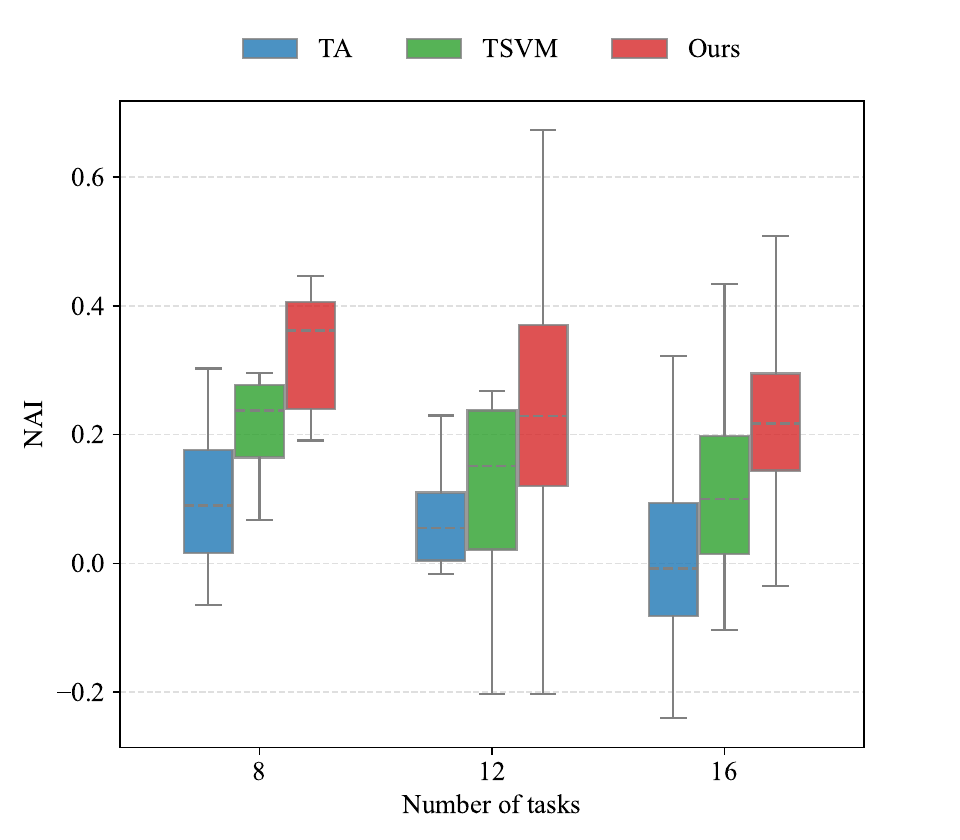}
        \caption{}
        \label{fig_box_nai}
    \end{subfigure}
    \vspace{-0.2cm}
    \caption{
    (a) Correlation of task-wise performance with \(\mathrm{DirSim}\). The result is based on \model{ViT-B-32} 8-task benchmark with our checkpoints, while we utilize checkpoints provided by KnOTS~\cite{stoica2024knots} for Figure~\ref{fig_metric_acc}.
    (b) Layer-averaged projected $\mathrm{DirSim}$ between merged multi-task vector and original task vectors for varying number of tasks.  
    (c) Task-wise NAI of the merged model for varying number of tasks. (b) and (c) are based on merging LoRA fine-tuned \model{ViT-B-32} models.
    }
    \vspace{-0.3cm}
    \label{fig_metric_boxplot}
\end{figure*}

\section{Experimental Details}
\label{app:E}
In this section, we first provide an overview of datasets and evaluation metrics employed in our experiments. We then present the implementation details of our experiments in depth and conclude with a discussion of different smoothing strategies. 
\subsection{Datasets and Metrics}
We present the datasets in the vision model benchmarks as follows. For full fine-tuning setting, the 8-task benchmark comprises the following datasets: Cars~\cite{cars}, DTD~\cite{dtd}, EuroSAT~\cite{eurosat}, GTSRB~\cite{gtsrb}, MNIST~\cite{MNIST}, RESISC45~\cite{cheng2017remote}, SUN397~\cite{sun397} and SVHN~\cite{svhn}.
The 14-task benchmark extends the previous one by incorporating six additional datasets: CIFAR100~\cite{krizhevsky2009learning}, STL10~\cite{coates_analysis_2011}, Flowers102~\cite{nilsback_automated_2008}, OxfordIIITPet~\cite{parkhi_cats_2012}, PCAM~\cite{veeling_rotation_2018} and FER2013~\cite{goodfellow_challenges_2013}.
The 20-task benchmark adds another six datasets to the 14-task configuration: EMNIST~\cite{cohen_emnist_2017}, CIFAR10~\cite{krizhevsky2009learning}, Food101~\cite{bossard_food-101_2014}, FashionMNIST~\cite{xiao_fashion-mnist_2017}, RenderedSST2~\cite{socher_recursive_nodate} and KMNIST~\cite{clanuwat_deep_2018}.

For LoRA setting, the 8-task benchmark includes the same datasets as in the full fine-tuning setting. The 12-task benchmark extends the previous one by introducing four new datasets: CIFAR100~\cite{krizhevsky2009learning}, Flowers102~\cite{nilsback_automated_2008}, OxfordIIITPet~\cite{parkhi_cats_2012} and STL10~\cite{coates_analysis_2011}. The 16-task benchmark further adds four datasets to the 12-task configuration: FER2013~\cite{goodfellow_challenges_2013}, CIFAR10~\cite{krizhevsky2009learning}, FashionMNIST~\cite{xiao_fashion-mnist_2017} and RenderedSST2~\cite{socher_recursive_nodate}.

The MM-MergeBench~\cite{zeng2025parameter} contain 8 multimodal datasets as seen tasks: ScienceQA~\cite{lu2022learn}, ImageNet~\cite{deng2009imagenet}, VQAv2~\cite{goyal2017making}, REC-COCO~\cite{kazemzadeh2014referitgame, mao2016generation}, OCRVQA~\cite{mishra2019ocr}, Flickr30k~\cite{plummer2015flickr30k}, VizWiz-caption~\cite{gurari2018vizwiz} and IconQA~\cite{lu2021iconqa}, along with 4 additional datasets as unseen tasks for generalizability evaluation: ImageNet-R~\cite{hendrycks2021many}, AOKVQA~\cite{schwenk2022okvqa}, Screen2Word~\cite{wang2021screen2words}, TabMWP~\cite{lu2022dynamic}. 

Since models fine-tuned on different datasets exhibit varying absolute accuracies, we measure the performance of the merged model by its \textit{average normalized accuracy}:
\begin{equation}
    \text{Average\,  Normalized\,  Accuracy}=\frac{1}{T}\sum_{i=1}^T \frac{acc(\theta_M,i)}{acc(\theta_i,i)}
\end{equation}
where $acc(\theta_i,i)$ denotes the absolute accuracy of the model fine-tuned on task $i$, $acc(\theta_M,i)$ denotes the absolute accuracy of the merged model on task $i$, and $T$ is the total number of tasks. 

\textit{Normalized Accuracy Improvement} (NAI) is another commonly-used metric which incorporates zero-shot performance of each task. NAI of task $i$ is defined as
\begin{equation}
    \text{NAI}(\theta_M, \theta_i;\theta_0)=\frac{acc(\theta_M,i)-acc(\theta_0,i)}{acc(\theta_i,i)-acc(\theta_0,i)}
\end{equation}
where $acc(\theta_0,i)$ denotes the zero-shot performance on task $i$. We use NAI to evaluate the task-wise performance of merged models in Figure~\ref{fig_mag_dir},~\ref{fig_metric_acc} and~\ref{fig_metric_boxplot}.

\subsection{Implementation Details}
\label{subsec:implementation}
\textbf{Compute resources.}
All experiments involving the fine-tuning and merging of vision models are performed on a single NVIDIA RTX 4090, while the merging of vision-language models is conducted on a single NVIDIA A6000.

\noindent
\textbf{Choice of rescaling coefficient.} For DC-Merge and other baselines which require a rescaling coefficient $\alpha$ before merging, we integrate the merged multi-task vector $\Delta \widetilde{\boldsymbol{W}}$ with the pretrained weights $\boldsymbol{W}_0$ to obtain the merged model by:
\begin{equation}
\widetilde{\boldsymbol{W}}=\boldsymbol{W}_0+\alpha\Delta \widetilde{\boldsymbol{W}},
\end{equation}
where $\alpha$ is chosen on a held-out validation set. Exceptionally, we adopt a fixed $\alpha=2.0$ for DC-Merge following RobustMerge~\cite{zeng2025parameter} in MM-MergeBench.

\noindent
\textbf{Choice of checkpoints.}
We directly utilize the checkpoints released by TSV-M for evaluation in Table~\ref{tab:FFT_task_acc}, which ensures a fair comparison with the state-of-the-art method Iso-CTS and other baselines. Similarly, we adopt the checkpoints released by RobustMerge in MM-Merge-Bench evaluation (Table~\ref{tab:8_pertask_MM}) for fair comparison with baseline methods. Unlike previous works that only consider evaluation for LoRA fine-tuned vision models on the 8-task benchmark, we extend our evaluation to 12-task and 16-task benchmarks. Thus, we perform LoRA fine-tuning of vision models by ourselves and results with respect to LoRA fine-tuned vision models are based on our checkpoints unless otherwise specified. We present the training details as follows.

\noindent
\textbf{Fine-tuning vision models with LoRA.}
We obtain the CLIP visual encoders \model{ViT-B-32}, \model{ViT-B-16} and \model{ViT-L-14} from Hugging Face (HF). These models are then fine-tuned by LoRA using \model{peft} library. Across all our experiments of vision models, we set the LoRA target modules to the query, key, value and output projection weights, which are the only learnable modules. We set the LoRA rank to be 16, LoRA alpha to be 16, LoRA dropout to be 0.1 and disable the use of bias parameters. All vision models are trained using the AdamW optimizer, with a cosine learning rate scheduler using Cross-Entropy loss. We use a standard learning rate of 3e-4, weight decay of 1e-1 and label smoothing set to 0 across all datasets and different types of visual encoders.
The text encoder in the CLIP model remains frozen across all our experiments of vision models.


\noindent
\textbf{Choice of hyperparameters.}
We adopt TIES-Merging (TIES)~\cite{yadav2023tiesmerging} for aggregation in the shared cover space across all our experiments, and the top-$k$ parameter (i.e., the percentage of elements retained) in TIES is fixed to $0.1$. For hyperparameter $r$, we plot the performance of our method with different $r$ values in Figure~\ref{fig:different_rs} for both LoRA and FFT settings. 
Consequently, in all LoRA fine-tuning scenarios including vision models and vision-language models, we set $r$ to the LoRA rank. For FFT setting, we set $r$ to $\lfloor \frac{min(m,n)}{T}\rfloor$ with each task vector $\boldsymbol{\Delta}_i \in \mathbb{R}^{m \times n}$.
\begin{figure}[h]
    \centering
    \includegraphics[width=\linewidth]{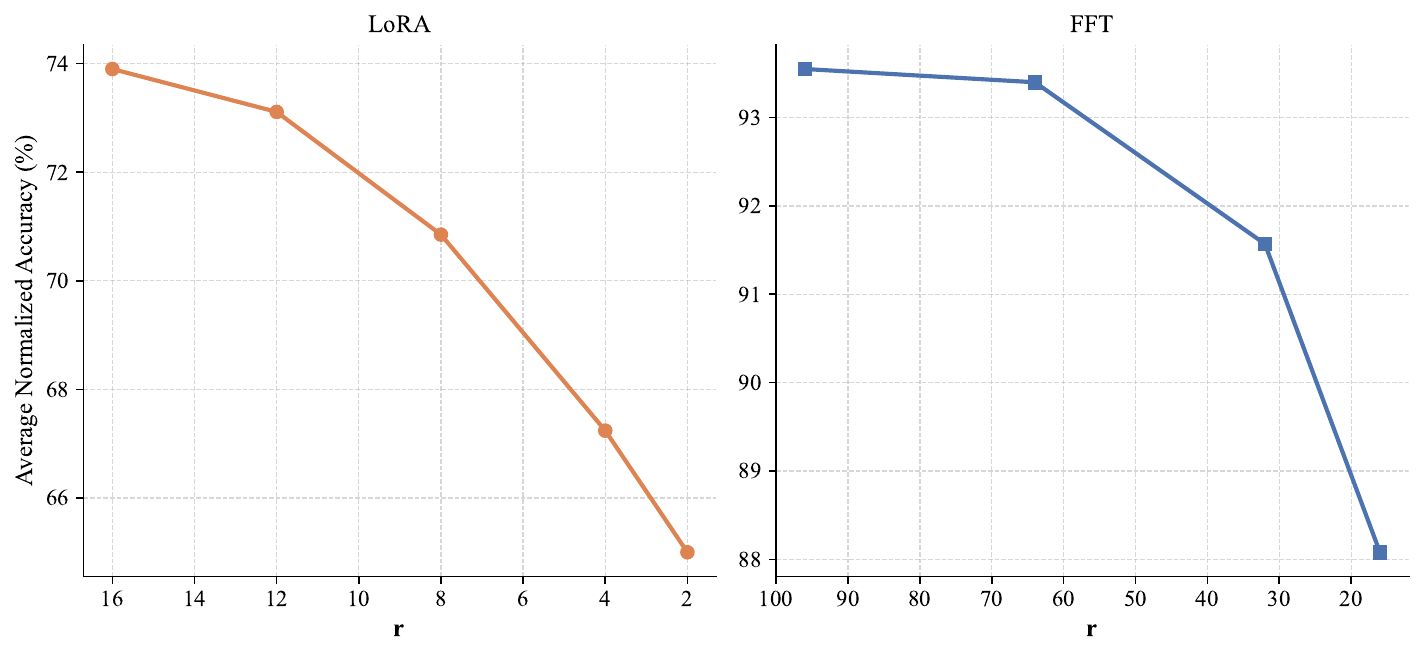}
    \caption{
    Impact of hyperparameter $r$ in LoRA and FFT settings. We report average normalized accuracy of DC-Merge on \model{ViT-B-32} 8-task benchmark. The empirical results align with our motivation to preserve the directional geometry of each task vector after merging as much as possible.
    }
    \label{fig:different_rs}
    \vspace{-0.3cm}
\end{figure}

\noindent
\textbf{Baseline adaptation in LoRA settings.} The state-of-the-art approaches TSV-M~\cite{tsv} and Iso-CTS~\cite{marczak2025notaskleftbehind} are also based on SVD. Unfortunately, the former ignores merging models fine-tuned by LoRA and the latter does not specify details of LoRA model merging neither in its paper nor in its codebase. We assume that ranks of the merged vectors $\Delta_{\text{TSV-M}}$ and $\Delta_{\text{Iso-CTS}}$ should be $T \times\text{LoRA\_rank}$ where $T$ is the total number of tasks. This setting aligns with our hyperparameter $r$ in LoRA setting.

\noindent
\textbf{Merging parameters with non-matrix structure.}
Some weights of the neural networks are represented by vectors, e.g. bias vectors and parameters of layer normalization~\cite{ba2016layer}. Following previous works~\cite{tsv, marczak2025notaskleftbehind}, we apply simple averaging to merge these parameters.

\subsection{Details of Perturbation for Figure~\ref{fig_mag_dir}}
\label{sec:fig_imple}
We separately construct controllable energy distribution and directional perturbation to verify that maintaining directional geometry is the key to maintaining performance.

\noindent
\textbf{Controllable Energy Distribution Perturbation.}
When keeping the directions $\boldsymbol{U}_{i}, \boldsymbol{V}_{i}$ fixed and perturbing only the energy distribution $\boldsymbol{\sigma}_{i}$ to $\widehat{\boldsymbol{\sigma}}_{i}$, we have
\begin{equation}
\mathrm{CosSim}(\boldsymbol{\Delta}_{i}, \widehat{\boldsymbol{\Delta}}_{i})
= \mathrm{CosSim}(\boldsymbol{\sigma}_{i}, \widehat{\boldsymbol{\sigma}}_{i}),
\end{equation}
where $\widehat{\boldsymbol{\Delta}}_{i}=
\boldsymbol{U}_{i}\,\mathrm{diag}(\widehat{\boldsymbol{\sigma}}_{i})\,\boldsymbol{V}_{i}^{\top}$. Hence, to obtain a controllable energy distribution perturbation, it suffices to construct $\widehat{\boldsymbol{\sigma}}_{i}$ such that
$\mathrm{CosSim}(\boldsymbol{\sigma}_{i}, \widehat{\boldsymbol{\sigma}}_{i})=p$, where $p\in[0,1]$ prescribes the perturbation degree.

Let $\boldsymbol{\sigma}\equiv \boldsymbol{\sigma}_{i}\in\mathbb{R}^{r}_{\ge 0}$ and let $\overline{\boldsymbol{\sigma}}\in\mathbb{R}^{r}_{\ge 0}$ be a \emph{balanced} spectrum (e.g., the averaging or linearly smoothed spectrum that preserves total energy).
We define the unit reference $\boldsymbol{s} \;=\; \frac{\boldsymbol{\sigma}}{\|\boldsymbol{\sigma}\|_{2}}$
and extract a component orthogonal to $\boldsymbol{s}$ from $\overline{\boldsymbol{\sigma}}$ by Gram-Schmidt:
\begin{equation}
\boldsymbol{s}^{\perp}  = \frac{\overline{\boldsymbol{\sigma}} - \left(\boldsymbol{s}^{\top} \overline{\boldsymbol{\sigma}}\right) \boldsymbol{s}}{\left\| \overline{\boldsymbol{\sigma}} - \left(\boldsymbol{s}^{\top} \overline{\boldsymbol{\sigma}}\right) \boldsymbol{s}\right\|_{2}}
\end{equation}

For any target $p\in[0,1]$, we set
\begin{equation}
\label{eq:s_hat_combo}
\widehat{\boldsymbol{s}}
\;=\;
p\,\boldsymbol{s} \;+\; \sqrt{1-p^{2}}\,\boldsymbol{s}^{\perp},
\qquad
\widehat{\boldsymbol{\sigma}}
\;=\;
\|\boldsymbol{\sigma}\|_{2}\;\widehat{\boldsymbol{s}}.
\end{equation}

By definition we have $\|\widehat{\boldsymbol{s}}\|_{2}=1$, $\|\boldsymbol{s}\|_{2}=1$, and $\boldsymbol{s}^{\top}\boldsymbol{s}^{\perp}=0$. Therefore,
\begin{equation}
\begin{aligned}
\mathrm{CosSim}(\boldsymbol{\sigma}, \widehat{\boldsymbol{\sigma}})
& =
\frac{\boldsymbol{\sigma}^{\top}\widehat{\boldsymbol{\sigma}}}{\|\boldsymbol{\sigma}\|_{2}\,\|\widehat{\boldsymbol{\sigma}}\|_{2}}
=
\boldsymbol{s}^{\top}\widehat{\boldsymbol{s}}
\\& =
p\,\boldsymbol{s}^{\top}\boldsymbol{s}+\sqrt{1-p^{2}}\,\boldsymbol{s}^{\top}\boldsymbol{s}^{\perp}
=
p.
\end{aligned}
\end{equation}
Hence Eq.~\eqref{eq:s_hat_combo} yields an energy distribution perturbation with prescribed cosine $p$ while keeping directions $(\boldsymbol{U}_{i},\boldsymbol{V}_{i})$ unchanged.

\noindent
\textbf{Controllable Directional Perturbation.}
To investigate the sensitivity of model merging to changes in knowledge directions, we introduce a procedure for applying controllable directional perturbations such that the resulting representation $\widehat{\boldsymbol{\Delta}}_{i}$ satisfies $\mathrm{DirSim}(\boldsymbol{\Delta}_{i}, \widehat{\boldsymbol{\Delta}}_{i}) = p$.  
Specifically, we perturb the orthogonal directions $\boldsymbol{U}_i$ and $\boldsymbol{V}_i$ by mixing them with randomly sampled bases from their orthogonal complements.

Let $\boldsymbol{U}_{i}^{\perp}$ and $\boldsymbol{V}_{i}^{\perp}$ denote the orthogonal complements of $\boldsymbol{U}_{i}$ and $\boldsymbol{V}_{i}$, respectively.  
We randomly select $r$ orthonormal columns $\widehat{\boldsymbol{U}}_{i}\subset\boldsymbol{U}_{i}^{\perp}$ and $\widehat{\boldsymbol{V}}_{i}\subset\boldsymbol{V}_{i}^{\perp}$, then construct the perturbed bases $\widehat{\boldsymbol{U}}_{i}(p)$ and $\widehat{\boldsymbol{V}}_{i}(p)$ as linear combinations:
\begin{equation}
\begin{aligned}
\widehat{\boldsymbol{U}}_{i}(p)
&= \sqrt{p}\,\boldsymbol{U}_{i} + \sqrt{1-p}\,\widehat{\boldsymbol{U}}_{i},\\
\widehat{\boldsymbol{V}}_{i}(p)
&= \sqrt{p}\,\boldsymbol{V}_{i} + \sqrt{1-p}\,\widehat{\boldsymbol{V}}_{i},
\end{aligned}
\end{equation}
which satisfy the orthonormality $\widehat{\boldsymbol{U}}_{i}(p)^{\top}\widehat{\boldsymbol{U}}_{i}(p)=
\widehat{\boldsymbol{V}}_{i}(p)^{\top}\widehat{\boldsymbol{V}}_{i}(p)=\boldsymbol{I}$.  
The perturbed task vector $\widehat{\boldsymbol{\Delta}}_{i}$ with directional geometry $(\widehat{\boldsymbol{U}}_{i}(p), \widehat{\boldsymbol{V}}_{i}(p))$ achieves $\mathrm{DirSim}(\boldsymbol{\Delta}_{i}, \widehat{\boldsymbol{\Delta}}_{i}) = p$, which can be easily verified by:
\begin{equation}
\begin{aligned}
\mathrm{DirSim}(\boldsymbol{\Delta}_{i}, \widehat{\boldsymbol{\Delta}}_{i}) & =  \frac{\mathrm{tr} \left(\boldsymbol{U}_{i}^{\top}  \widehat{\boldsymbol{U}}_{i}(p)\widehat{\boldsymbol{V}}_{i}(p)^{\top}\boldsymbol{V}_{i} \right)}{r}
\\& = \frac{\mathrm{tr} \left(p \cdot \boldsymbol{I}_{r\times r} \right)}{r} = p,
\end{aligned}
\end{equation}
where the second equality holds because $\boldsymbol{U}_{i}^{\top}  \widehat{\boldsymbol{U}}_{i}(p) = \widehat{\boldsymbol{V}}_{i}(p)^{\top}\boldsymbol{V}_{i} = \sqrt{p} \cdot \boldsymbol{I}_{r\times r}$.

\subsection{Smoothing Strategies}
\label{sec:smoothing}

\begin{table*}[t]
\centering
\resizebox{0.9\textwidth}{!}{%
\tablestyle{9pt}{1.1}
\begin{tabular}{lccccccc}
\toprule
\multicolumn{1}{c}{\multirow{2}{*}[-0.5ex]{Method}} & \multicolumn{3}{c}{\model{ViT-B-16}} & \multicolumn{3}{c}{\model{ViT-L-14}} & \multicolumn{1}{c}{\model{LLaVA-v1.5-7B}}\\ 
\cmidrule{2-8}
 & 8 tasks & 12 tasks & \multicolumn{1}{c|}{16 tasks} & 8 tasks & 12 tasks & \multicolumn{1}{c|}{16 tasks} & 8 tasks\\ 
\midrule
No smoothing        & 71.64 & 76.13 & \multicolumn{1}{c|}{75.48} & 83.96 & 85.63 & \multicolumn{1}{c|}{80.86} & 79.10\\
Linear smoothing  & 78.69 (\textcolor{green!50!black}{+7.05}) & 80.35 (\textcolor{green!50!black}{+4.22}) & \multicolumn{1}{c|}{78.82 (\textcolor{green!50!black}{+3.34})} & 89.22 (\textcolor{green!50!black}{+5.26}) & 89.11 (\textcolor{green!50!black}{+3.48}) & \multicolumn{1}{c|}{85.38 (\textcolor{green!50!black}{+4.52})} & 87.58 (\textcolor{green!50!black}{+8.48})\\
Averaging           & 78.86 (\textcolor{green!50!black}{+7.22}) & 80.56 (\textcolor{green!50!black}{+4.43}) & \multicolumn{1}{c|}{78.91 (\textcolor{green!50!black}{+3.43})} & 89.42 (\textcolor{green!50!black}{+5.46}) & 89.31 (\textcolor{green!50!black}{+3.68}) & \multicolumn{1}{c|}{85.71 (\textcolor{green!50!black}{+4.85})} & 87.91 (\textcolor{green!50!black}{+8.81})\\
\bottomrule
\end{tabular}%
}
\caption{Performance with different smoothing strategies. We report average normalized accuracy.}
\label{lora_b32_smoothing_strategy_app}
\vspace{-0.4cm}
\end{table*}

\begin{figure*}[t]
    \centering
    
    \begin{subfigure}[t]{0.40\textwidth}
        \centering
        \includegraphics[width=\textwidth]{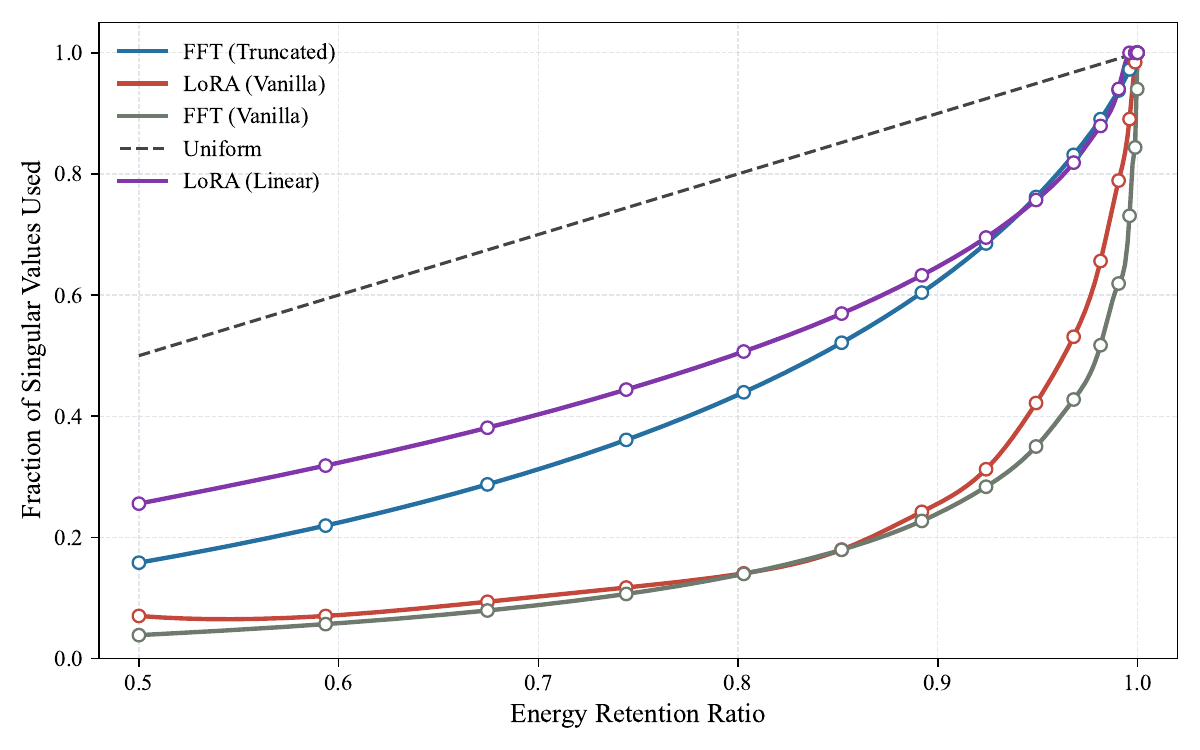}
        \caption{}
        \label{fig:fft_energy_frac}
    \end{subfigure}
    \hfill
    \begin{subfigure}[t]{0.57\textwidth}
        \centering
        \includegraphics[width=\textwidth]{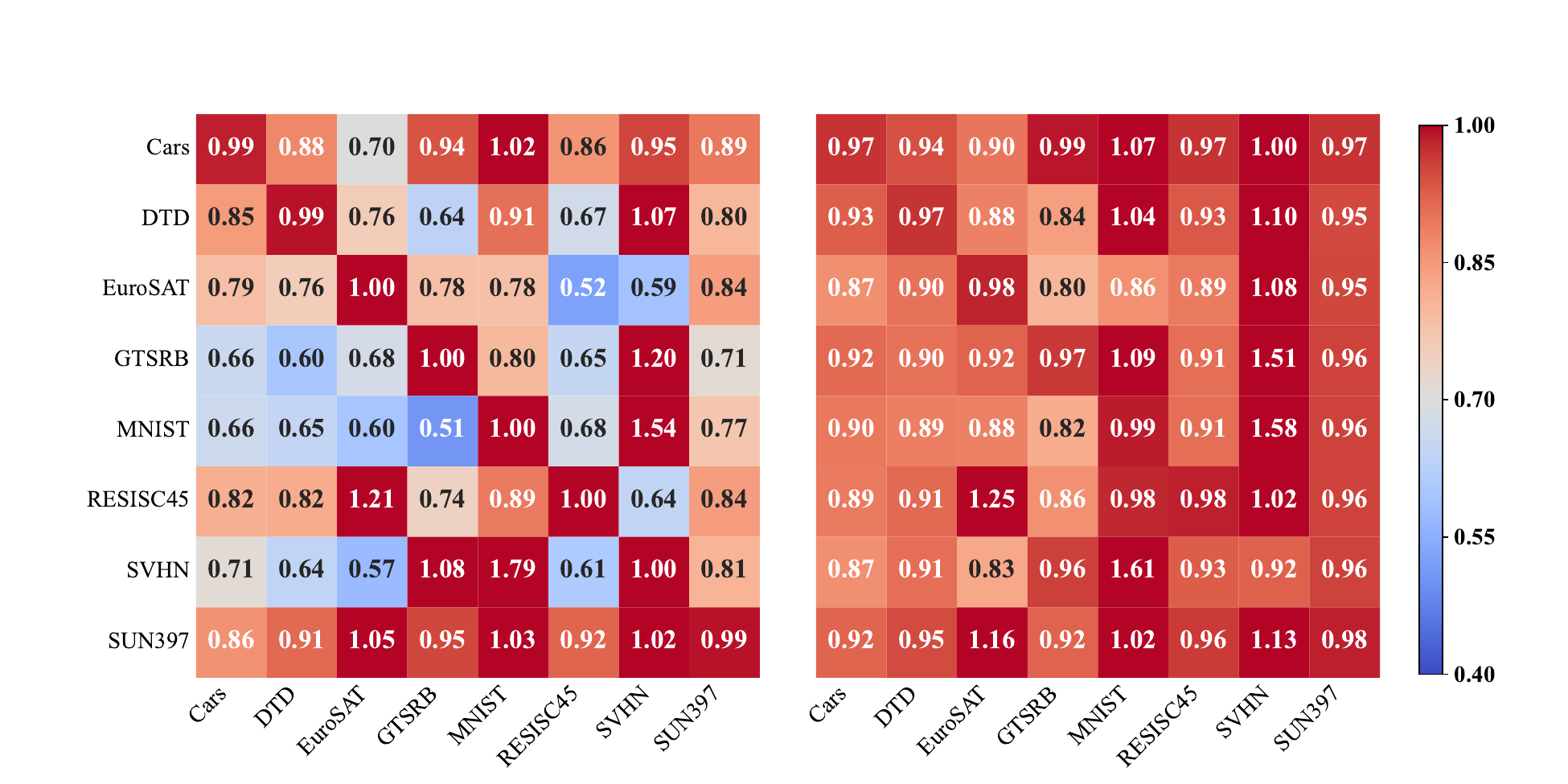}
        \caption{}
        \label{fig:transfer_fft}
    \end{subfigure}
    \vspace{-0.2cm}
    \caption{%
        (a) The relationship between energy retention ratios and fractions of singular values averaged across all datasets. The result is based on the \model{ViT-B-32} 8-task benchmark. FFT (Truncated) considers reconstructing each task vector using the top-$r$ singular values (and their associated singular vectors) in the full fine-tuning setting without additional smoothing.  
        (b) Comparison of cross-task transferability w/o (left) and w/ (right) further balancing the energy distribution after truncating long-tailed energy parts. Compared with LoRA setting, applying energy smoothing in FFT scenario yields relatively small improvements in generalization.
    }
    \label{fig:fft_combined}
    \vspace{-0.3cm}
\end{figure*}

In line 6 of Algorithm~\ref{alg:ioa}, we mention the smoothing of $\boldsymbol{\Sigma}_i^{(r)}$. Here, we propose two \emph{explicit} smoothing strategies as follows.

\noindent
\textbf{Averaging.} We consider the simplest and most effective form of smoothing by replacing all top-$r$ singular values with their average value: 
\begin{equation}
    \overline{\boldsymbol{\sigma}}_{i, \text{avg}} = \left(\frac{1}{r}\sum_{j=1}^{r} \sigma_i^j \right) \boldsymbol{1}_r.
\end{equation}

\noindent
\textbf{Linear Smoothing.} Beyond uniform averaging, we further consider a more flexible linear smoothing strategy that controls the energy flattening degree through a single hyperparameter $\rho$. Specifically, we first constrain the ratio between the largest and smallest smoothed singular values to be no greater than $\rho$, i.e., 
\begin{equation}
    \frac{\overline{\sigma}_{\max}}{\overline{\sigma}_{\min}} = \min\!\left(\frac{\sigma_{\max}}{\sigma_{\min}}, \rho\right).
\end{equation}
We then generate a linearly decreasing distribution $\boldsymbol{w} = [w_1, w_2, \dots, w_r]$ whose entries sum to one, such that $w_1 / w_r = \overline{\sigma}_{\max} / \overline{\sigma}_{\min}$.  
Finally, we scale this distribution by the sum of singular values, yielding the smoothed energy distribution
\begin{equation}
    \overline{\boldsymbol{\sigma}}_{i, \text{linear}} = \left(\sum_{j=1}^{r} \sigma_j\right) \boldsymbol{w}.
\end{equation}
This procedure maintains the relative ordering of the singular values while reducing their extreme disparity. We set $\rho$ to $5.0$ in all LoRA scenarios, including vision models and vision-language models. We provide a comprehensive comparison between averaging and linear smoothing in Table~\ref{lora_b32_smoothing_strategy} for \model{ViT-B-32} and Table~\ref{lora_b32_smoothing_strategy_app} for \model{ViT-B-16} and \model{ViT-L-14}. We report the better result of the two aforementioned smoothing strategies in Table~\ref{tab:LoRA_task_acc} and~\ref{tab:8_pertask_MM}. 

Besides, directly utilizing top-$r$ singular values along with their corresponding singular vectors of each task vector to obtain $\overline{\boldsymbol{\Delta}}_i$ without any additional smoothing (i.e., skipping line 6 in Algorithm~\ref{alg:ioa}) can also serve as an \emph{implicit} smoothing strategy, as it truncates the long-tailed energy parts. We only apply this truncation strategy to full fine-tuned vision models since we set $r$ to LoRA rank in LoRA setting. Surprisingly, we notice that applying truncation to long-tailed singular values without further smoothing in FFT scenario can generally exhibit similar energy distribution to explicitly linearly smoothing in LoRA setting.

\begin{figure}[h]
    \centering
    \includegraphics[width=0.95\linewidth, trim={0 0.3cm 0 0.3cm},
                 clip]{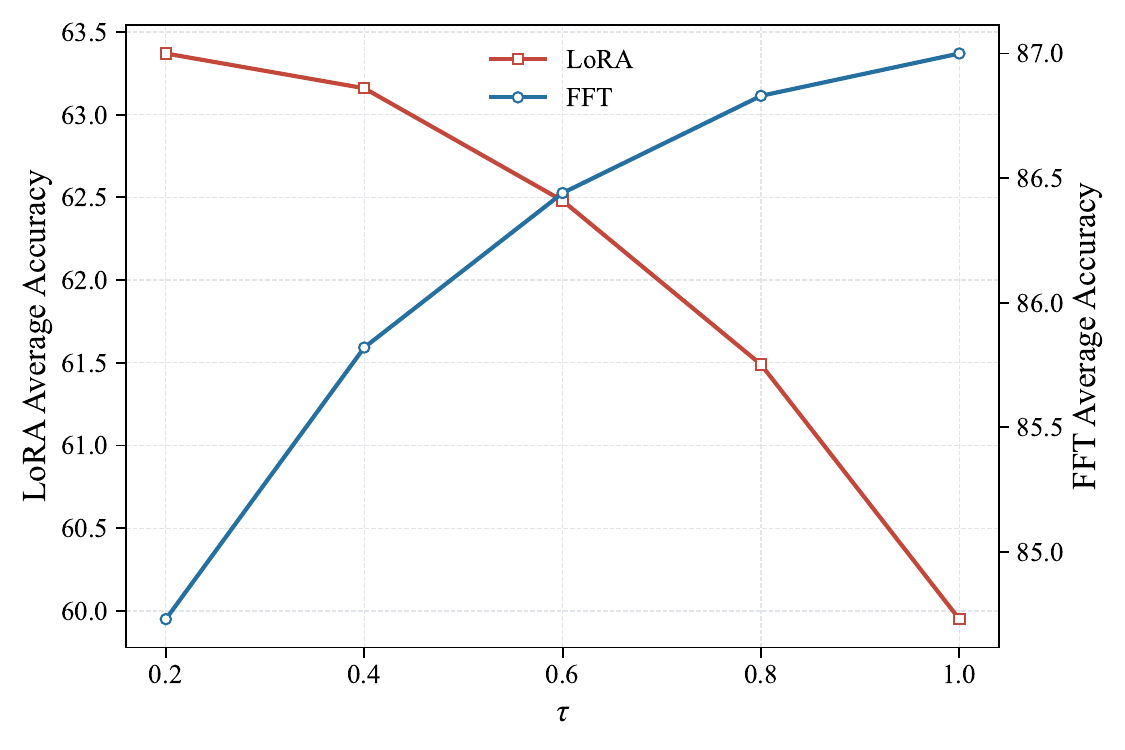}
    \caption{
    The relationship between average absolute accuracy of DC-Merge in LoRA/FFT setting and interpolation coefficient $\tau$.
    }
    \label{fig:smoothing_interpolation}
    \vspace{-0.4cm}
\end{figure}

\begin{table*}[h]
\resizebox{1\textwidth}{!}{
    \tablestyle{10pt}{1}
    \begin{tabular}{p{1.5cm}<{\centering} p{1cm}<{\centering}p{1cm}<{\centering}p{1cm}<{\centering}p{1cm}<{\centering}p{1cm}<{\centering}p{1cm}<{\centering}p{1cm}<{\centering}p{1cm}<{\centering}p{1cm}<{\centering}p{1cm}}
    \toprule
    \multicolumn{1}{c}{\multirow{2}{*}[-0.5ex]{\textbf{Method}}} & \multicolumn{9}{c}{\model{Datasets} } \\
    \cmidrule{2-10} 
    & Cars &DTD& EuroSAT & GTSRB &MNIST&RESISC45&SUN397&SVHN&\textbf{Average}
        \\
        \midrule
    \multicolumn{1}{c}{Task Arithmetic} & $82.0$&$73.6$&$48.8$&$42.1$&$53.1$&$71.5$&$97.5$&$41.2$&$63.7$ \\
        \multicolumn{1}{c}{KnOTS-TIES} & $82.7$&$73.7$&$49.3$&$48.9$&$68.9$&$70.9$&$95.5$&$\underline{53.8}$&$68.0$\\
        \multicolumn{1}{c}{WUDI-Merging} &$82.4$&$73.5$&$48.6$&$46.8$&$54.5$&$72.4$&$96.0$&$46.5$&$65.1$\\
        \multicolumn{1}{c} {TSV-M} & $\underline{83.9}$&$75.1$&$\underline{52.4}$&$45.5$&$58.3$&$73.1$&$97.6$&$45.3$&$66.4$\\
        \multicolumn{1}{c}{Iso-CTS} & $83.3$&$\mathbf{84.7}$&$49.0$&$\mathbf{79.6}$&$\underline{69.4}$&$\mathbf{82.3}$&$\mathbf{99.0}$&$53.2$&$\underline{75.1}$\\
        \bottomrule
        \rowcolor{gray!15}
        \multicolumn{1}{c}{\textbf{DC-Merge}} & $\mathbf{90.9}$ & $\underline{79.6}$ & $\mathbf{65.5}$ & $\underline{54.0}$ & $\mathbf{92.6}$& $\underline{75.0}$ & $\underline{98.2}$ & $\mathbf{66.8}$ & $\boldsymbol{77.8}$\\
        \bottomrule
    \end{tabular}
}
\caption{Performance on \model{ViT-B-32} 8-task benchmark using checkpoints provided by KnOTS. We report normalized accuracy.}
\label{tab:8task_lora_b32_knots}
\vspace{-0.2cm}
\end{table*}

It is worth noting that energy smoothing can enhance cross-task transferability of each task vector at a cost of task performance fidelity (see Figure~\ref{fig_transfer} and~\ref{fig:transfer_fft}), so there may exist a trade-off between its pros and cons. To verify this, we perform a linear interpolation between original singular values of each task vector and averaging smoothed ones by:
\begin{equation}
\overline{\boldsymbol{\Sigma}}_i^{(r)}=\tau\boldsymbol{\Sigma}_{i,\text{orig}}^{(r)}+(1-\tau)\boldsymbol{\Sigma}_{i,\text{avg}}^{(r)}
\end{equation}
The results are presented in Figure~\ref{fig:smoothing_interpolation}. With $\tau$ decreases, the energy distribution of each task vector tends to be smoother. We notice a consistent performance improvement as each $\overline{\boldsymbol{\Sigma}}_i^{(r)}$ is smoothed gradually in LoRA setting, while a reverse trend can be observed in FFT setting. This phenomenon suggests that the aforementioned truncation strategy is sufficient to ensure the top-$r$ knowledge components to be adequately expressed within each task vector, while further smoothing induces additional bias for individual tasks and thereby harms overall performance. We thus explicitly apply smoothing to LoRA fine-tuned task vectors across all experiments while obtaining $\overline{\boldsymbol{\Delta}}_i$ solely by retaining top-$r$ knowledge vectors as implicitly smoothing in FFT setting.

\section{Additional Experiments and Analyses}

In this section, we extend our evaluations of DC-Merge to additional checkpoints, larger task scales, and varying smoothing levels, demonstrating its strong robustness and scalability. Furthermore, we empirically verify that the two key modules in DC-Merge are both effective to preserve directional geometry of task vectors during merging and assess the generalization capability of DC-Merge on vision tasks. Finally, we provide a step-by-step analysis on computational complexity of our method.  
\subsection{Additional checkpoints}
\label{sec:other_ckpts}

Ideally, the performance of a well-designed model merging method is robust to various checkpoints. For LoRA setting, we present the performance of DC-Merge along with all baselines on \model{ViT-B-32} checkpoints provided by KnOTS~\cite{stoica2024knots} in Table \ref{tab:8task_lora_b32_knots}. DC-Merge still achieves the state-of-the-art performance, outperforming Iso-CTS by 2.7\%. For FFT setting, we provide the results on 8-task benchmarks using checkpoints from Task Arithmetic~\cite{ilharco2023task} in Table~\ref{tab:vit_results_ta}, where the capability of DC-Merge is consistently superior (or comparable) to state-of-the-art methods.

\vspace{-0.2cm}
\begin{table}[H]
\centering
\resizebox{\columnwidth}{!}{%
\tablestyle{10pt}{1}
\begin{tabular}{c|ccc}
\toprule
\textbf{Method} & \model{ViT-B-32} & \model{ViT-B-16} & \model{ViT-L-14} \\
\midrule
Zeroshot  & $48.3$ & $55.5$ & $64.7$ \\
Individual & $90.5$ & $92.6$ & $94.2$ \\
\midrule
Task Arithmetic & $70.5$ & $74.6$ & $84.6$ \\
Fisher Merging  & $68.3$ & $71.7$ & $83.7$ \\
RegMean         & $71.8$ & $76.6$ & $82.2$ \\
PCB-Merging             & $76.3$ & $81.5$ & $87.5$ \\
TSV-M            & $83.8$ & $87.2$ & $91.5$ \\
Iso-CTS            & $\underline{84.0}$ & $\underline{88.6}$ & $\mathbf{92.9}$ \\
\midrule
\rowcolor{gray!15}
\textbf{DC-Merge}       & $\mathbf{85.1}$ & $\mathbf{88.8}$ & $\underline{92.7}$ \\
\bottomrule
\end{tabular}
}
\caption{Performance on 8-task benchmark using checkpoints provided by Task Arithmetic. We report average absolute accuracy.}
\label{tab:vit_results_ta}
\vspace{-0.2cm}
\end{table}

\subsection{Complementary Modules in DC-Merge}
We conduct comparative experiments on vanilla Task Arithemetic (TA)~\cite{ilharco2023task} to demonstrate that both \emph{energy smoothing} (ES) and \emph{cover space merging} (CSM) are conducive to maintaining the directional consistency between merged vector and original task vectors. As illustrated in Figure~\ref{fig:ta_with_wo_smoothing}, individually applying ES and CSM significantly boosts the projected $\mathrm{DirSim}(\Delta \boldsymbol{W}_i, \Delta \widetilde{\boldsymbol{W}}_{i})$ of each task. Furthermore, by combining ES and CSM, we can achieve higher projected $\mathrm{DirSim}$, indicating that the two key modules in DC-Merge are complementary.

\begin{figure}[h]
    \centering
    \includegraphics[width=\linewidth, trim={0 0.4cm 0 0.4cm},
                 clip]{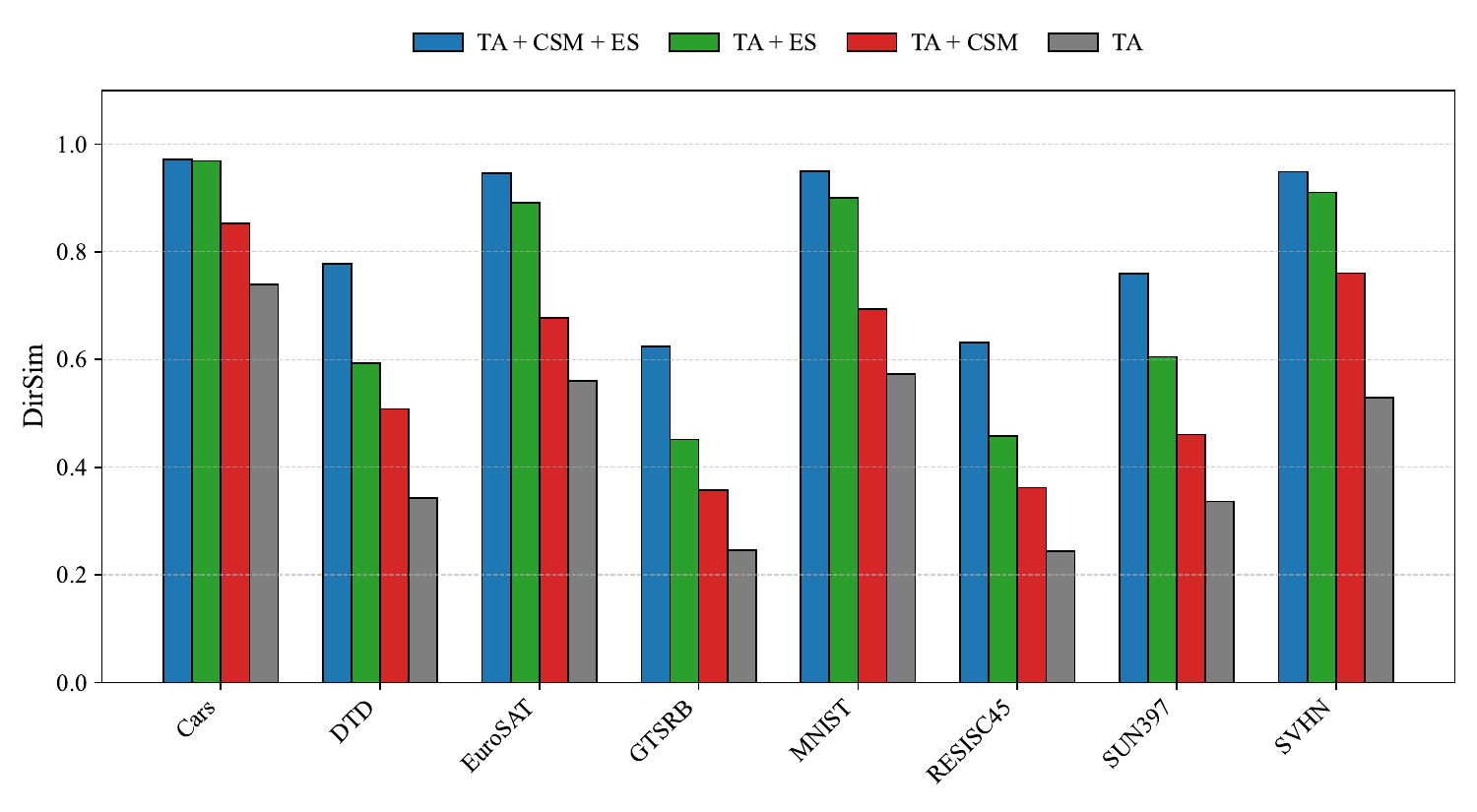}
    \caption{
    $\mathrm{DirSim}$ of TA w/o and w/ CSM and ES on \model{ViT-B-32} 8-task benchmark in LoRA setting. Both two modules in DC-Merge enhance $\mathrm{DirSim}(\Delta \boldsymbol{W}_i, \Delta \widetilde{\boldsymbol{W}}_{i})$ of each task. 
    }
    \label{fig:ta_with_wo_smoothing}
    \vspace{-0.3cm}
\end{figure}

We also present the task-wise projected $\mathrm{CosSim}$ of TA and DC-Merge on \model{ViT-B-32} 8-task benchmark in Figure~\ref{fig:cossim_ta_ours}. While TA exhibits higher projected $\mathrm{CosSim}$ on each individual task, its projected $\mathrm{DirSim}$ is much lower than DC-Merge across all tasks, indicating that TA fails to retain the directions of weaker but semantically important knowledge components within each task vector in the merging process and thereby suffers more performance degradation on each task than our method. 
\begin{figure}[h]
    \centering
    \includegraphics[width=\linewidth, trim={0 0.5cm 0 0.3cm},
                 clip]{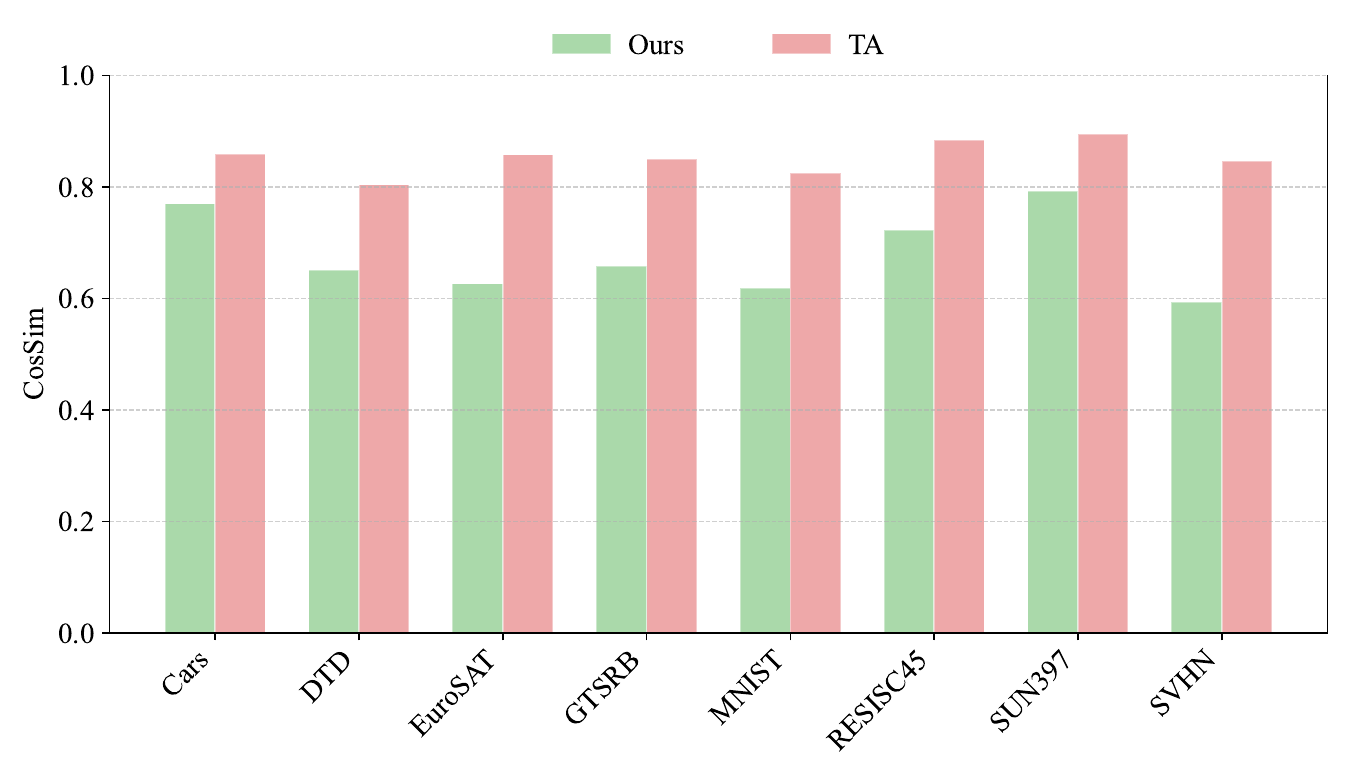}
    \caption{
    Projected $\mathrm{CosSim}$ of TA and DC-Merge (Ours) on \model{ViT-B-32} 8-task benchmark in LoRA setting.
    }
    \label{fig:cossim_ta_ours}
    \vspace{-0.5cm}
\end{figure}

\begin{table*}[t]
\resizebox{1\textwidth}{!}{
    \tablestyle{5pt}{1.1}
    \begin{tabular}{p{1.5cm}<{\centering} p{1cm}<{\centering}p{1cm}<{\centering}p{1cm}<{\centering}p{1cm}<{\centering}p{1cm}<{\centering}p{1cm}<{\centering}p{1cm}<{\centering}p{1cm}<{\centering}p{1cm}<{\centering}|p{1cm}<{\centering}p{1cm}<{\centering}p{1cm}<{\centering}p{1cm}<{\centering}p{1cm}<{\centering}}
    \toprule
    \multicolumn{1}{c}{\multirow{2}{*}[-0.5ex]{\textbf{Method}}} & \multicolumn{9}{c|}{\model{Seen Tasks} } & \multicolumn{5}{c}{\model{Unseen Tasks} }\\
    \cmidrule{2-15} 
    & Cars &DTD& EuroSAT & GTSRB &MNIST&RESISC45&SUN397&SVHN&\textbf{Average}&CIFAR100&Flowers&Pets&STL10&\textbf{Average}
        \\
        \cmidrule{1-15} 
        \multicolumn{1}{c|}{Individual} &  $76.61$&$67.34$&$98.19$&$98.29$&$99.15$&$93.97$&$72.49$&$96.50$&$87.82$&$-$&$-$&$-$&$-$&$-$\\
        \multicolumn{1}{c|}{Zeroshot} &  $59.49$&$44.15$&$44.30$&$32.27$&$47.94$&$60.25$&$63.21$&$31.42$&$47.88$&$64.61$&$66.18$&$87.74$&$96.77$&$78.83$\\
        \cmidrule{1-15}
    \multicolumn{1}{c|}{Task Arithmetic} & $59.82$&$44.31$&$46.52$&$34.71$&$63.43$&$65.78$&$62.61$&$45.25$&$52.80$&$43.83$&$\mathbf{62.30}$&$82.12$&$96.14$&$71.10$ \\
        \multicolumn{1}{c|}{KnOTS-TIES} & $61.39$&$44.79$&$48.26$&$43.42$&$71.24$&$65.49$&$63.48$&$49.36$&$55.93$&$47.21$&$\underline{62.11}$&$82.97$&$96.03$&$72.08$\\
        \multicolumn{1}{c|}{WUDI-Merging} &$60.89$&$45.64$&$51.82$&$39.90$&$66.96$&$67.40$&$63.32$&$46.12$&$55.25$&$46.81$&$62.05$&$82.46$&$\underline{96.19}$&$71.89$\\
        \multicolumn{1}{c|} {TSV-M} & $\underline{62.66}$&$46.54$&$55.44$&$46.88$&$75.86$&$69.37$&$\underline{63.84}$&$50.66$&$58.91$&$51.13$&$61.87$&$84.33$&$\mathbf{96.27}$&$\underline{73.40}$\\
        \multicolumn{1}{c|}{Iso-CTS} & $60.72$&$\mathbf{52.34}$&$\mathbf{64.26}$&$\underline{52.31}$&$\underline{82.80}$&$\underline{72.87}$&$63.68$&$\underline{55.10}$&$\underline{63.01}$&$\underline{52.64}$&$60.14$&$\underline{84.37}$&$95.67$&$73.21$\\
        \bottomrule
        \rowcolor{gray!15}
        \multicolumn{1}{c|}{$\textbf{DC-Merge}$} & $\mathbf{62.93}$ & $\underline{50.00}$ & $\underline{61.15}$ & $\mathbf{57.27}$ & $\mathbf{84.79}$& $\mathbf{75.29}$ & $\mathbf{64.98}$ & $\mathbf{56.93}$ & $\boldsymbol{64.17}$ & $\mathbf{56.90}$ & $61.44$ & $\mathbf{85.08}$ & $96.00$ & $\boldsymbol{74.86}$\\
        \bottomrule
    \end{tabular}
}
\caption{Performance of eight seen tasks and four unseen tasks on \model{ViT-B-32} in LoRA setting. We report absolute accuracy.}
\label{tab:8task_lora_4task_unseen_b32}
\vspace{-0.2cm}
\end{table*}

\begin{table*}[t]
\resizebox{1\textwidth}{!}{
    \tablestyle{5pt}{1}
    \begin{tabular}{p{1.5cm}<{\centering} p{1cm}<{\centering}p{1cm}<{\centering}p{1cm}<{\centering}p{1cm}<{\centering}p{1cm}<{\centering}p{1cm}<{\centering}p{1cm}<{\centering}p{1cm}<{\centering}p{1cm}<{\centering}|p{1cm}<{\centering}p{1cm}<{\centering}p{1cm}<{\centering}p{1cm}<{\centering}p{1cm}<{\centering}}
    \toprule
    \multicolumn{1}{c}{\multirow{2}{*}[-0.5ex]{\textbf{Method}}} & \multicolumn{9}{c|}{\model{Seen Tasks} } & \multicolumn{5}{c}{\model{Unseen Tasks} }\\
    \cmidrule{2-15} 
    & Cars &DTD& EuroSAT & GTSRB &MNIST&RESISC45&SUN397&SVHN&\textbf{Average}&CIFAR100&Flowers&Pets&STL10&\textbf{Average}
        \\
        \cmidrule{1-15} 
        \multicolumn{1}{c|}{Individual} &  $89.18$&$77.93$&$98.44$&$99.11$&$99.28$&$96.95$&$80.49$&$97.53$&$92.36$&$-$&$-$&$-$&$-$&$-$\\
        \multicolumn{1}{c|}{Zeroshot} &  $77.94$&$55.64$&$64.68$&$50.68$&$76.16$&$71.33$&$68.34$&$58.58$&$65.42$&$76.13$&$79.43$&$93.32$&$99.39$&$87.07$\\
        \cmidrule{1-15}
    \multicolumn{1}{c|}{Task Arithmetic} & $79.01$&$57.39$&$66.07$&$55.72$&$80.15$&$74.25$&$68.90$&$64.85$&$68.29$&$57.28$&$65.71$&$83.52$&$\underline{98.95}$&$76.36$ \\
        \multicolumn{1}{c|}{KnOTS-TIES} & $81.39$&$60.53$&$71.22$&$65.98$&$89.48$&$79.48$&$69.41$&$71.36$&$73.61$&$57.50$&$72.01$&$90.33$&$98.84$&$79.67$\\
        \multicolumn{1}{c|}{WUDI-Merging} &$79.62$&$58.19$&$66.56$&$57.26$&$85.03$&$76.51$&$68.59$&$66.49$&$69.78$&$60.81$&$69.65$&$87.26$&$\mathbf{99.03}$&$79.19$\\
        \multicolumn{1}{c|} {TSV-M} & $\underline{82.53}$&$62.61$&$75.74$&$75.12$&$88.74$&$82.52$&$\underline{70.32}$&$74.61$&$76.52$&$60.79$&$\mathbf{73.42}$&$91.55$&$98.86$&$\underline{81.15}$\\
        \multicolumn{1}{c|}{Iso-CTS} & $81.74$&$\underline{67.34}$&$\underline{84.56}$&$\underline{88.03}$&$\underline{96.96}$&$\underline{86.44}$&$69.53$&$\underline{78.50}$&$\underline{81.64}$&$\underline{61.90}$&$\underline{72.68}$&$\underline{91.62}$&$98.31$&$81.13$\\
        \bottomrule
        \rowcolor{gray!15}
        \multicolumn{1}{c|}{$\textbf{DC-Merge}$} & $\mathbf{83.17}$ & $\mathbf{68.35}$ & $\mathbf{84.78}$ & $\mathbf{88.48}$ & $\mathbf{97.06}$& $\mathbf{86.84}$ & $\mathbf{71.78}$ & $\mathbf{80.45}$ & $\boldsymbol{82.61}$ & $\mathbf{65.40}$ & $72.07$ & $\mathbf{91.79}$ & $98.45$ & $\boldsymbol{81.93}$\\
        \bottomrule
    \end{tabular}
}
\caption{Performance of eight seen tasks and four unseen tasks on \model{ViT-L-14} in LoRA setting. We report absolute accuracy.}
\label{tab:8task_lora_4task_unseen_l14}
\vspace{-0.4cm}
\end{table*}

\subsection{Generalization Capability on Vision Tasks}
\label{sec:vision_unseen_tasks}

We evaluate the generalization capability with four additional vision tasks on both \model{ViT-B-32} and \model{ViT-L-14} 8-task benchmarks in LoRA setting and the results are respectively presented in Figure~\ref{tab:8task_lora_4task_unseen_b32} and~\ref{tab:8task_lora_4task_unseen_l14}. Unlike vision-language tasks, all baselines along with DC-Merge exhibit limited generalization capabilities due to the complexity of unseen tasks. Nevertheless, DC-Merge consistently outperforms other methods across both benchmarks.


\begin{table*}[t]
\resizebox{1\textwidth}{!}{
    \tablestyle{5pt}{1.1}
    \begin{tabular}{p{1.5cm}<{\centering} p{1cm}<{\centering}p{1cm}<{\centering}p{1cm}<{\centering}p{1cm}<{\centering}p{1cm}<{\centering}p{1cm}<{\centering}p{1cm}<{\centering}p{1cm}<{\centering}p{1cm}<{\centering}|p{1cm}<{\centering}p{1cm}<{\centering}p{1cm}<{\centering}p{1cm}<{\centering}p{1cm}<{\centering}}
    \toprule
    \multicolumn{1}{c}{\multirow{2}{*}[-0.5ex]{\textbf{Method}}} & \multicolumn{9}{c|}{\model{Seen Tasks} } & \multicolumn{5}{c}{\model{Unseen Tasks} }\\
    \cmidrule{2-15} 
    & SciQA &Image& VQA& REC &OCR&VizWiz&Flickr&IconQA&\textbf{Average}&AVQA&Image-R&S2W&TabMWP&\textbf{Average}
        \\
        \cmidrule{1-15} 
        \multicolumn{1}{c}{Individual} &  $83.74$&$96.02$&$67.58$&$43.40$&$65.50$&$64.80$&$57.29$&$75.54$&$69.23$&$-$&$-$&$-$&$-$&$-$\\
        \multicolumn{1}{c}{Zeroshot}  &$ 61.73$	&$40.87$&	$62.88$&$36.10$& $41.16$&$41.03$&$49.07$&$14.09$ &$43.37$&$51.62$&$28.27$&$5.98$&$15.01$&$25.22$\\
        \multicolumn{1}{c}{Multi Task}  &$76.90$&$74.08$&$67.05$&$35.98$&$65.37$&$66.67$&$56.09$&$66.87$&$63.62$&$76.33$&$41.39$&$8.34$&$18.20$&$36.06$\\
        \cmidrule{1-15}
    \multicolumn{1}{c}{Task Arithmetic} & $71.94$&$57.49$&$67.06$&$38.90$&$62.87$&$44.80$&$49.20$&$39.21$&$53.93$&$74.78$&$37.37$&$7.52$&$13.57$&$33.31$ \\
        \multicolumn{1}{c}{DARE-Merging} & $71.59$&$57.25$&$66.26$&$39.38$&$62.56$&$44.93$&$49.13$&$39.59$&$53.84$&$73.75$&$37.67$&$7.56$&$13.62$&$33.15$\\
        \multicolumn{1}{c}{TIES-Merging} &$71.49$&$55.88$&$66.73$&$39.67$&$\mathbf{65.12}$&$44.35$&$47.06$&$34.46$&$53.09$&$73.43$&$38.44$&$7.47$&$13.23$&$33.14$\\
        \multicolumn{1}{c} {PCB-Merging} & $71.10$&$57.82$&$\mathbf{67.59}$&$38.22$&$\underline{64.35}$&$44.58$&$48.90$&$37.01$&$53.70$&$74.57$&$36.28$&$7.84$&$15.44$&$33.53$\\
        \multicolumn{1}{c} {TSV-M} & $75.66$&$\mathbf{78.61}$&$59.48$&$41.92$&$41.71$&$41.19$&$52.23$&$45.97$&$54.60$&$77.29$&$42.72$&$11.12$&$13.80$&$36.23$\\
        \multicolumn{1}{c} {WUDI-Merging} & $69.46$&$\underline{78.40}$&$56.84$&$38.95$&$38.58$&$40.23$&$51.13$&$38.82$&$51.55$&$68.56$&$41.57$&$\underline{11.16}$&$11.80$&$33.27$\\
        \multicolumn{1}{c} {Iso-CTS} & $\mathbf{77.54}$&$77.01$&$63.50$&$\underline{45.72}$&$40.75$&$42.34$&$\underline{54.74}$&$\mathbf{53.25}$&$56.86$&$78.48$&$44.81$&$\mathbf{11.21}$&$13.99$&$37.12$\\
        \multicolumn{1}{c}{RobustMerge} & $\underline{73.43}$&$\underline{65.54}$&$\underline{67.20}$&$44.80$&$62.97$&$\mathbf{46.61}$&$52.80$&$\underline{45.90}$&$\underline{57.33}$&$\underline{79.30}$&$\underline{45.79}$&$9.23$&$\underline{17.62}$&$\underline{37.99}$\\
        \bottomrule
        \rowcolor{gray!15}
        \multicolumn{1}{c}{$\textbf{DC-Merge}$} & $\underline{77.05}$ & $73.33$ & $65.87$ & $\mathbf{47.88}$ & $59.66$& $\underline{45.23}$ & $\mathbf{55.06}$ & $\underline{52.92}$ & $\mathbf{59.63}$ & $\mathbf{79.30}$ & $\mathbf{50.24}$ & $10.88$ & $\mathbf{18.94}$ & $\mathbf{39.84}$\\
        \bottomrule
    \end{tabular}
}
\caption{Task-wise absolute accuracy results on MM-MergeBench~\cite{zeng2025parameter}, containing eight seen tasks (LoRA fine-tuned) and four unseen tasks. The best results are in \textbf{bold} and the second-best are \underline{underlined}.}
\label{tab:8_pertask_MM_supp}
\vspace{-0.3cm}
\end{table*}

\subsection{Computational Complexity Analysis}
We provide a comprehensive computational complexity analysis on DC-Merge in this subsection. For simplicity, we follow previous work~\cite{marczak2025notaskleftbehind} to consider a deep neutral network with $L$ layers, and each layer consists of a single task vector $\boldsymbol{\Delta}_i\in\mathbb{R}^{n\times n}$. Besides, we consider the traditional full SVD implementation with a complexity of $\mathcal{O}(n^3)$~\cite{vasudevan2017hierarchical} and compute $r$-rank SVD by first computing full SVD and then retaining the top-$r$ singular values and their corresponding singular vectors. Supposing that the number of tasks is $T$, the step-by-step computational complexity of DC-Merge is:
\begin{itemize}
    \item Compute $r$-rank SVD for each task vector $\boldsymbol{\Delta}_i$ in each layer, with the total complexity of $\mathcal{O}(TLn^3)$;
    \item Perform energy smoothing for each task vector in each layer, with the total complexity of $\mathcal{O}(TLr)$;
    \item Reconstruct energy-balanced task vector $\boldsymbol{\Delta}_i$ layer-wise, with the total complexity of $\mathcal{O}(TLn^2r)$;
    \item Whiten concatenated basis $\boldsymbol{U}\in\mathbb{R}^{n\times k}$ and $\boldsymbol{V}\in\mathbb{R}^{n\times k}$ respectively, with the total complexity of $\mathcal{O}(TL(nk^2+k^3))$;
    \item Project each energy-balanced task vector $\boldsymbol{\Delta}_i$ onto cover space for each layer, with the total complexity of $\mathcal{O}(TL(nk^2+n^2k))$;
    \item Merge $\{\boldsymbol{M}\}_{i=1}^T$ via TA or TIES layer-wise, with the total complexity of $\mathcal{O}(TLk^2)$;
    \item Project the aggregated $\boldsymbol{\widetilde{M}}$ back to original parameter space for each layer, with the total complexity of $\mathcal{O}(TL(nk^2+n^2k))$.
\end{itemize}
As $r\le \lfloor n/T\rfloor$, the overall computational complexity of DC-Merge is $\mathcal{O}(TLn^3)$, which shares the same asymptotic computational complexity with the current state-of-art method TSV-M~\cite{tsv} and Iso-CTS~\cite{marczak2025notaskleftbehind}. In practical applications, we recommend using randomized SVD for low-rank approximation to enhance computational efficiency.


\subsection{Performance under Increasing Task Scales}
We compare the performance of DC-Merge with existing state-of-the-art methods under different task scales (i.e., the number of tasks). As illustrated in Table~\ref{tab:task_scales}, with the increase of task scales, DC-Merge consistently achieves state-of-the-art overall performance while WUDI-Merging~\cite{cheng2025whoever} suffers severe performance degradation. Furthermore, the superiority of DC-Merge becomes more significant as the task scale increases, which demonstrates the strong robustness of our method to task scales. The same trend can be observed in Table~\ref{tab:FFT_task_acc} on \model{ViT-B-16} and \model{ViT-L-14}.

\begin{table}[h]
\centering
\resizebox{\columnwidth}{!}{%
\tablestyle{8pt}{1}
\begin{tabular}{c|ccc}
\toprule
\textbf{Method} & 8 tasks & 14 tasks & 20 tasks \\
\midrule
Zeroshot  & $48.26$ & $57.21$ & $56.10$\\
Individual & $92.83$ & $90.88$ & $91.37$\\
\midrule
TSV-M             & $85.86_{(92.31)}$ & $80.06_{(87.88)}$ & $77.07_{(84.29)}$\\
WUDI-Merging            & $\underline{86.47_{(93.05)}}$ & $78.87_{(86.71)}$ & $69.90_{(76.71)}$\\
Iso-CTS            & $86.20_{(91.78)}$ & $\underline{81.71_{(89.70)}}$ & $\underline{78.05_{(85.48)}}$\\
\midrule
\rowcolor{gray!15}
\textbf{DC-Merge}       & $\mathbf{87.05_{(93.55)}}$ & $\mathbf{82.52_{(90.62)}}$ & $\mathbf{80.58_{(88.18)}}$\\
\bottomrule
\end{tabular}
}
\caption{Performance on \model{ViT-B-32} benchmarks in FFT setting. We report average absolute accuracy and subscript (in parentheses) is the average normalized accuracy.}
\label{tab:task_scales}
\vspace{-0.5cm}
\end{table}

\subsection{Performance on Varying Smoothing Levels}
In Appendix~\ref{sec:smoothing}, we propose a flexible linear smoothing strategy that controls the energy flattening degree through a single hyperparameter $\rho$.
Higher $\rho$ indicates more skewness in energy distribution.
We present the performance of DC-Merge on 8-task benchmark in LoRA setting with varying $\rho$ values in Figure~\ref{fig:rho_linear_smoothing}. Notably, DC-Merge exhibits strong robustness to the choice of $\rho$, achieving state-of-the-art performance across a wide range of smoothing levels.   
\begin{figure}[h]
    \centering
    \includegraphics[width=0.9\linewidth, trim={0 0.5cm 0 0.3cm},
                 clip]{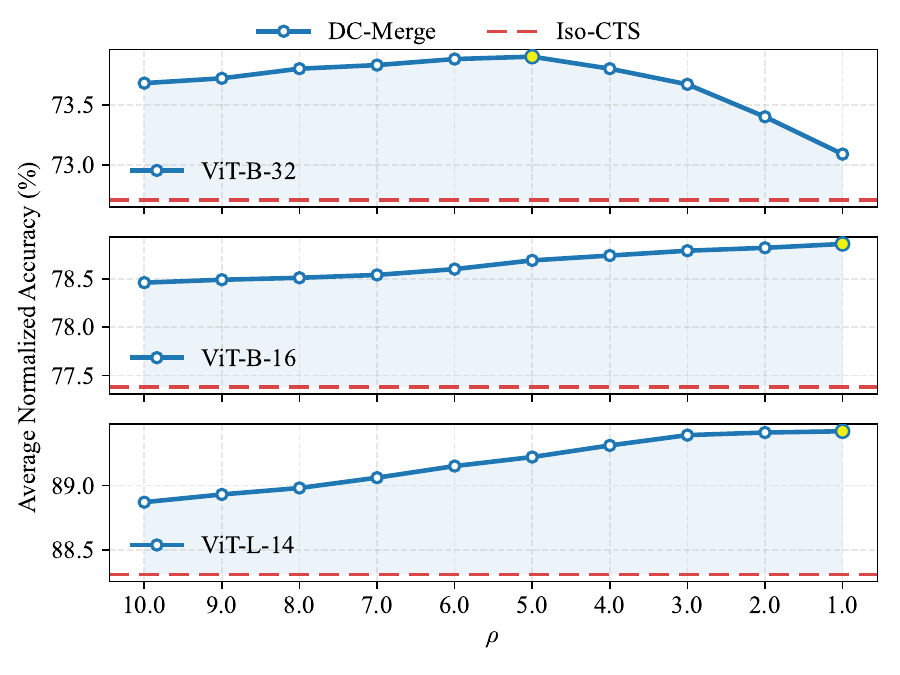}
    \caption{
    Performance of DC-Merge with respect to various smoothing levels in linear smoothing. We report average normalized accuracy. When $\rho=1.0$, linear smoothing reduces to simple averaging. 
    }
    \label{fig:rho_linear_smoothing}
    \vspace{-0.3cm}
\end{figure}

\subsection{Scaling to Vision-Language Models}
\label{sec:other_baselines_MMBench}

Table~\ref{tab:8_pertask_MM_supp} presents task-wise performance of existing baselines in Table~\ref{tab:8_pertask_MM}, along with additional state-of-the-art methods TSV-M~\cite{tsv}, WUDI-Merging~\cite{cheng2025whoever} and Iso-CTS~\cite{marczak2025notaskleftbehind}. While some state-of-the-art approaches struggle to maintain strong performance, DC-Merge consistently achieves superior results on both seen and unseen tasks, demonstrating that the effectiveness of DC-Merge is not limited to vision models, but naturally extends to large-scale multi-modal models.




\begin{figure*}[h]
    \centering
    \begin{subfigure}[t]{0.75\textwidth}
        \centering
        \includegraphics[width=\textwidth]{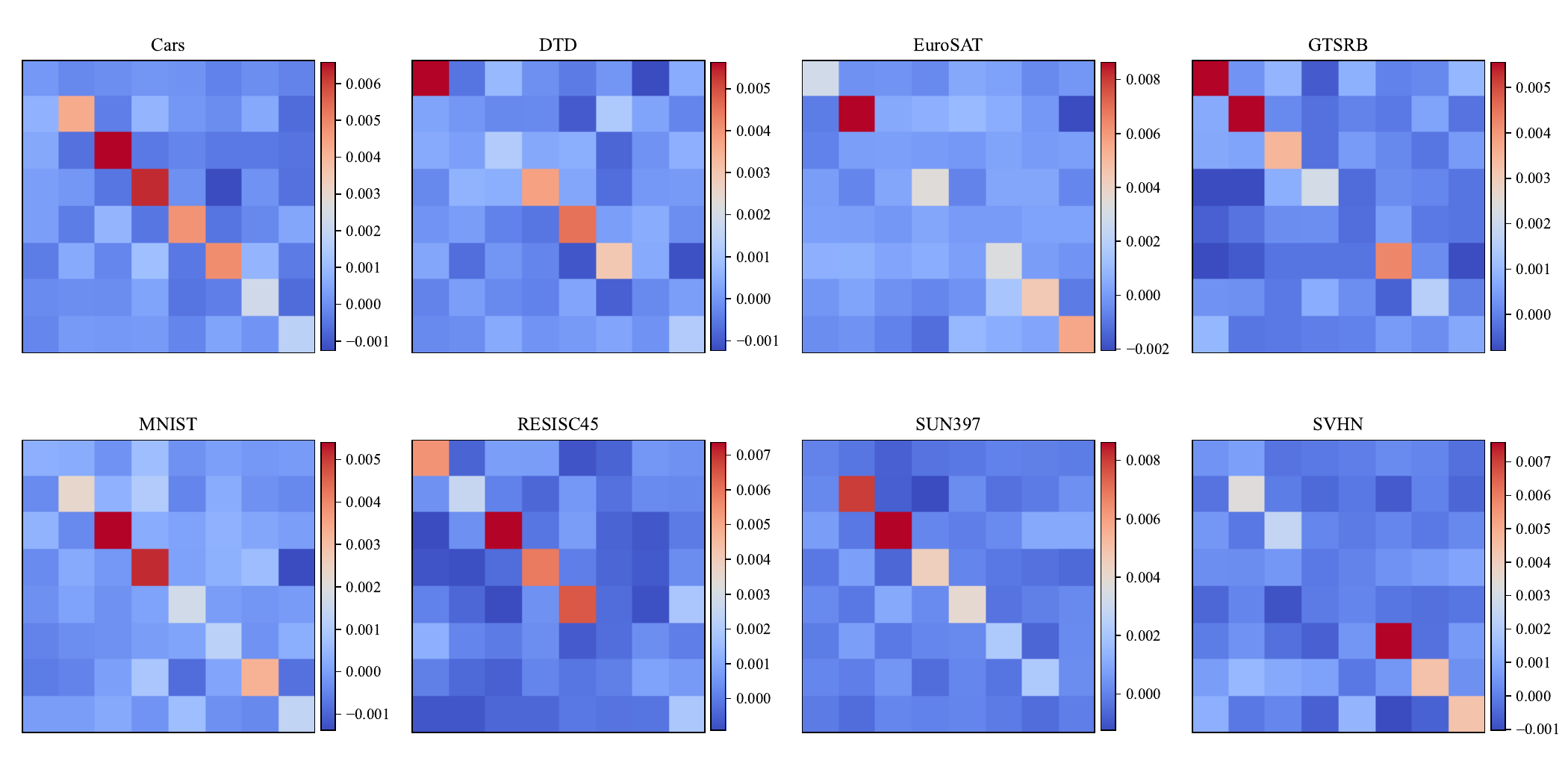}
        \caption{encoder.layers.2.self\_attn.q\_proj.weight}
        \label{fig:a}
    \end{subfigure}
    
    
    \begin{subfigure}[t]{0.75\textwidth}
        \centering
        \includegraphics[width=\textwidth]{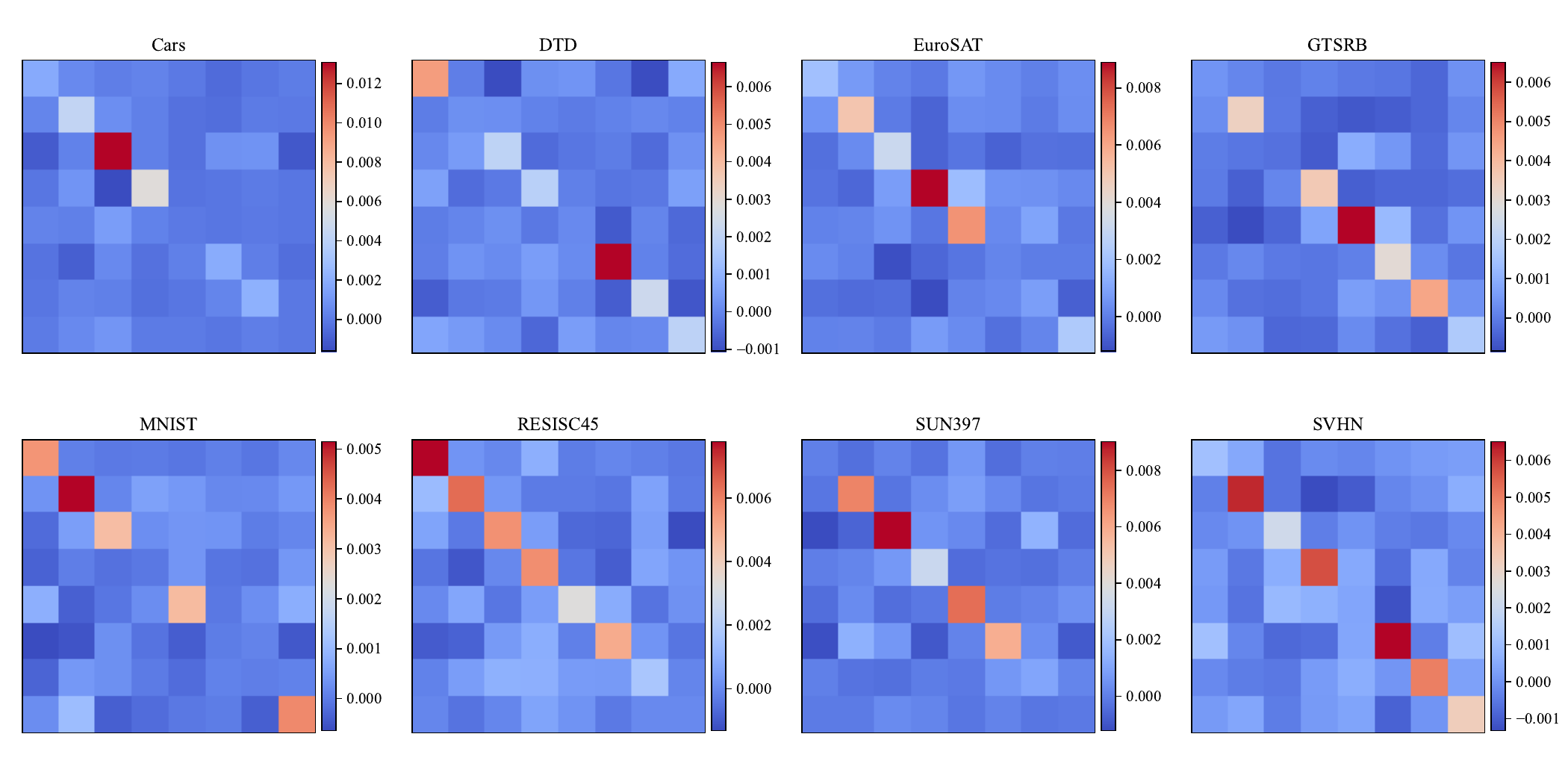}
        \caption{encoder.layers.2.self\_attn.k\_proj.weight}
        \label{fig:b}
    \end{subfigure}

    \begin{subfigure}[t]{0.75\textwidth}
        \centering
        \includegraphics[width=\textwidth]{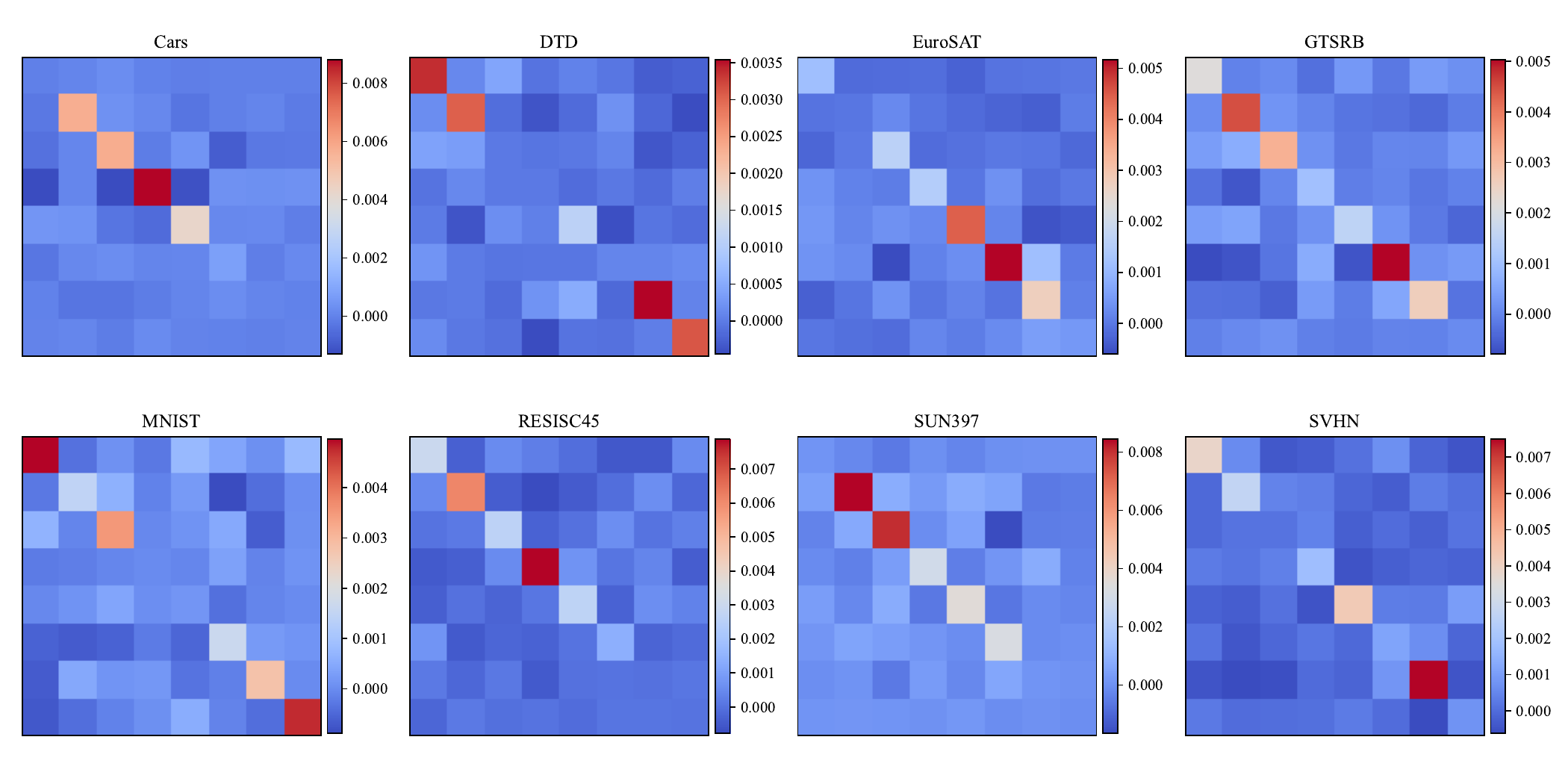}
        \caption{encoder.layers.2.self\_attn.v\_proj.weight}
        \label{fig:c}
    \end{subfigure}
    
    \caption{Visualization of task representations in the shared cover space. We apply block-diagonal masks to eliminate cross-task directional interference (off-diagonal blocks) when projecting the aggregated task representations back to parameter space, where the mask size implicitly balances directional preservation and task fidelity. (Continued on next page)}
    \label{fig:mi_cover_space_part1}
\end{figure*}


\begin{figure*}[h]
    \centering
    \ContinuedFloat 
    
    \begin{subfigure}[t]{0.75\textwidth}
        \centering
        \includegraphics[width=\textwidth]{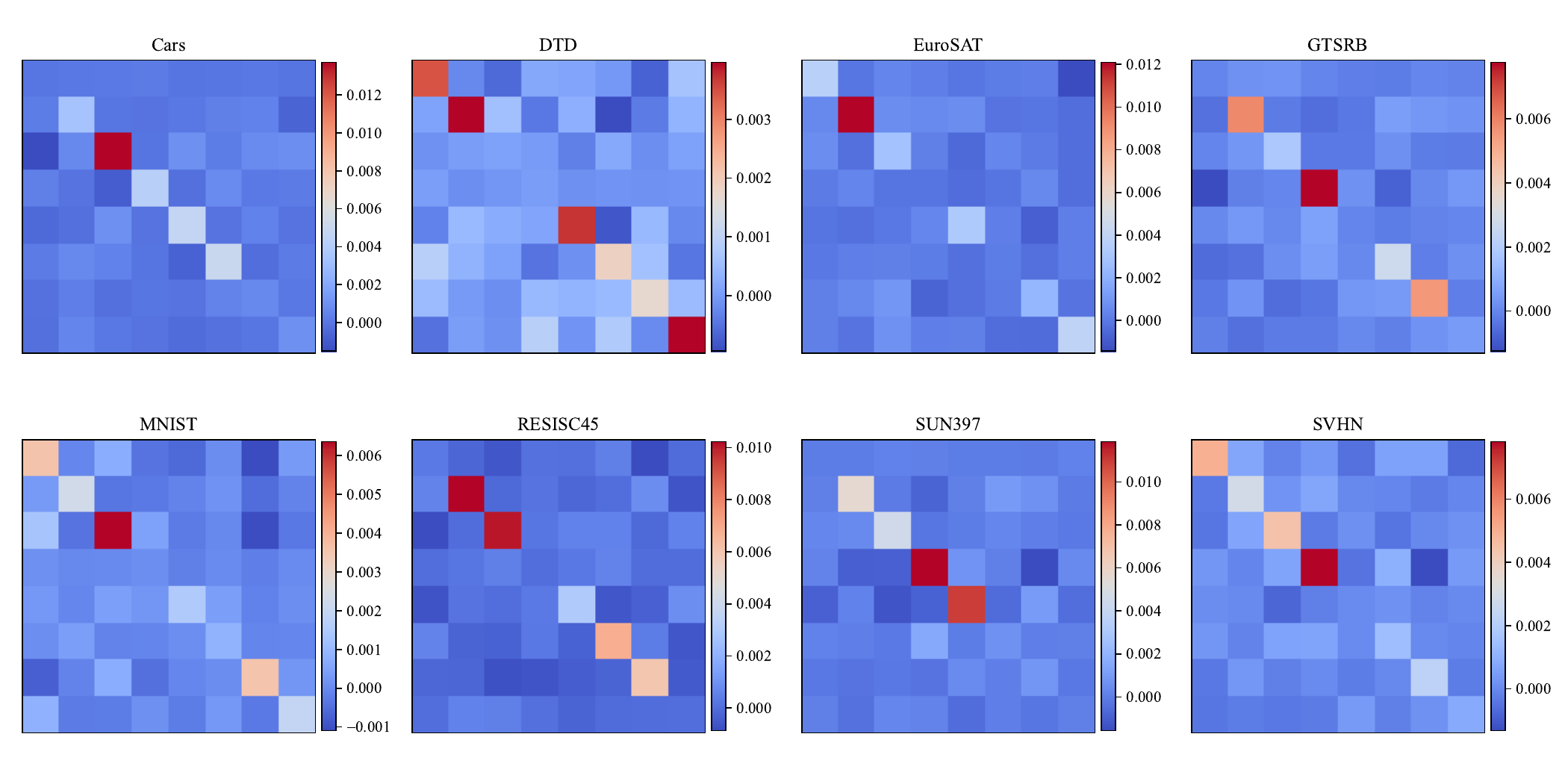}
        \caption{encoder.layers.5.self\_attn.q\_proj.weight}
        \label{fig:d}
    \end{subfigure}

    \begin{subfigure}[t]{0.75\textwidth}
        \centering
        \includegraphics[width=\textwidth]{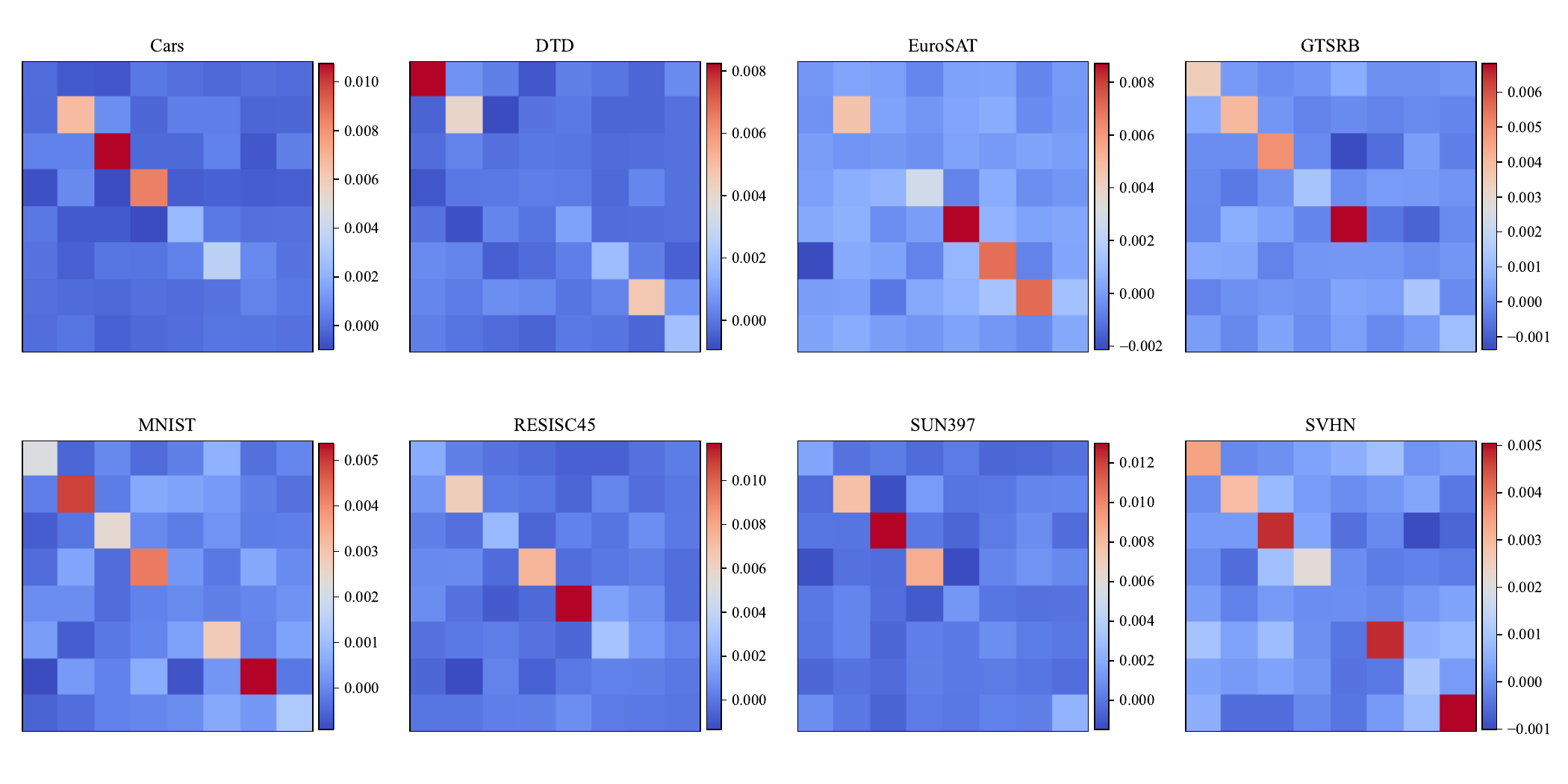}
        \caption{encoder.layers.5.self\_attn.k\_proj.weight}
        \label{fig:e}
    \end{subfigure}

    \begin{subfigure}[t]{0.75\textwidth}
        \centering
        \includegraphics[width=\textwidth]{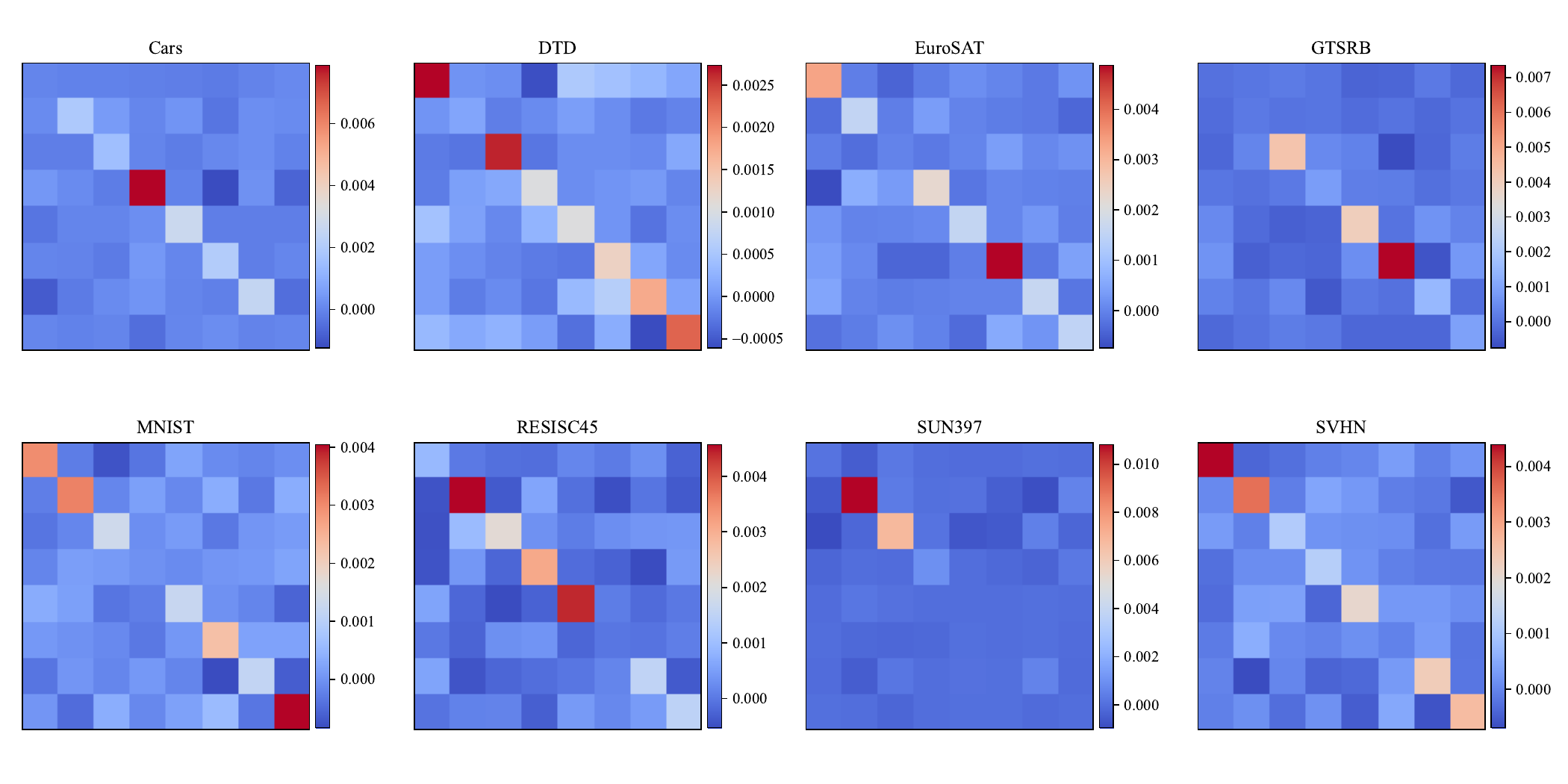}
        \caption{encoder.layers.5.self\_attn.v\_proj.weight}
        \label{fig:f}
    \end{subfigure}
    

    \caption{Visualization of task representations in the shared cover space. We apply block-diagonal masks to eliminate cross-task directional interference (off-diagonal blocks) when projecting the aggregated task representations back to parameter space, where the mask size implicitly balances directional preservation and task fidelity. (Continued on next page)}
    \label{fig:mi_cover_space_part2}
\end{figure*}

\begin{figure*}[h]
    \centering
    \ContinuedFloat 
    
    \begin{subfigure}[t]{0.75\textwidth}
        \centering
        \includegraphics[width=\textwidth]{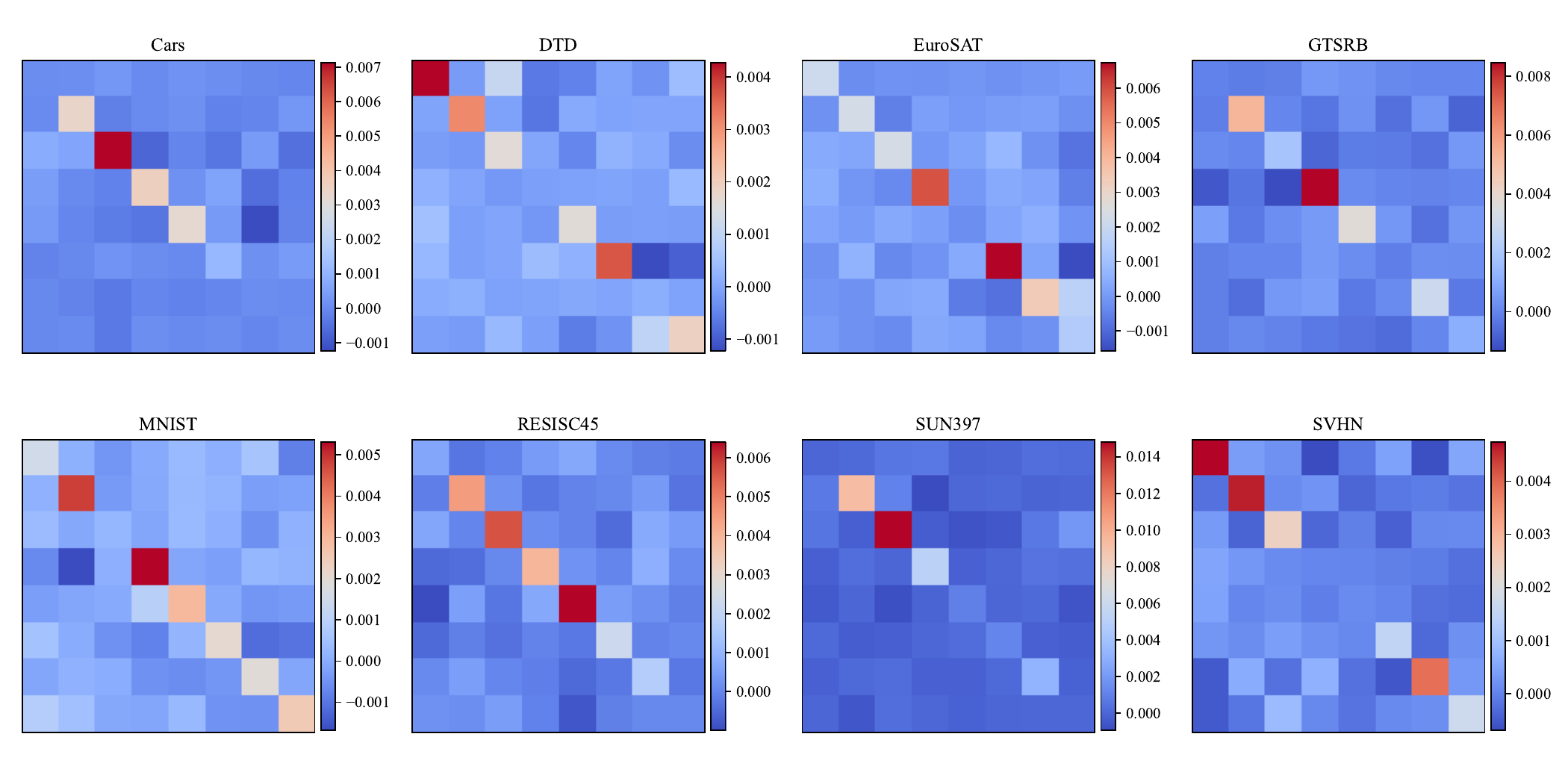}
        \caption{encoder.layers.8.self\_attn.q\_proj.weight}
        \label{fig:g}
    \end{subfigure}

    \begin{subfigure}[t]{0.75\textwidth}
        \centering
        \includegraphics[width=\textwidth]{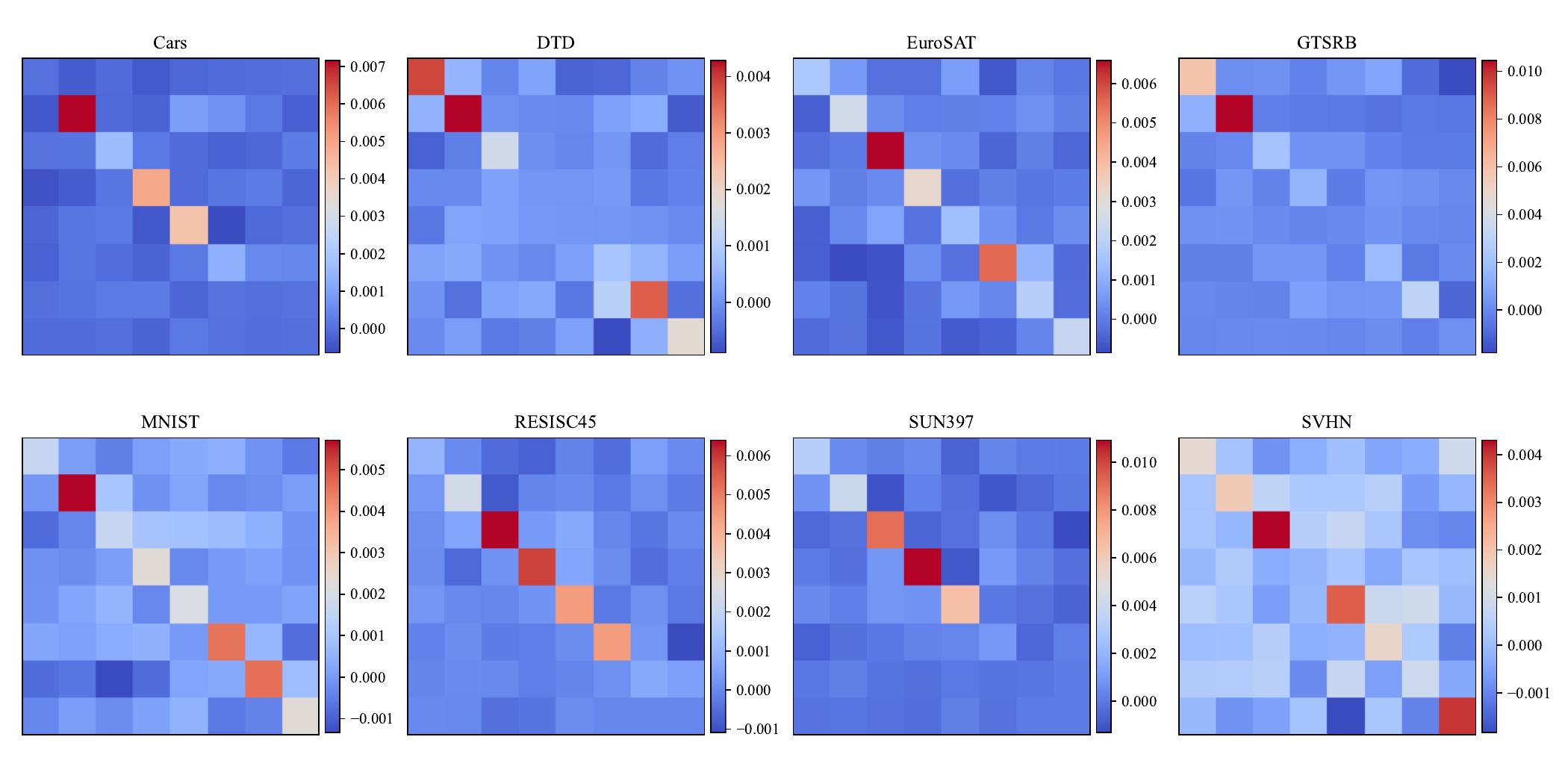}
        \caption{encoder.layers.8.self\_attn.k\_proj.weight}
        \label{fig:h}
    \end{subfigure}

    \begin{subfigure}[t]{0.75\textwidth}
        \centering
        \includegraphics[width=\textwidth]{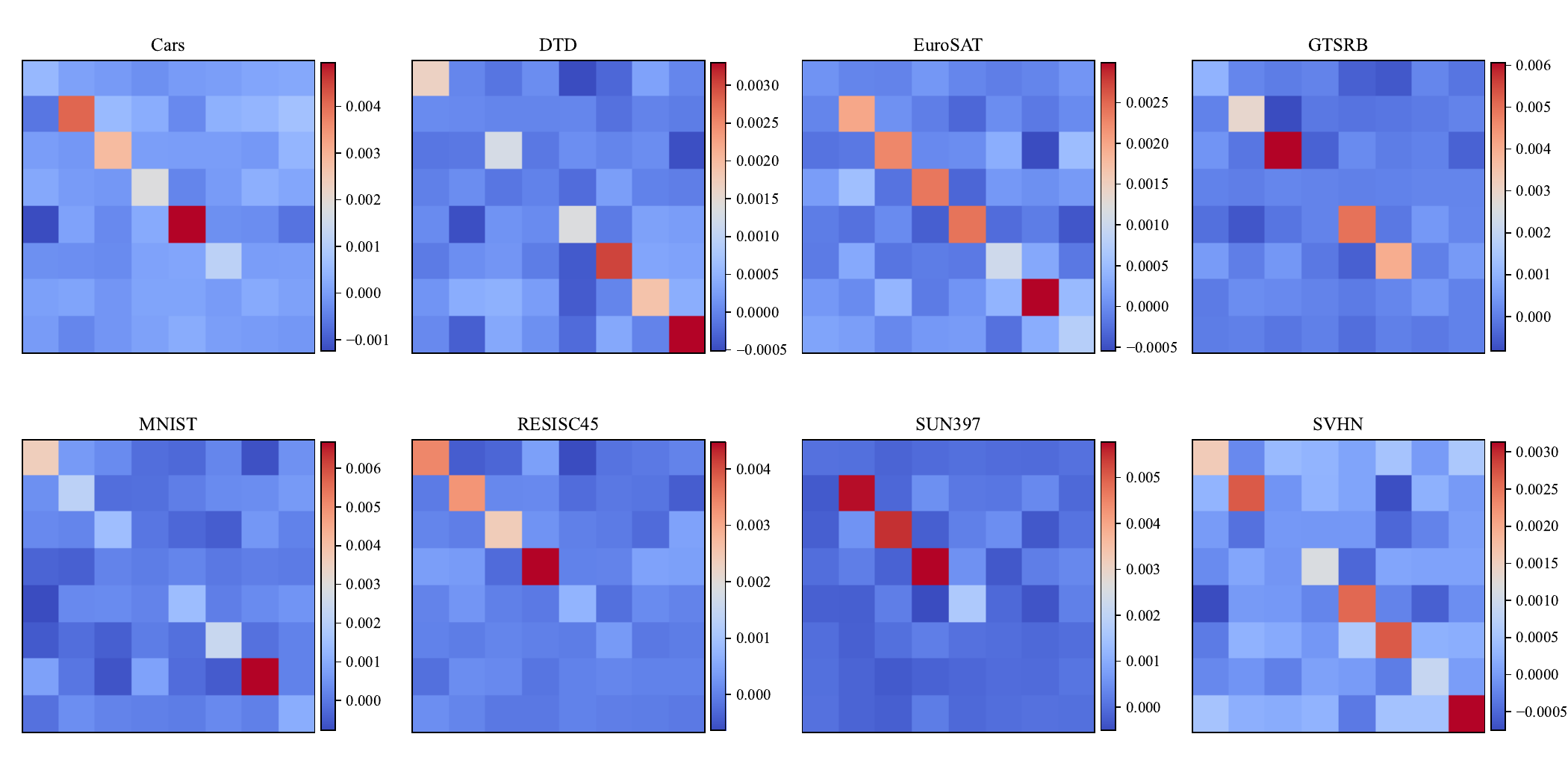}
        \caption{encoder.layers.8.self\_attn.v\_proj.weight}
        \label{fig:i}
    \end{subfigure}
    

    \caption{Visualization of task representations in the shared cover space. We apply block-diagonal masks to eliminate cross-task directional interference (off-diagonal blocks) when projecting the aggregated task representations back to parameter space, where the mask size implicitly balances directional preservation and task fidelity.}
    \label{fig:mi_cover_space_part3}
\end{figure*}

\begin{figure*}[h]
    \centering
    \includegraphics[width=0.98\textwidth]{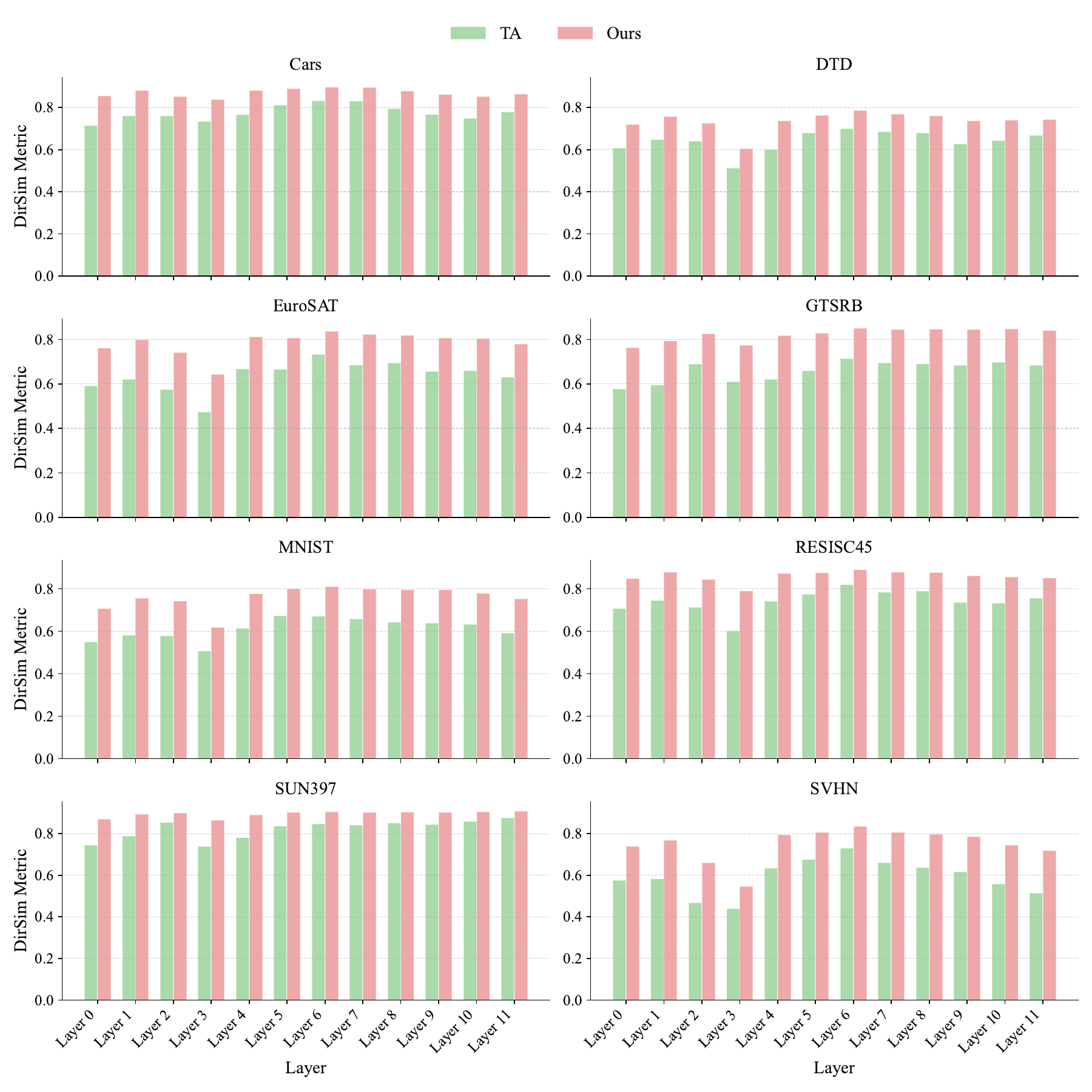}
    \caption{
    Layer-wise visualization of projected $\mathrm{DirSim}$ with each individual task. The results are based on \model{ViT-B-32} 8-task benchmark in LoRA setting. DC-Merge consistently exhibits higher projected $\mathrm{DirSim}$ with task vectors across all layers.
    }
    \label{fig:per_layer_per_task_dirsim}
\end{figure*}

\end{document}